\documentclass{article}
\usepackage{iclr2025_conference,times}

\usepackage{amsmath,amsfonts,bm}

\def\eqref#1{equation~\ref{#1}}

\def\1{\bm{1}}

\DeclareMathAlphabet{\mathsfit}{\encodingdefault}{\sfdefault}{m}{sl}
\SetMathAlphabet{\mathsfit}{bold}{\encodingdefault}{\sfdefault}{bx}{n}

\newcommand{\E}{\mathbb{E}}

\usepackage{microtype}
\usepackage{graphicx}
\usepackage{subfigure}
\usepackage{booktabs} %
\usepackage{lipsum}
\usepackage{appendix}
\usepackage{enumitem}
\usepackage{wrapfig}
\usepackage{minitoc}
\usepackage{hyperref}
\usepackage{url}
\usepackage[capitalize,noabbrev]{cleveref}
\usepackage{colortbl}
\usepackage[most]{tcolorbox}
\definecolor{papercolor}{HTML}{0668E1}
\definecolor{windows_98}{HTML}{008080}

\hypersetup{
    colorlinks,
    linkcolor={papercolor!75!black},
    citecolor={papercolor!75!black},
    urlcolor={papercolor!75!black}
}

\newcommand{\Revision}[1]{#1}
\newcommand{\RevisionTwo}[1]{#1}
\newcommand{\RevisionThree}[1]{#1}
\newcommand{\RevisionFour}[1]{#1}

\usepackage{amsmath}
\usepackage{amssymb}
\usepackage{mathtools}
\usepackage{amsthm}

\newcommand{\appropto}{\mathrel{\vcenter{
  \offinterlineskip\halign{\hfil$##$\cr
    \propto\cr\noalign{\kern2pt}\sim\cr\noalign{\kern-2pt}}}}}

\usepackage[capitalize,noabbrev]{cleveref}
\usepackage{mdframed}

\usepackage{graphicx}

\usepackage{xcolor}         %

\usepackage{graphicx}       %

\usepackage{dsfont}         %

\PassOptionsToPackage{comma}{natbib}

\newcommand{\trace}{\operatorname{tr}\,}
\newcommand{\ntrace}{\bar{\operatorname{tr}}\,}
\newcommand{\dof}{\operatorname{df}}
\newcommand{\ndof}{\bar{\operatorname{df}}}
\newcommand{\id}{I}

\theoremstyle{plain}
\newtheorem{theorem}{Theorem}[section]

\newtheorem{lemma}{Lemma}[section]
\newtheorem{corollary}{Corollary}[section]
\newtheorem{definition}{Definition}[section]
\newtheorem{assumption}{Assumption}[section]

\usepackage{caption}
\usepackage{subcaption}
\usepackage{amsmath}
\usepackage{bm} %
\usepackage{amsthm} %
\usepackage{wrapfig}

\title{An Effective Theory of Bias Amplification}

\author{%
    Arjun Subramonian$^{1,}\thanks{Work done while interning at Meta. Corresponding author: Arjun Subramonian (arjunsub@cs.ucla.edu).}$, Samuel J.~ Bell$^{2}$, Levent Sagun$^{2}$, Elvis Dohmatob$^{2,3,4}$ \\
    $^1$UCLA\quad $^2$Meta FAIR \quad $^3$Concordia University \quad $^4$Mila
}

\iclrfinalcopy
\begin{document}

\doparttoc %
\faketableofcontents %

\maketitle

\begin{abstract}
Machine learning models can capture and amplify biases present in data, leading to disparate test performance across social groups. To better understand, evaluate, and mitigate these biases, a deeper theoretical understanding of how model design choices and data distribution properties contribute to bias is needed. In this work, we contribute a precise analytical theory in the context of ridge regression, both with and without random projections, where the former models feedforward neural networks in a simplified regime. Our theory offers a unified and rigorous explanation of machine learning bias, providing insights into phenomena such as bias amplification and minority-group bias in various feature and parameter regimes. For example, we observe that there may be an optimal regularization penalty or training time to avoid bias amplification, and there can be differences in test error between groups that are not alleviated with increased parameterization. Importantly, our theoretical predictions align with  empirical observations reported in the literature on machine learning bias. We extensively empirically validate our theory on synthetic and semi-synthetic datasets.
\end{abstract}

\section{Introduction}

Machine learning (ML) datasets can encode a plethora of biases which, when said data is used to train models, can result in systems that can cause practical harm.
Datasets that encode correlations that only hold for a subset of the data may cause disparate performance when models are used more broadly, such as an X-ray pneumonia classifier that only functions on images from certain hospitals \citep{zech2018variable}.
This issue is magnified when coupled with under-representation, whereby a dataset fails to adequately reflect parts of the underlying data distribution, often further marginalizing certain groups. Lack of representation results in systems that might work well on average, but fail for minoritized groups, including facial recognition systems that fail for darker-skinned women \citep{pmlr-v81-buolamwini18a}, large language models that consistently misgender transgender and nonbinary people \citep{Ovalle2023TGNB}, or image classification technology that only works in Western contexts \citep{devries2019does,richards2023does}.

Unfortunately, contemporary models may exhibit \emph{bias amplification}, whereby dataset biases are not only replicated, but exacerbated \citep{zhao-etal-2017-men,hendricks2018women,wang2021directionala}.
While previous research has shown that amplification is a function of both dataset properties and how we choose to construct our models \citep{hall2022systematic,sagawa2020investigation,bell2023simplicity}, it is not fully clear how bias amplification occurs mechanistically, nor do we precisely understand which settings lead to its emergence. Thus, in this work, we propose a novel theoretical framework that explains how model design choices (e.g., number of parameters, regularization penalty) and data distributional properties (e.g., number of features, group imbalance, label noises) interact to amplify bias. Moreover, our framework provides an account of  prior work on bias amplification \citep{bell2023simplicity} and minority-group bias \citep{sagawa2020investigation}.

A theory of bias amplification is important for several reasons.
First, as empirical research necessarily yields only sparse data points---often focused on only the most common regimes---theory allows us to interpolate between past findings, and reason about how bias emerges in under-explored settings.
Second, a precise theory gives us the depth of understanding needed in order to intervene, potentially supporting the development of both novel evaluations and mitigations.
Finally, beyond explaining already-known phenomena, our theory makes new predictions, suggesting new avenues for future research.

\subsection{Main Contributions}
\label{sec:main-contributions}
In this work, we develop a unifying and rigorous theory of ML bias in the settings of ridge regression with and without random projections. In particular, we precisely analyze test error disparities between groups (e.g., demographic groups or protected categories) with different data distributions when training on a mixture of data from these groups. We characterize these disparities in high dimensions using operator-valued free probability theory (OVFPT), thereby avoiding possibly loose bounds on critical quantities. Our theory encompasses different parameterization regimes, group sizes, label noises, and data covariance structures. Moreover, our theory has applications to important problems in ML bias that have recently been empirically investigated:
\begin{itemize}[leftmargin=1em]
\item \textbf{Bias amplification.} Even in the absence of group imbalance and spurious correlations, a single model that is trained on a combination of data from different groups can amplify bias beyond separate models that are trained on data from each group \citep{bell2023simplicity}. With our theory, we reproduce and analyze the bias amplification findings of \citet{bell2023simplicity} in controlled settings. We further observe how stopping model training early or tuning the regularization hyperparameter can alleviate bias amplification.

\item \textbf{Minority-group bias.} Overparamaterization can hurt test performance on minority groups due to spurious features \citep{sagawa2020investigation, Khani2021Spurious}. We theoretically analyze how model size and extraneous features affect minority-group bias.
\end{itemize}

We extensively empirically validate our theory in controlled and semi-synthetic settings. Specifically, we show that our theory aligns with practice in the cases of: (1) bias amplification with synthetic data generated from isotropic covariance matrices and the semi-synthetic dataset Colored MNIST \citep{arjovsky2019invariant}, and (2) minority-group bias under different model sizes with synthetic data generated from diatomic covariance matrices. In these applications, we expose new, interesting phenomena in various regimes. For example, a larger number of features than samples can amplify bias under overparameterization, there may be an optimal regularization penalty or training time to avoid bias amplification, and there can be differences in test error between groups that are not alleviated with increased parameterization. Our observations of phenomena in Sections \ref{sec:bias-amp-exp} and \ref{sec:min-group-exp} are largely empirical but are supported by their agreement with our theory. Our theory of ML bias can inform strategies to evaluate and mitigate unfairness in ML, or be used to caution against the usage of ML in certain applications.

\subsection{Related Work}
\label{sec:rw-main}

\textbf{Bias amplification.} A long line of research has explored how ML exacerbates biases in data. For example, a single model that is trained on a combination of data from different groups can amplify bias \citep{zhao-etal-2017-men, wang2021directionala}, even beyond what would be expected when separate models are trained on data from each group \citep{bell2023simplicity}. \citet{hall2022systematic} conduct a systematic empirical study of bias amplification in the context of image classification, finding that amplification can vary greatly as a function of model size, training set size, and training time. Furthermore, overparameterization, despite reducing a model's overall test error, can disproportionately hurt test performance for minority groups \citep{sagawa2020investigation, Khani2021Spurious}. Models can also overestimate the importance of poorly-predictive, low-signal features for minority groups, thereby hurting performance on these groups \citep{leino2018feature}. In this paper, we distill a holistic theory of how model design choices and data distributional properties affect disparate test performance across groups, which can encompass seemingly disparate bias phenomena.

\textbf{High-dimensional analysis of ML.}
A suite of works have analyzed the expected dynamics of ML in appropriate asymptotic scaling limits, e.g., the rate of features $d$ to samples $n$ converges to a finite values as $d$ and $n$ respectively scale towards infinity \citep{adlam2020neural,tripuraneni2021covariate,lee2023demystifying}. Notably, \citet{bach2024high} theoretically analyzes the double descent phenomenon \citep{spigler2019jamming,belkin2019reconciling} in ridge regression with random projections by computing deterministic equivalents for relevant random matrix quantities in a proportionate scaling limit.  
Like \citet{adlam2020neural, tripuraneni2021covariate, lee2023demystifying}, we leverage the tools of %
OVFPT \citep{mingo2017free}, which is at the intersection of random matrix theory (RMT) and functional analysis.
\RevisionThree{However, \citet{adlam2020neural} focus on training and testing a random features model on data from the same Gaussian distribution. Furthermore, \citet{tripuraneni2021covariate, lee2023demystifying} focus on training a random features model on data from one Gaussian distribution and testing the model on a different Gaussian. In contrast, we study the random features model in the setting of training on a mixture of Gaussian distributions and testing on each component. Because a mixture is more expressive than a single Gaussian, our theoretical results cannot be derived as a special case of these other works.}
Furthermore, our theory non-trivially generalizes \citet{bach2024high}, which we recover in Corollary \ref{corr:unregularized-edd} as a special case, and requires more powerful analytical techniques.

\Revision{Certain prior theoretical work precisely analyzes the bias of models trained on a mixture of data from different groups in a high-dimensional setting \citep{mannelli2022unfair, jain2024bias}. Like \citet{mannelli2022unfair, jain2024bias}, we study linear models and consider bias as the disparity in test performance of a model between groups. We further consider some similar factors that give rise to bias amplification (e.g., group imbalance, group data variance, inter-group similarity, dataset size). We also share some theoretical conclusions, such as bias can occur even when the groups have the same ground-truth weights (see Section \ref{sec:spurious-diatomic-setup}) and are balanced (Section \ref{sec:bias-amp-iso-results}). Additionally, we both discuss the paradigms of training a single model for both groups vs. separate models for each group.} \Revision{However, the {\em main distinction} between our work and \citet{mannelli2022unfair, jain2024bias} is that we precisely characterize how models amplify bias in different {\em parameterization regimes}, that is, we examine the impact of model size on bias. This enables us to expose new, richer insights into the impact of over- and underparameterization on bias amplification (see Figure \ref{fig:phase-diags}, Section \ref{sec:bias-amp-exp}, and Section \ref{sec:spurious-diatomic-setup}). See Appendix \ref{sec:rw-cont} for further comparison of our work to \citet{mannelli2022unfair, jain2024bias}.}

\section{Preliminaries}
\label{sec:prelim}
\subsection{Data Distributions}
We consider a ridge regression problem on a dataset from the following multivariate Gaussian mixture with two groups $s=1$ and $s=2$. These groups could represent different demographic groups or protected categories. 
\begin{align}
&\textbf{(Group ID) }\mathrm{Law}(s) = \mathrm{Bernoulli}(p),\\
    &\textbf{(Features) }\mathrm{Law}(x \mid s) = \mathcal N(0,\Sigma_s),\\
    &\textbf{(Ground-truth weights) }\mathrm{Law}(w_1^*) = \mathcal N(0, \Theta / d),\quad \mathrm{Law}(w_2^* - w_1^*) = \mathcal N(0, \Delta / d), \\
    &\textbf{(Labels) }\mathrm{Law}(y \mid s,x) = \mathcal N(f^\star_s(x),\sigma_s^2),\text{ with }f^\star_s(x) := x^\top w_s^*.
\end{align}
The scalar  $p \in (0,1)$ controls for the relative size of the two groups (e.g., $p=1/2$ in the balanced setting). For simplicity of notation, we define $p_1 = p$ and $p_2 = 1 - p$. The $d \times d$ positive-definite matrices $\Sigma_1$ and $\Sigma_2$ are the covariance matrices for the different groups. The $d$-dimensional vectors $w_1^*$ and $w_2^*$ are the ground-truth weights vectors for each group. $w_1^*$ and $w_2^* - w_1^*$ are independently sampled from zero-mean Gaussian distributions with covariances $\Theta / d$ and $\Delta / d$, respectively. In particular, setting $\Delta = 0$ corresponds to the case that both groups have identical ground-truth weights. We define $\Theta_1 = \Theta, \Theta_2=\Theta+\Delta$. Finally, $\sigma_s^2$ corresponds to the label noise for each group $s$. While we consider the case of two groups only for conciseness, our theoretical methods readily extend to any finite number of groups.

\subsection{Models and Metrics}
\label{sec:models-metrics}

\textbf{Learning.} A learner is given an IID sample $\mathcal D = \{(x_1,y_1),\ldots,(x_n,y_n)\} \equiv (X \in \mathbb{R}^{n \times d}, Y \in \mathbb{R}^n)$ of data from the above distribution and it learns a model for predicting the label $y$ from the feature vector $x$. Thus, $X$ is the total design matrix with $i$th row $x_i$, and $y$ the total response vector with $i$th component $y_i$.
Let $\mathcal D^s = (X \in \mathbb{R}^{n_s \times d}, Y \in \mathbb{R}^{n_s})$ be the data pertaining only to group $s$, so that $\mathcal D = \mathcal D^1 \cup \mathcal D^2$ is a partitioning of the entire dataset. Two choices are available to the learner: (1) learn a model a $\widehat f_s \in \mathcal F$ on each dataset $\mathcal D^s$, or (2) learn a single model $\widehat f \in \mathcal F$ on the entire dataset $\mathcal D$. In practice, a choice is made based on scaling vs. personalization considerations.

We consider two solvable settings for linear models: classical ridge regression in the ambient input space, and ridge regression in a feature space given by random projections. The latter allows us to study the role of model size in ML bias, by varying the output dimension of the random projection mapping. This output dimension $m$ controls the size of a feedforward neural network in a simplified regime \citep{maloney2022solvable,bach2024high}.

\textbf{Classical Ridge Regression.}
We will first consider the function class $\mathcal F \subseteq \{\mathbb R^d \to \mathbb R\}$ of linear ridge regression models without random projections.
For any vector $w \in \mathbb R^d$, the model $f$ with parameters $w$ is defined by $f(x) = x^\top w,\text{ for all }x \in \mathbb R^d,$ and is learned with $\ell_2$-regularization. We define the generalization error or risk of any model $f$ with respect to group $s$ as:
\begin{equation}
    R_s(f) = \mathbb E\,[(f(x) - f_s^\star(x))^2 \mid s].
\end{equation}
We consider ridge regression because in addition to its analytical tractability, it can be viewed as the asymptotic limit of many learning problems \citep{dobriban2018high, richards2021asymptotics, hastie2022surprises}. We now formally define some metrics related to bias amplification.
\begin{definition}[Bias Amplification]
\label{def:bias-amp}
    We isolate the contribution of the model to bias when learning from data with different groups. This intuitive conceptualization of bias amplification allows us to quantify the phenomenon. Grounded in the literature \citep{bell2023simplicity}, we define the Expected Difficulty Disparity (EDD) as:
    \begin{eqnarray}
        EDD = | \mathbb E\, R_2(\widehat f_2) - \mathbb E\, R_1(\widehat f_1)|,
    \end{eqnarray}
    where the expectations are w.r.t. randomness in the training data and any other sources of randomness in the models. The $EDD$ captures the difference in test risk between models trained and evaluated on each group separately. In contrast, we define the Observed Difficulty Disparity (ODD) as:
\begin{eqnarray}
    ODD = | \mathbb E\, R_2(\widehat f) - \mathbb E\, R_1(\widehat f)|.
\end{eqnarray}
The $ODD$ captures the bias (i.e., difference in test risk between groups) of a model trained on both groups. Finally, we define the Amplification of Difficulty Disparity (ADD) as $ADD = \frac{ODD}{EDD}$. We say that bias amplification occurs when $ADD > 1$. See Appendix \ref{sec:rw-cont} for further motivation of $ADD$.
\end{definition}

\textbf{Ridge Regression with Random Projections.}
We consider feedforward neural networks in a simplified regime which can be approximated via random projections, i.e., a one-hidden-layer neural network $f(x) = v^\top Sx$ with a {\em linear} activation function. In particular, we extend classical ridge regression by transforming our learned weights as $\widehat w = S \widehat \eta \in \mathbb{R}^d$, where $S \in \mathbb{R}^{d \times m}$ is a random projection with entries that are IID sampled from ${\cal N} (0, 1 / d)$. Ridge regression with random projections offers analytical tractability while exposing bias amplification phenomena related to model size; such phenomena are not exposed by classical ridge regression (see Figure \ref{fig:phase-diags}).
\Revision{Moreover, it has been shown that in high dimensions, training a one-hidden-layer neural network with gradient descent effectively learns a linear predictor over random features \citep{Yehudai2019RandomFeatures}. Furthermore, \citet{Adlam2020FineGrained, bach2024high}; inter alia are able to reproduce interesting phenomena like double descent using the random features model. Nevertheless, \citet{Yehudai2019RandomFeatures} have shown that the model often cannot learn even a ReLU neuron, suggesting that some mechanisms of bias amplification could be different in nonlinear networks.}

\section{Theoretical Analysis}
\textbf{Assumptions.} Some of our theorems will require standard technical assumptions that %
we detail here and in Appendix \ref{sec:assumptions}. Assumption \ref{ass:scaling-random-proj} describes the proportionate scaling limits, standard in RMT, in which we will work. These limits %
enable us to derive deterministic analytical formulae for the expected test risk of models. Our experiments (see Sections \ref{sec:bias-amp-exp} and \ref{sec:min-group-exp}) validate our theory.

\begin{assumption}
In the case of ridge regression with random projections, we will work in the following proportionate scaling limit:
\begin{gather}
n,n_1,n_2,d \to \infty,\quad n_1/n \to p_1,\,n_2/n \to p_2, \, d/n \to \phi,\,m/n \to \psi,\,m/d \to \gamma, \\
d/n_1 \to \phi_1,\,m/n_1 \to \psi_1,\quad d/n_2 \to \phi_2,\,m/n_2 \to \psi_2,
\label{eq:proportionate-random-proj}
\end{gather}
for some constants $\phi_1,\phi_2,\phi,\psi_1,\psi_2,\psi \in (0,\infty)$. The scalar $\phi$ captures the rate of features to samples. Observe that $\phi=p_1 \phi_1$ and $\phi = p_2 \phi_2$. We note that $\phi\gamma = \psi$ and $\phi_s\gamma = \psi_s$. The scalar $\psi$ captures the rate of parameters to samples. The setting $\psi > 1$ (resp. $\psi < 1$) corresponds to the overparameterized  (resp. underparameterized) regime.
\label{ass:scaling-random-proj}
\end{assumption}

\subsection{Main Result: Ridge Regression with Random Projections}
\label{sec:randproj}

To provide a mechanistic understanding of how ML models may amplify bias, our theory elucidates differences in the test error between groups when a single model is trained on a combination of data from both groups vs. when separate models are trained on data from each group. We first consider the classical ridge regression model in Appendix \ref{sec:warm-up-classical} before studying ridge regression with random projections below, which is a more realistic but still analytically solvable setup.

\textbf{Single Random Projections Model Learned for Both Groups.} We first consider the ridge regression model $\widehat f$ with random projections, which is learned using empirical risk minimization and $\ell_2$-regularization with penalty $\lambda$. The parameter $\widehat w$ of the linear model $\widehat f$ is given by the following optimization problem:
\begin{eqnarray}
\widehat w = S\widehat\eta \in \mathbb R^d,\text{ with } \widehat w = \arg\min_{\eta \in \mathbb R^m} L(\eta)= \sum_{s=1}^2 n^{-1}\|X_sS\eta-Y_s\|^2_2 + \lambda\|\eta\|_2^2.
\end{eqnarray}
Explicitly, one can write $\widehat w = S(Z^\top Z + n\lambda I_m)^{-1}Z^\top Y$, where $Z := XS$.
Before presenting our result for the random projections model, we provide some relevant definitions.

\begin{definition}
\label{def:cosmo-constants}
Let $\ntrace A := (1/d)\trace A$ be the normalized trace operator and $(e_1,e_2,\tau,u_1,u_2,\rho)$ be the unique positive solution to the following fixed-point equations:
\begin{gather}
1/\tau = 1 + \ntrace L K^{-1},\quad 1/e_s = 1 + \psi \tau \ntrace \Sigma_s K^{-1},\text{ for }s\in\{1,2\}, \\
\rho = \tau^2 \ntrace (\gamma \rho L^2 + \lambda^2 D) K^{-2},\quad u_s = \psi e_s^2 \ntrace \Sigma_s (\gamma \tau^2 D + \rho I_d) K^{-2},\text{ for }s\in\{1,2\},\\
\text{ where: }
L = p_1 e_1 \Sigma_1 + p_2 e_2 \Sigma_2,\, K = \gamma \tau L + \lambda I_d,\, D = p_1 u_1 \Sigma_1 + p_2 u_2 \Sigma_2 + B.
\end{gather}

For deterministic $d \times d$ PSD matrices $A$ and $B$, we define the following auxiliary quantities:
\begin{align}
    h_j^{(1)} (A) &:= p_j \gamma e_j \tau \ntrace A \Sigma_j K^{-1}, \\
    h_j^{(2)} (A, B) &:= p_j \gamma \ntrace A \Sigma_j (\gamma e_j \tau^2 B + p_{j'} \gamma \tau^2  \Sigma_{j'} (e_j u_{j'} - e_{j'} u_j) + e_j \rho I_d - \lambda u_j \tau I_d) K^{-2}, \\
    h_j^{(3)} (A, B) &:= p_j \ntrace A \Sigma_j (\gamma e_j^2 p_j \Sigma_j (p_{j'}\gamma \tau^2 u_{j'}  \Sigma_{j'} + \gamma \tau^2 B + \rho I_d)\nonumber \\
    &+ u_j (p_{j'} \gamma e_{j'} \tau  \Sigma_{j'} + \lambda I_d)^2) K^{-2}, \\
    h_j^{(4)} (A, B) &:= p_j \gamma p_{j'} \ntrace \Sigma_j \Sigma_{j'} A (\gamma \tau^2 ( e_j e_{j'}B - p_je_j^2 u_{j'}\Sigma_j  - p_{j'} \Sigma_{j'} e_{j'}^2 u_j)\nonumber\\
    &- \lambda \tau (e_j u_{j'} + e_{j'} u_j) I_d + e_j e_{j'} \rho I_d) K^{-2}.
\end{align}
\end{definition}
In Appendix \ref{sec:edd-random-proj-proof}, we intuitively interpret the scalars $e_s, \tau, u_s, \rho$ in the setting where a separate model is learned for each group. In essence, our theory extends the scalars to the more general setting where a single model is trained on a mixture of data from different groups. Furthermore, each of the terms $h_j^{(1)}, \ldots, h_j^{(4)}$ capture the limiting values of different sources of covariance between the sample covariance matrices for the groups, the resolvent matrix, and the random projections matrix $S$. These sources of covariance are written explicitly in Appendix \ref{ref:rand-proj-proof}, and naturally arise from expanding the solution to the ridge regression problem with random projections.

We now present Theorem \ref{thm:odd-theory-rand-proj}, which is our {\em main contribution}. Theorem \ref{thm:odd-theory-rand-proj} presents a novel bias-variance decomposition for the test error $R_s(\widehat f)$ for each group $s\in\{1,2\}$ in the context of ridge regression with random projections. It is a non-trivial generalization of theories in high-dimensional ML which requires the powerful machinery of OVFPT (see proof in Appendix \ref{ref:rand-proj-proof}). 

\begin{theorem}%
Under Assumptions \ref{ass:commute} and \ref{ass:scaling-random-proj}, it holds that $ R_s (\widehat f) \simeq B_s (\widehat f) + V_s (\widehat f)$, with
    \begin{align}
    V_s (\widehat f) &= \sum_{j = 1}^2 \sigma_j^2 \phi h_j^{(2)} (I_d, \Sigma_s), \\
    B_s (\widehat f) &= \ntrace \Theta_s \Sigma_s + h_1^{(3)} (\Theta_s, \Sigma_s) + h_2^{(3)} (\Theta_s, \Sigma_s) + 2 h_1^{(4)} (\Theta_s, \Sigma_s) \\
    &- 2 h_1^{(1)} (\Theta_s \Sigma_s) - 2 h_2^{(1)} (\Theta_s \Sigma_s) + h_{s'}^{(3)} (\Delta, \Sigma_s) \\
    &- 2 \begin{cases}
    0, & s = 1, \\
    h_1^{(3)} (\Delta, \Sigma_2) + h_2^{(4)} (\Delta, \Sigma_2) - h_1^{(1)} (\Delta \Sigma_2), & s = 2.
    \end{cases}
    \end{align}
\label{thm:odd-theory-rand-proj}
\end{theorem}
The unregularized limit corresponds to the minimum-norm interpolator, and alternatively may be viewed as training a neural network until convergence \citep{ali2019continuous}. We discuss methods for, and the complexity of, solving the above fixed-point equations in Appendix \ref{sec:fixed-point-complexity}. Furthermore, in Appendix \ref{sec:unregularized-edd-proof}, we directly express the bias and variance of the test risk of an unregularized model trained on just group $s$ in terms of the second and first-order degrees of freedom of  $\Sigma_s$ and the parameterization rate $\psi_s$. Moreover, in Appendix \ref{sec:bias-amp-pow-setup}, we derive the approximate bias amplification profile of an unregularized model with respect to the ratio $c = \sigma_2^2 / \sigma_1^2$ of label noises, in the setting where the eigenspectra of the covariance matrices have power-law decay.

\textbf{Separate Random Projections Model Learned for Each Group.} We now consider the ridge regression models $\widehat f_1$ and $\widehat f_2$ with random projections, which are learned using empirical risk minimization and $\ell_2$-regularization with penalties $\lambda_1$ and $\lambda_2$, respectively. In particular, we have the following optimization problem for each group $s$:
$\arg\min_{\eta \in \mathbb R^m} L(w)= n_s^{-1}\|X_sS\eta-Y_s\|^2_2 + \lambda_s\|\eta\|_2^2$.
Alternatively, the reader can think of each $\widehat f_s$ as the limit of $\widehat f$ when $p_s \to 1$. In this setting, we deduce Theorem \ref{thm:edd-theory-rand-proj}, which follows from Theorem \ref{thm:odd-theory-rand-proj}.
\begin{theorem}
Under Assumptions \ref{ass:commute} and \ref{ass:scaling-random-proj}, it holds that $R_s (\widehat f_s) \simeq B_s (\widehat f_s) + V_s (\widehat f_s)$, where $V_s (\widehat f_s) = \lim_{p_s \to 1} V_s (\widehat f)$ and $B_s (\widehat f_s) = \lim_{p_s \to 1} B_s (\widehat f)$ (see Appendix \ref{sec:edd-random-proj-proof} for explicit formulae).
\label{thm:edd-theory-rand-proj}
\end{theorem}

\textbf{Phase Diagram.}
The phase diagram for the random projections model (Figure \ref{fig:phase-diags}) offers rich insights into how the rate of parameters to samples ($\psi$), in interaction with the rate of features to samples ($\phi$), affects bias amplification. In the $ODD$ and $EDD$ profiles, we observe phase transitions at $\phi = \psi$ (when $\psi < 0.5$) and $\psi = 0.5$ (i.e., $\psi_1 = \psi_2 = 1$), where these metrics begin decreasing significantly. $\psi_s = 1$ is a known interpolation threshold for random features models \citep{Adlam2020FineGrained, dascoli2020doubletrouble}. In contrast, at $\psi = 1$ and $\phi = 1$, the $ODD$ drastically increases. Furthermore, at $\phi = \psi$ (when $\psi < 0.5$) and $\phi = 0.5$ (for $\psi > 0.5$), the $EDD$ greatly increases. Accordingly, in the $ADD$ profile, we observe phase transitions at $\phi = \psi$ (when $\psi < 0.5$), $\psi = 0.5$, $\psi = 1$, and $\phi = 1$, where bias amplification begins occurring (i.e., $ADD > 1$). However, bias seems to be deamplified (i.e., $ADD < 1$) at $\phi = \psi$ (when $\psi < 0.5$) and $\phi = 0.5$ (when $\psi > 0.5$). Some observations are less visible due the granularity of the color thresholding in Figure \ref{fig:phase-diags}.

\begin{figure}[t!]
    \centering
    \includegraphics[width=0.9\linewidth]{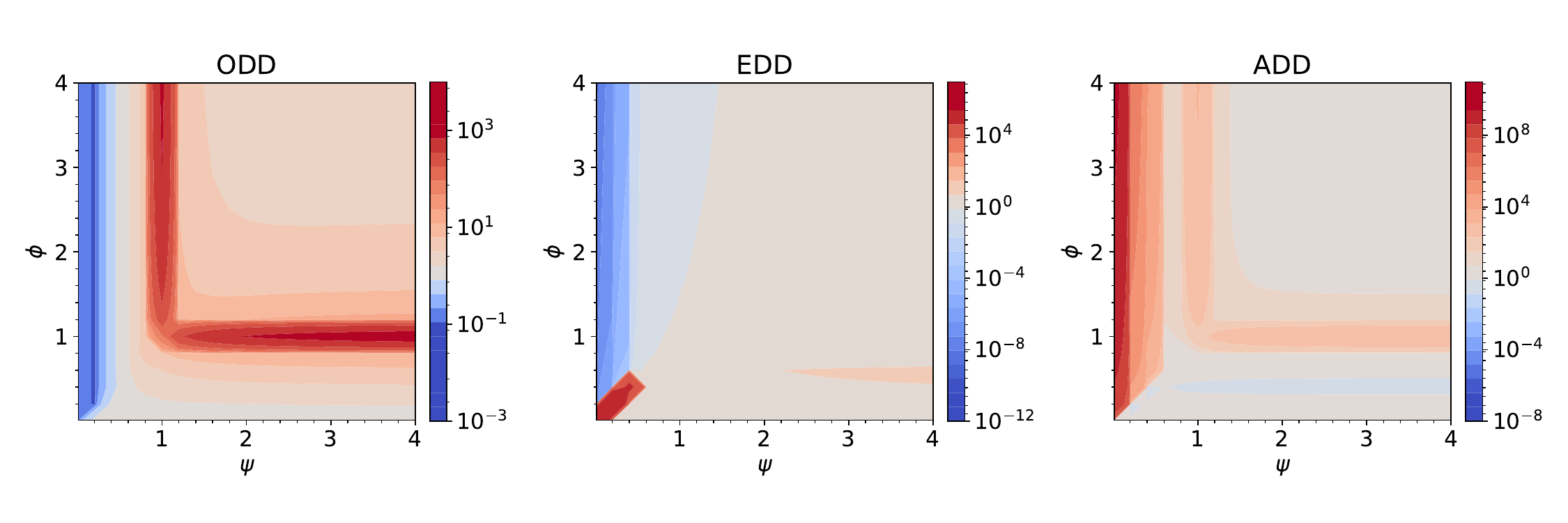}
    \vspace{-.4cm}
    \caption{\textbf{$ODD$, $EDD$, and $ADD$ phase diagrams for ridge regression with random projections.} We plot the bias amplification phase diagrams with respect to $\phi$ (rate of features to samples) and $\psi$ (rate of parameters to samples), as predicted by our theory for ridge regression with random projections (Theorems \ref{thm:odd-theory-rand-proj}, \ref{thm:edd-theory-rand-proj}). Red regions indicate theoretical predictions greater than 1 (i.e., bias amplification in the rightmost plot), while blue regions indicate theoretical predictions less than 1 (i.e., bias deamplification in the rightmost plot). Darkness indicates intensity. We consider isotropic covariance matrices: $\Sigma_1 = 2 I_d, \Sigma_2 = I_d$, $\Theta = 2 I_d$, $\Delta = I_d$. Additionally, $n = 1 \times 10^4, \sigma_1^2 = \sigma_2^2 = 1$. We further choose $\lambda = \lambda_1 = \lambda_2 = 1 \times 10^{-6}$ to approximate the minimum-norm interpolator. We show that bias amplification can occur even in the balanced data setting, i.e., when $p_1 = p_2 = 1/2$.}
    \label{fig:phase-diags}
\end{figure}

\section{Bias Amplification}
\label{sec:bias-amp-exp}

We empirically show how ridge regression models with random projections may amplify bias when a single model is trained on a combination of data from different groups vs. when separate models are trained on data from each group \citep{bell2023simplicity}. We further show how our theory: (1) predicts bias amplification, and (2) exposes new, interesting bias amplification phenomena in various regimes.

\subsection{Isotropic Covariance}
\label{sec:bias-amp-iso-results}
\label{sec:bias-amp-iso-setup}

\textbf{Setup.} To mirror the setting of \citet{bell2023simplicity}, we consider balanced data ($p_1 = p_2 = 1/2$) without spurious correlations ($\Sigma_1 = a_1 I_d, \Sigma_2 = a_2 I_d$, for $a_1, a_2 > 0$). The groups have different ground-truth weights ($\Theta = 2 I_d, \Delta = I_d$). Refer to App. \ref{sec:common-exp-details} for full details due to space limitations.

\textbf{Validation of Theory.} Figure \ref{fig:bias-amp} and the figures in Appendix \ref{sec:bias-amp-plots} reveal that Theorems \ref{thm:odd-theory-rand-proj} and \ref{thm:edd-theory-rand-proj} closely predict the $ODD$, $EDD$, and $ADD$ of ridge regression models with random projections under diverse settings. Note that, as indicated by the error bars, some of our empirical estimates (especially those with larger magnitude) have higher variance and their variance is influenced by the choice of $\psi, \phi, a_1, a_2, \sigma_1^2, \sigma_2^2$. \textbf{Notably, our theory predicts the observation of \cite{bell2023simplicity} that models can amplify bias even with balanced groups and without spurious correlations.} We present new phenomena predicted by our theory below.

\textbf{Effect of Label Noise.} In the $ODD$ profile, when the label noise ratio $c = \sigma_2^2 / \sigma_1^2$ is larger, the right tail is higher for $\phi$ (rate of features to samples) closer to 1 than other $\phi$. This suggests that under overparameterization, a larger noise ratio and similar number of features and samples can increase disparities in test risk between groups when a single model is learned for both groups. We aim to explain this phenomenon analytically in Section \ref{sec:bias-amp-pow-setup}.
Moreover, the $EDD$ curve is generally higher for larger $c$, suggesting that a larger noise ratio increases disparities in test risk when a separate model is learned for each group. This finding is supported by our experiment with real data (see Figure \ref{fig:colored-mnist-1-errors}).

\textbf{Effect of Model Size.} We observe interesting divergent behavior as $\psi$ (rate of parameters to samples) increases for different $\phi$ (rate of features to samples). When $\phi > 1$, as $\psi$ increases, the $ODD$ increases and then decreases, peaking at the interpolation threshold at $\psi = 1$. Similarly, when $\phi > 0.5$ (i.e., $\phi_1 = \phi_2 > 1$), as $\psi$ increases, the $EDD$ increases and then decreases, peaking at the interpolation threshold at $\psi = 0.5$ (i.e., $\psi_1 = \psi_2 = 1$). Accordingly, when $\phi > 0.5$, bias is effectively deamplified ($ADD < 1$) at $\psi = 0.5$ and when $\phi > 1$, bias amplification peaks ($ADD > 1$) at $\psi = 1$. In contrast, when $\phi < 1$, the $ODD$ decreases as $\psi$ increases, plateauing at different finite values. Similarly, when $\phi < 0.5$, the $EDD$ generally decreases and plateaus as $\psi$ increases; in some cases, the $EDD$ dips and/or increases and plateaus. A notable exception to these trends occurs when $\phi \approx 1$, with the corresponding $ODD$ and $ADD$ curves increasing as $\psi$ increases, plateauing at a significantly larger value (i.e., $ADD \gg 1$) than the curves corresponding to other values of $\phi$. We observe a similar phenomenon for the $EDD$ curves when $\phi_1 = \phi_2 \approx 1$. Hence, overparameterization can greatly amplify bias when the number of features is close to the number of samples. Regardless of the regime of $\phi$, the left tail of the $ADD$ profile appears to plateau at 1. The right tail plateaus at different finite values, with the curves corresponding to $\phi > 1$ consistently plateauing above 1. This suggests that when there are more features than samples, overparameterization amplifies bias.

\Revision{Some of the peaks and valleys in Figure \ref{fig:bias-amp} can be attributed to double descent. However, double descent in high dimensions has primarily been studied in the setting where data are drawn from a single Gaussian distribution; this corresponds to the $EDD$ setting, where a separate model is learned for each group. In Figure \ref{fig:phase-diags}, we observe a double descent peak in the $EDD$ at $\psi_1 = \psi_2 = 1$ \citep{Adlam2020FineGrained, dascoli2020doubletrouble}. Our work extends the theoretical treatment of double descent to the setting of training a model on a mixture of Gaussians. However, our theory of bias amplification cannot be reduced to double descent. For example, we note other interpolation thresholds in Figure \ref{fig:phase-diags}; our use of a linear activation does not have a confounding effect here, as interpolation thresholds have also been observed in random features models with nonlinear activations \citep{adlam2020neural}. In addition, much of Sections \ref{sec:bias-amp-exp} and \ref{sec:min-group-exp}, and Appendix \ref{sec:bias-amp-pow-setup}, are devoted to studying the tails or limiting behavior of bias amplification with respect to $\psi$ and $\phi$.}

\textbf{Effect of Number of Features.} In the $ODD$ and $ADD$ profiles, when the rate of features to samples $\phi > 1$, the right tail generally plateaus at higher values (i.e., greater than 1) when $\phi$ is closer to 1. This suggests that with a similar number of features and samples, under overparameterization, bias amplification increases. In contrast, when $\phi < 1$, the right tail of the $ODD$ and $EDD$ curves seems to plateau at higher values when $\phi$ is larger. Regardless of the regime of the rate of features to samples $\phi$, the left tails of the $ODD$ and $EDD$ curves are generally higher for larger $\phi$.  

\begin{figure}[t!]
    \centering
    \includegraphics[trim={0 0 0 1.25cm},clip,width=0.9\linewidth]{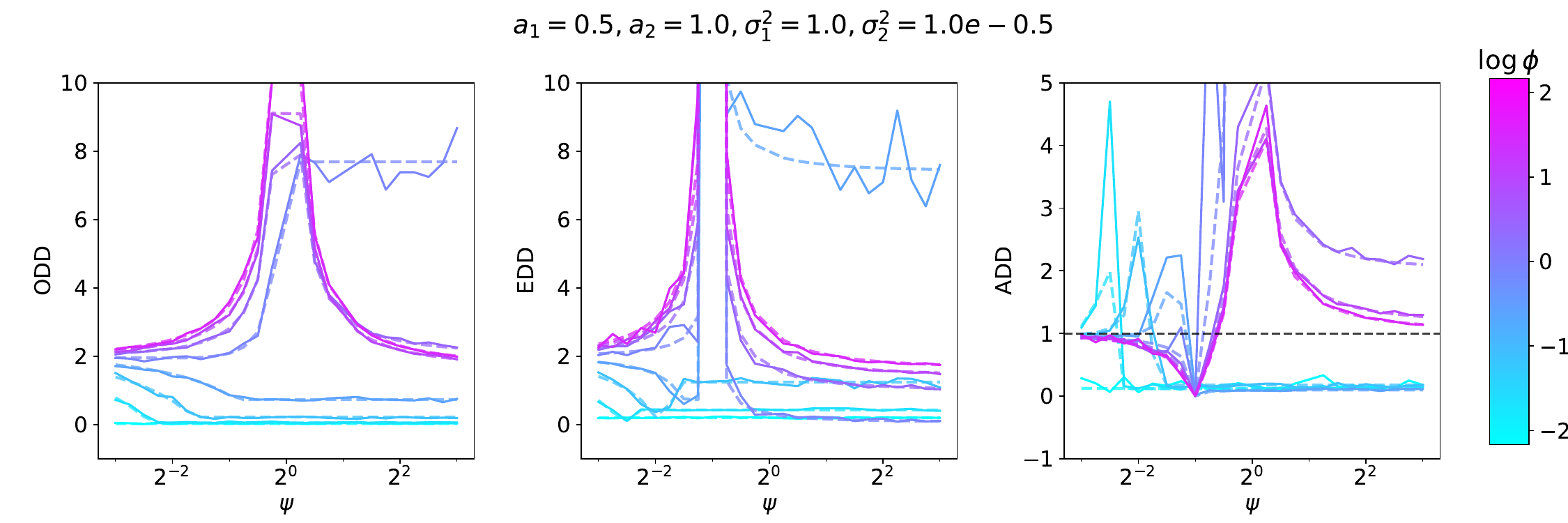}
        \vspace{-.25cm}
    \caption{\textbf{Our theory predicts that models can amplify bias even with balanced groups and without spurious correlations.} We empirically validate our theory (Theorems \ref{thm:odd-theory-rand-proj} and \ref{thm:edd-theory-rand-proj}) for $ODD$, $EDD$, and $ADD$ under the setup described in Section \ref{sec:bias-amp-iso-setup}, with $a_1 = 0.5, a_2 = 1, \sigma_1^2 = 1$, and $\sigma_2^2 = 1 \times 10^{-5}$. The solid lines capture empirical values while the corresponding lower-opacity dashed lines represent what our theory predicts. We plot $ODD$ and $EDD$ on the same scale for easy comparison, and include a black dashed line at $ADD = 1$ to contrast bias amplification vs. deamplification. We include all the plots with error bars in Appendix \ref{sec:bias-amp-plots}.}
    \label{fig:bias-amp}
    \vspace{-0.5cm}
\end{figure}

\subsection{Regularization and Training Dynamics}
\label{sec:reg-training-dynamics}

We now explore how regularization and training dynamics affect bias amplification.

\textbf{Setup.}
\label{sec:over-time-setup} We revisit the experimental setup for Section \ref{sec:bias-amp-iso-setup}. We modulate $a_1, a_2, \psi$ (rate of parameters to samples), as well as $\lambda$ (regularization penalty) to understand the effects of regularization and early stopping on bias amplification. We fix $\sigma_1^2 = \sigma_2^2 = 1$, and the rate of features to samples $\phi = 0.75$.

\textbf{Effect of Regularization and Training Time.} \Revision{In simplistic settings,} we can simulate model learning over training time $t$ by setting $\lambda = 1 / t$ \citep{ali2019continuous}. In the figures in Appendix \ref{sec:bias-amp-training}, we observe that regardless of the regime of $\psi$, $ADD \approx 1$ (i.e., there is neither bias amplification nor deamplification) with high regularization or a short training time. When $\psi > 1$ (i.e., in the overparameterized regime), the $ADD$ is generally greater than 1 across values of $\lambda$ (i.e., bias is amplified), while when $\psi < 1$ (i.e., in the underparameterized regime), the $ADD$ is generally less than 1 (i.e., bias is deamplified). Moreover, when $\psi > 1$, as regularization decreases (or training time increases), bias amplification increases and plateaus. In contrast, when $\psi < 1$, as regularization decreases (or training time increases), bias deamplification increases and plateaus. A notable exception to this trend occurs when $\psi$ is close to 1, where bias is initially deamplified and then amplified as $\lambda$ decreases (or $t$ increases). \textbf{This suggests that there may be an optimal regularization penalty or training time to avoid bias amplification and increase bias deamplification.} Intuitively, as training progresses, overparameterized models may discover ``shortcut'' associations \citep{geirhos2020shortcut} that do not generalize equally well across groups, yielding bias amplification. %
\RevisionThree{In practice, an optimal $\lambda$ or $t$ can be selected by searching for values that strike a desired balance between overall validation error and empirical bias amplification. The search space can be  reduced by using the above $ADD$ trends w.r.t. $\lambda$ and $t$ that our theory predicts for over- vs. underparameterized models (see Appendix \ref{sec:actionable-insights} for more details). It is important for ML practitioners the consider the interplay between high vs. low feature-to-sample regimes and overparameterization in inducing bias amplification vs. deamplification when selecting optimal hyperparameters (see Figure \ref{fig:phase-diags}).}

\Revision{In general, the calibration $\lambda = 1 / t$ may not yield a theoretically tight picture of how bias evolves with $t$. The use of discrete gradient descent in practice rather than continuous-time gradient flows might yield further discrepancies. However, the calibration $\lambda = 1 / t$ yields a ratio of gradient flow to ridge risk that is at most 1.69, with no assumptions on the features $X$ \citep{ali2019continuous}. Moreover, in the controlled settings considered by \citet{ali2019continuous}, this ratio empirically appears to be quite close to 1, and thus may be sufficient for extrapolating our results. Like us, \citet{jain2024bias} and \citet{hall2022systematic} find that bias and bias amplification can vary substantially during training; future work can establish stronger connections between our observations and the results of \citet{jain2024bias}, who analytically identify phases in the evolution of bias and a crossing phenomenon in the test error curves of groups during training. However, \citet{jain2024bias} do not consider the effect of over- and underparameterization on bias evolution. While our analysis relies on the simplistic calibration $\lambda = 1/t$, it reveals divergent behavior in how bias evolves depending on model size.}

\textbf{Corroboration on Real Data.} We further investigate the effect of training time on bias amplification on a more realistic dataset. We train a convolutional neural network (CNN) on Colored MNIST (see Appendix \ref{sec:colored-mnist-details} for more details). Colored MNIST is a semi-synthetic dataset derived from MNIST where digits are randomly re-colored to be red or green \citep{arjovsky2019invariant}. We treat the color of each digit as its group, and we manipulate the groups to have different levels of label noise. In our experimental protocol: (1) the color of each digit (in both train and test) is chosen uniformly at random (i.e., with probability 0.5) and independently of the label; (2) by default, in the training set, the labels of red digits are flipped with probability 0.05 while the labels of green digits are flipped with probability 0.25; (3) labels are binarized (i.e., digits 0-4 correspond to 0 while digits 5-9 correspond to 1); and (4) each training step constitutes a step of gradient descent based on a batch of 250 instances. Although Colored MNIST is a classification task and we use a complex CNN architecture, \textbf{our theory correctly predicts that as the training time $t$ increases, the $ODD$ of the CNN is relatively low while the $EDD$ is much larger}, producing bias deamplification.

\begin{wrapfigure}{R}{6cm}
    \vspace{-0.5cm}
    \centering
    \includegraphics[width=0.35\textwidth]{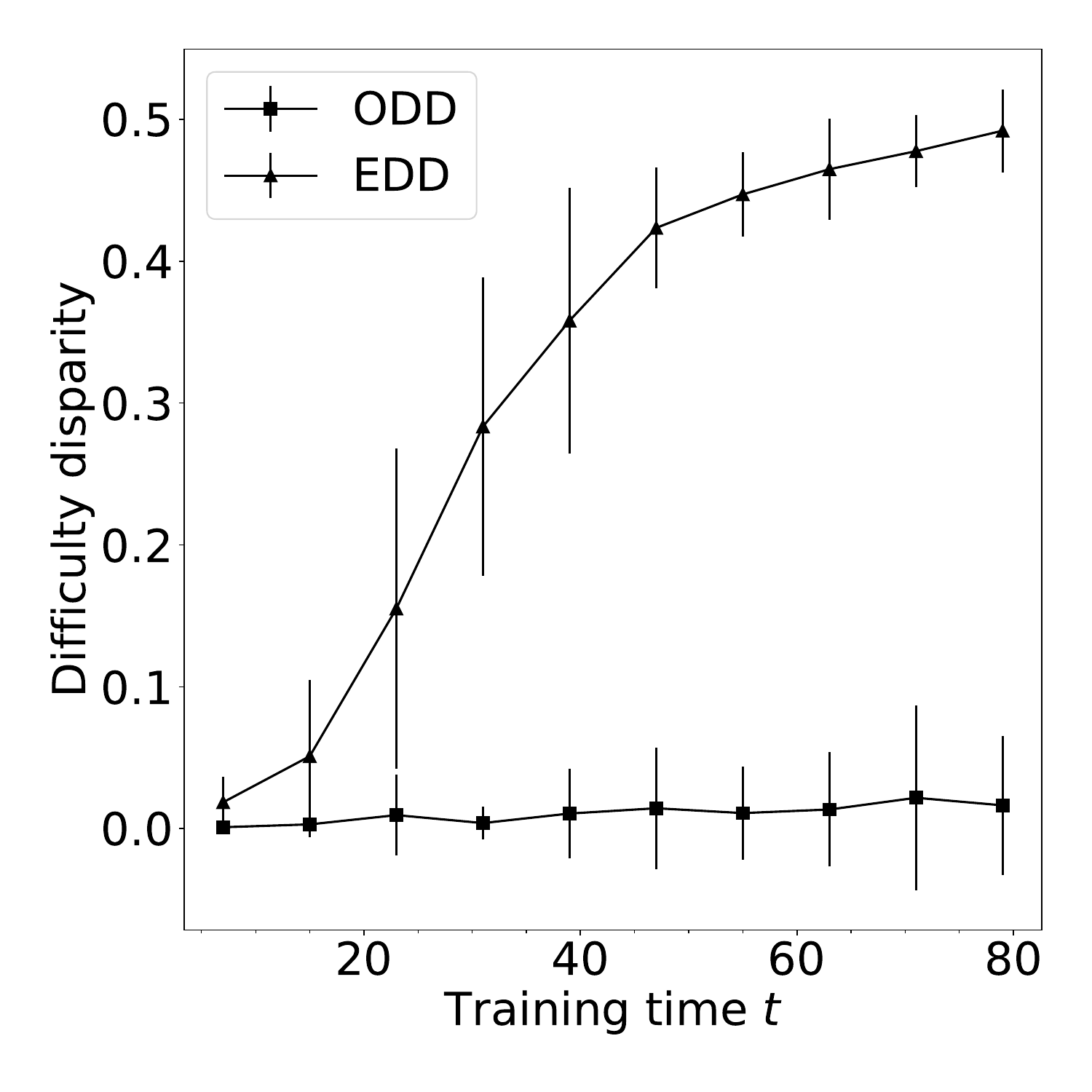}
        \vspace{-.5cm}
    \caption{\textbf{Our theory predicts that disparate label noise between groups deamplifies bias on Colored MNIST.} We plot the $ODD$ and $EDD$ of a CNN over training time $t$ for Colored MNIST. As $t$ increases, the $ODD$ is relatively low while the $EDD$ is noticeably higher. The error bars capture the standard deviation computed over 10 random seeds.}
    \label{fig:colored-mnist}
    \vspace{-0.5cm}
\end{wrapfigure}

Taking $t \to \infty$ corresponds to the setting of $\lambda \to 0^+$ in our theory (Theorems \ref{thm:odd-theory-rand-proj}, \ref{thm:edd-theory-rand-proj}). Because we assign the colors at random, the only difference in image features between groups is color; therefore, we expect the covariance matrices $\Sigma_1$ and $\Sigma_2$ to roughly coincide and $\Delta = 0$ (i.e., $w_1^*=w_2^*$). Note that we do not make any assumptions about the structure of $\Sigma_1, \Sigma_2$. Furthermore, $p_1 = p_2 = 1/2$, and thus, $\phi_1 = \phi_2$ and $\psi_1 = \psi_2$.  Additionally, we analogize the probability of label flipping to label noise in ridge regression. Hence, $e_1 = e_2, u_1 = u_2$. Accordingly, $\lim_{\lambda \to 0^+} B_1 (\widehat f) = \lim_{\lambda \to 0^+} B_1 (\widehat f_1) \approx \lim_{\lambda \to 0^+} B_2 (\widehat f) = \lim_{\lambda \to 0^+} B_2 (\widehat f_2)$. Simultaneously, $\lim_{\lambda \to 0^+} V_1 (\widehat f) \approx \lim_{\lambda \to 0^+} V_2 (\widehat f)$. However, $\lim_{\lambda \to 0^+} V_1 (\widehat f_1) \approx \sigma_1^2 / 2 \cdot V = 0.05 / 2 \cdot V = 0.025 V$ (where $V = \phi_1 h_1^{(2)} (I_d, \Sigma)$), while $\lim_{\lambda \to 0^+} V_2 (\widehat f_2) \approx \sigma_2^2 / 2 \cdot V = 0.25 / 2 \cdot V = 0.125 V$. This results in $ODD \approx 0$ while $EDD \approx 0.1 |V|$, which explains the divergence of $ODD$ and $EDD$ in Figure \ref{fig:colored-mnist}. Intuitively, the high label noise for the green digits prohibits the separate model $\widehat f_2$ from achieving a low test risk compared to $\widehat f_1$; the single model $\widehat f$ achieves a comparable test risk on both groups, effectively deamplifying bias, because of the better learning signal from the red digits. \Revision{This phenomenon is similar to {\em positive transfer}, wherein the $EDD$ of a model generally tends to be higher than the $ODD$ when the labeling rules of imbalanced groups are sufficiently similar \citep{mannelli2022unfair}. However, \citet{mannelli2022unfair} do not explore the impact of model size on positive transfer. We show that the $ODD$ can be less than the $EDD$ depending on $\psi$ in Figure \ref{fig:bias-amp}, where $\Delta = I_d$ (i.e., the groups have different labeling rules). Future work can study the $ADD = \frac{ODD}{EDD}$ profile when $\Delta = 0$.} Refer to Appendix \ref{sec:add-cmnist-plots} for additional Colored MNIST experiments.

\section{Minority-Group Bias}
\label{sec:min-group-exp}
\label{sec:overparam-diatomic-setup}
\label{sec:spurious-diatomic-setup}

Recent work has revealed that overparameterization may hurt test performance on minority groups due to spurious features \citep{sagawa2020investigation, Khani2021Spurious}. Our theory provides new insights into how model size and extraneous features affect minority-group bias.

\textbf{Setup.} To mirror the settings of \citet{sagawa2020investigation, Khani2021Spurious}, we consider diatomic covariance matrices of {\em core} and {\em extraneous} features. We define $A \oplus B = \begin{pmatrix}
    A & 0 \\ 0 & B
\end{pmatrix}$, and choose $\Sigma_1 = a_1 I_{\pi d} \oplus 0 I_{(1-\pi)d}, \Sigma_2 = a_2 I_{\pi d} \oplus b_2 I_{(1-\pi)d}$, for $\pi \in (0,1)$, $a_1, b_2 > 0$, $a_1 = a_2$. Refer to App. \ref{sec:common-exp-details} for full details and a discussion of extraneous vs. spurious features (due to space limitations). 

\begin{figure}[t!]
    \centering
    \includegraphics[trim={0 0 0 1.15cm},clip,width=\linewidth]{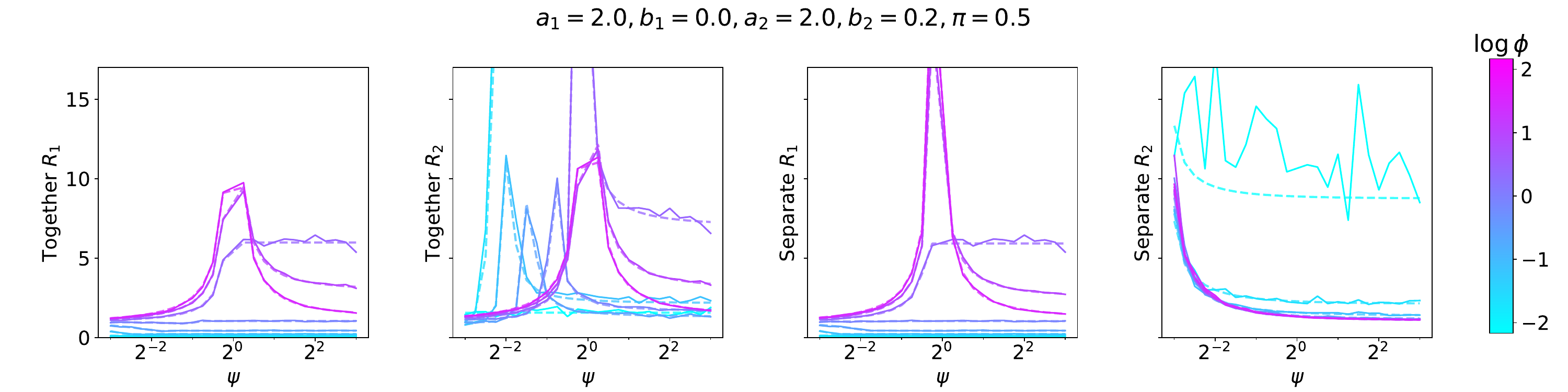}
    \vspace{-.5cm}
    \caption{\textbf{Minority-group test risk can peak with different model sizes depending on the rate of features to samples.} We empirically demonstrate that minority-group bias is affected by extraneous features. We validate our theory (Theorems \ref{thm:odd-theory-rand-proj} and \ref{thm:edd-theory-rand-proj}) for together $R_1, R_2$ (i.e., single model learned for both groups) and separate $R_1, R_2$ (i.e., separate model learned per group) under the setup described in Section \ref{sec:over-time-setup}, with $a_1 = 2, b_2 = 0.2$, and $\pi = 0.5$. The solid lines capture empirical values while the corresponding lower-opacity dashed lines represent what our theory predicts. We include a black dashed line at $ADD = 1$ to contrast bias amplification vs. deamplification. All y-axes are on the same scale for easy comparison. All the plots with error bars are in Appendix \ref{sec:spurious-plots}.}
    \label{fig:spurious-plots}
    \vspace{-0.5cm}
\end{figure}

\textbf{Interpolation Thresholds.} The together $R_2$ (i.e., the test risk for the minority group in the single model setting) has different interpolation thresholds as $\psi$ (rate of parameters to samples) increases, depending on $\phi$ (rate of features to samples) and $\pi$ (fraction of core features). Notably, as $\phi$ increases, the interpolation thresholds occur at larger model sizes, culminating at $\psi = 1$. This suggests that for a higher rate of features to samples, a larger model size can greatly increase the together test risk of the minority group. Furthermore, the interpolation thresholds all occur closer to $\psi = 1$ for larger $\pi$, collapsing to a single threshold at $\psi = 1$ when $\pi \to 1$ (as in Appendix \ref{sec:bias-amp-plots}). Therefore, a lower fraction of core features can yield more possible model sizes that increase the test risk of the minority group. In addition, the together $R_2$ exhibits a steeper rate of growth around the interpolation thresholds for larger $b_2$, suggesting that a higher variance in the extraneous features can also increase the test risk of the minority group in the single model setting. The phenomenon of different interpolation thresholds is not visible for $R_2$ when a separate model is trained per group; however, we do observe the expected double descent peaks in the separate $R_1$ and $R_2$ curves at $\psi_1 = 1$ and $\psi_2 = 1$, respectively.

\textbf{Overparameterization.} The right tails of the together $R_2$ curves plateau at different finite values depending on $\phi$. In particular, for $\phi$ closer to 1, the together $R_2$ curves generally plateau at a higher value, suggesting that a similar number of features and samples can exacerbate minority-group bias under overparameterization. Furthermore, for smaller $\pi$ and certain values of $\phi < 1$, the right tail of the together $R_1$ curve plateaus at a lower value than the together $R_2$ curve. This suggests that there can be differences in test error between groups that are not alleviated even with increased model size. This phenomenon diminishes in magnitude as the fraction of core features increases. This phenomenon supports the finding of \citet{sagawa2020investigation} that \textbf{overparameterization with spurious features can increase test risk disparities between groups}. We identify that the magnitude of this phenomenon may \textbf{depend on both the rate of features to samples and fraction of core features}.

\section{Conclusion}

We present a unifying, rigorous, and effective theory of ML bias in the settings of ridge regression with and without random projections.
Our theory predicts interesting insights into bias amplification and minority-group bias in different feature and parameter regimes.
These findings can inform strategies to evaluate and mitigate unfairness in ML (see Appendix \ref{sec:actionable-insights} for more details). \RevisionTwo{However, there remain practical challenges to assessing whether a model is prone to bias amplification. These include robustly estimating the feature covariance matrices \citep{Bickel2008RegularizedEO} and label noises \citep{Frnay2014ACI} for groups from sample data, especially for minority groups which have limited data. Even so, practitioners can use our theory and empirical observations to form intuition about when {\em disparities} in the variability of features and labels across groups can amplify bias.}

Our theoretical methods are easily extendable to the case of more than two groups and can accommodate label noise sampled from other distributions. However, our theory is not directly extendable to different proportionate scaling limits (e.g., $d^2 / n$ has a finite limit instead of $d / n$). Additionally, our theory requires approximately normally-distributed data and thus does not currently account for missing features, which are common in the real world \citep{feng2024adapting}. Furthermore, our theory implicitly assumes that group information is known, which is not always true  \citep{Coston2019Missing}; however, because we work in an asymptotic scaling limit, having access to group information with $o (\min(n_1, n_2))$ noise is sufficient. As future work, we can leverage ``Gaussian equivalents'' \citep{goldt2022gaussian} to extend our theory to wide, fully-trained networks in the NTK \citep{NEURIPS2018_Jacot} and lazy \citep{chizat} regimes; %
this will enable us to understand how, apart from model size, other design choices like nonlinear activation functions and learning rate may affect bias amplification. %

\clearpage

\bibliography{iclr2025_conference}

\begin{thebibliography}{48}
\providecommand{\natexlab}[1]{#1}
\providecommand{\url}[1]{\texttt{#1}}
\expandafter\ifx\csname urlstyle\endcsname\relax
  \providecommand{\doi}[1]{doi: #1}\else
  \providecommand{\doi}{doi: \begingroup \urlstyle{rm}\Url}\fi

\bibitem[Adlam \& Pennington(2020{\natexlab{a}})Adlam and Pennington]{Adlam2020FineGrained}
Ben Adlam and Jeffrey Pennington.
\newblock Understanding double descent requires a fine-grained bias-variance decomposition.
\newblock In H.~Larochelle, M.~Ranzato, R.~Hadsell, M.F. Balcan, and H.~Lin (eds.), \emph{Advances in Neural Information Processing Systems}, volume~33, pp.\  11022--11032. Curran Associates, Inc., 2020{\natexlab{a}}.

\bibitem[Adlam \& Pennington(2020{\natexlab{b}})Adlam and Pennington]{adlam2020neural}
Ben Adlam and Jeffrey Pennington.
\newblock The neural tangent kernel in high dimensions: Triple descent and a multi-scale theory of generalization.
\newblock In \emph{International Conference on Machine Learning}, pp.\  74--84. PMLR, 2020{\natexlab{b}}.

\bibitem[Ali et~al.(2019)Ali, Kolter, and Tibshirani]{ali2019continuous}
Alnur Ali, J~Zico Kolter, and Ryan~J Tibshirani.
\newblock A continuous-time view of early stopping for least squares regression.
\newblock In \emph{The 22nd international conference on artificial intelligence and statistics}, pp.\  1370--1378. PMLR, 2019.

\bibitem[Arjovsky et~al.(2019)Arjovsky, Bottou, Gulrajani, and Lopez-Paz]{arjovsky2019invariant}
Martin Arjovsky, L{\'e}on Bottou, Ishaan Gulrajani, and David Lopez-Paz.
\newblock Invariant risk minimization.
\newblock \emph{arXiv preprint arXiv:1907.02893}, 2019.

\bibitem[Bach(2023)]{bach2024high}
Francis Bach.
\newblock High-dimensional analysis of double descent for linear regression with random projections, 2023.

\bibitem[Belkin et~al.(2019)Belkin, Hsu, Ma, and Mandal]{belkin2019reconciling}
Mikhail Belkin, Daniel Hsu, Siyuan Ma, and Soumik Mandal.
\newblock Reconciling modern machine-learning practice and the classical bias--variance trade-off.
\newblock \emph{Proceedings of the National Academy of Sciences}, 116\penalty0 (32):\penalty0 15849--15854, 2019.

\bibitem[Bell \& Wang(2024)Bell and Wang]{bell2024multipledimensionsspuriousnessmachine}
Samuel~J. Bell and Skyler Wang.
\newblock The multiple dimensions of spuriousness in machine learning, 2024.

\bibitem[Bell \& Sagun(2023)Bell and Sagun]{bell2023simplicity}
Samuel~James Bell and Levent Sagun.
\newblock Simplicity bias leads to amplified performance disparities.
\newblock In \emph{Proceedings of the 2023 ACM Conference on Fairness, Accountability, and Transparency}, FAccT '23, pp.\  355–369, New York, NY, USA, 2023. Association for Computing Machinery.
\newblock ISBN 9798400701924.
\newblock \doi{10.1145/3593013.3594003}.

\bibitem[Bickel \& Levina(2008)Bickel and Levina]{Bickel2008RegularizedEO}
Peter~J. Bickel and Elizaveta Levina.
\newblock Regularized estimation of large covariance matrices.
\newblock \emph{Annals of Statistics}, 36:\penalty0 199--227, 2008.

\bibitem[Buolamwini \& Gebru(2018)Buolamwini and Gebru]{pmlr-v81-buolamwini18a}
Joy Buolamwini and Timnit Gebru.
\newblock Gender shades: Intersectional accuracy disparities in commercial gender classification.
\newblock In Sorelle~A. Friedler and Christo Wilson (eds.), \emph{Proceedings of the 1st Conference on Fairness, Accountability and Transparency}, volume~81 of \emph{Proceedings of Machine Learning Research}, pp.\  77--91. PMLR, 23--24 Feb 2018.

\bibitem[Caponnetto \& de~Vito(2007)Caponnetto and de~Vito]{Caponnetto2007OptimalRF}
Andrea Caponnetto and Ernesto de~Vito.
\newblock Optimal rates for the regularized least-squares algorithm.
\newblock \emph{Foundations of Computational Mathematics}, 7:\penalty0 331--368, 2007.

\bibitem[Chizat et~al.(2019)Chizat, Oyallon, and Bach]{chizat}
Lenaic Chizat, Edouard Oyallon, and Francis Bach.
\newblock On lazy training in differentiable programming.
\newblock \emph{Advances in neural information processing systems}, 32, 2019.

\bibitem[Coston et~al.(2019)Coston, Ramamurthy, Wei, Varshney, Speakman, Mustahsan, and Chakraborty]{Coston2019Missing}
Amanda Coston, Karthikeyan~Natesan Ramamurthy, Dennis Wei, Kush~R. Varshney, Skyler Speakman, Zairah Mustahsan, and Supriyo Chakraborty.
\newblock Fair transfer learning with missing protected attributes.
\newblock In \emph{Proceedings of the 2019 AAAI/ACM Conference on AI, Ethics, and Society}, AIES '19, pp.\  91–98, New York, NY, USA, 2019. Association for Computing Machinery.
\newblock ISBN 9781450363242.
\newblock \doi{10.1145/3306618.3314236}.

\bibitem[Cui et~al.(2022)Cui, Loureiro, Krzakala, and Zdeborová]{Cui_2022}
Hugo Cui, Bruno Loureiro, Florent Krzakala, and Lenka Zdeborová.
\newblock Generalization error rates in kernel regression: the crossover from the noiseless to noisy regime.
\newblock \emph{Journal of Statistical Mechanics: Theory and Experiment}, 2022\penalty0 (11):\penalty0 114004, nov 2022.

\bibitem[D'Ascoli et~al.(2020)D'Ascoli, Refinetti, Biroli, and Krzakala]{dascoli2020doubletrouble}
St{\'e}phane D'Ascoli, Maria Refinetti, Giulio Biroli, and Florent Krzakala.
\newblock Double trouble in double descent: Bias and variance(s) in the lazy regime.
\newblock In Hal~Daumé III and Aarti Singh (eds.), \emph{Proceedings of the 37th International Conference on Machine Learning}, volume 119 of \emph{Proceedings of Machine Learning Research}, pp.\  2280--2290. PMLR, 13--18 Jul 2020.

\bibitem[{de Vries} et~al.(2019){de Vries}, Misra, Wang, and {van der Maaten}]{devries2019does}
Terrance {de Vries}, Ishan Misra, Changhan Wang, and Laurens {van der Maaten}.
\newblock Does object recognition work for everyone?
\newblock In \emph{Proceedings of the {{IEEE}}/{{CVF Conference}} on {{Computer Vision}} and {{Pattern Recognition Workshops}}}, pp.\  52--59, 2019.

\bibitem[Dobriban \& Wager(2018)Dobriban and Wager]{dobriban2018high}
Edgar Dobriban and Stefan Wager.
\newblock High-dimensional asymptotics of prediction: Ridge regression and classification.
\newblock \emph{The Annals of Statistics}, 46\penalty0 (1):\penalty0 247--279, 2018.

\bibitem[Dohmatob et~al.(2024{\natexlab{a}})Dohmatob, Feng, and Kempe]{dohmatob2024Demystified}
Elvis Dohmatob, Yunzhen Feng, and Julia Kempe.
\newblock Model collapse demystified: The case of regression.
\newblock In A.~Globerson, L.~Mackey, D.~Belgrave, A.~Fan, U.~Paquet, J.~Tomczak, and C.~Zhang (eds.), \emph{Advances in Neural Information Processing Systems}, volume~37, pp.\  46979--47013. Curran Associates, Inc., 2024{\natexlab{a}}.

\bibitem[Dohmatob et~al.(2024{\natexlab{b}})Dohmatob, Feng, Subramonian, and Kempe]{dohmatob2024strong}
Elvis Dohmatob, Yunzhen Feng, Arjun Subramonian, and Julia Kempe.
\newblock Strong model collapse.
\newblock \emph{arXiv preprint arXiv:2410.04840}, 2024{\natexlab{b}}.

\bibitem[Far et~al.(2006)Far, Oraby, Bryc, and Speicher]{far2006spectra}
Reza~Rashidi Far, Tamer Oraby, Wlodzimierz Bryc, and Roland Speicher.
\newblock Spectra of large block matrices.
\newblock \emph{arXiv preprint cs/0610045}, 2006.

\bibitem[Feng et~al.(2024)Feng, Calmon, and Wang]{feng2024adapting}
Raymond Feng, Flavio Calmon, and Hao Wang.
\newblock Adapting fairness interventions to missing values.
\newblock \emph{Advances in Neural Information Processing Systems}, 36, 2024.

\bibitem[Fr{\'e}nay \& Kab{\'a}n(2014)Fr{\'e}nay and Kab{\'a}n]{Frnay2014ACI}
Beno{\^i}t Fr{\'e}nay and Ata Kab{\'a}n.
\newblock A comprehensive introduction to label noise.
\newblock In \emph{The European Symposium on Artificial Neural Networks}, 2014.

\bibitem[Geirhos et~al.(2020)Geirhos, Jacobsen, Michaelis, Zemel, Brendel, Bethge, and Wichmann]{geirhos2020shortcut}
Robert Geirhos, J{\"o}rn-Henrik Jacobsen, Claudio Michaelis, Richard Zemel, Wieland Brendel, Matthias Bethge, and Felix~A Wichmann.
\newblock Shortcut learning in deep neural networks.
\newblock \emph{Nature Machine Intelligence}, 2\penalty0 (11):\penalty0 665--673, 2020.

\bibitem[Goldt et~al.(2022)Goldt, Loureiro, Reeves, Krzakala, M{\'e}zard, and Zdeborov{\'a}]{goldt2022gaussian}
Sebastian Goldt, Bruno Loureiro, Galen Reeves, Florent Krzakala, Marc M{\'e}zard, and Lenka Zdeborov{\'a}.
\newblock The gaussian equivalence of generative models for learning with shallow neural networks.
\newblock In \emph{Mathematical and Scientific Machine Learning}, pp.\  426--471. PMLR, 2022.

\bibitem[Hall et~al.(2022)Hall, {van der Maaten}, Gustafson, Jones, and Adcock]{hall2022systematic}
Melissa Hall, Laurens {van der Maaten}, Laura Gustafson, Maxwell Jones, and Aaron Adcock.
\newblock A systematic study of bias amplification, oct 2022.

\bibitem[Hastie et~al.(2022)Hastie, Montanari, Rosset, and Tibshirani]{hastie2022surprises}
Trevor Hastie, Andrea Montanari, Saharon Rosset, and Ryan~J Tibshirani.
\newblock Surprises in high-dimensional ridgeless least squares interpolation.
\newblock \emph{Annals of statistics}, 50\penalty0 (2):\penalty0 949, 2022.

\bibitem[Hendricks et~al.(2018)Hendricks, Burns, Saenko, Darrell, and Rohrbach]{hendricks2018women}
Lisa~Anne Hendricks, Kaylee Burns, Kate Saenko, Trevor Darrell, and Anna Rohrbach.
\newblock Women also snowboard: Overcoming bias in captioning models.
\newblock In \emph{Proceedings of the {{European Conference}} on {{Computer Vision}} ({{ECCV}})}, pp.\  771--787, 2018.

\bibitem[Jacot et~al.(2018)Jacot, Gabriel, and Hongler]{NEURIPS2018_Jacot}
Arthur Jacot, Franck Gabriel, and Clement Hongler.
\newblock Neural tangent kernel: Convergence and generalization in neural networks.
\newblock In S.~Bengio, H.~Wallach, H.~Larochelle, K.~Grauman, N.~Cesa-Bianchi, and R.~Garnett (eds.), \emph{Advances in Neural Information Processing Systems}, volume~31. Curran Associates, Inc., 2018.

\bibitem[Jain et~al.(2024)Jain, Nobahari, Baratin, and Sarao~Mannelli]{jain2024bias}
Anchit Jain, Rozhin Nobahari, Aristide Baratin, and Stefano Sarao~Mannelli.
\newblock Bias in motion: Theoretical insights into the dynamics of bias in sgd training.
\newblock In A.~Globerson, L.~Mackey, D.~Belgrave, A.~Fan, U.~Paquet, J.~Tomczak, and C.~Zhang (eds.), \emph{Advances in Neural Information Processing Systems}, volume~37, pp.\  24435--24471. Curran Associates, Inc., 2024.

\bibitem[Kargin(2015)]{Kargin2015Subordination}
V.~Kargin.
\newblock Subordination for the sum of two random matrices.
\newblock \emph{The Annals of Probability}, 43\penalty0 (4):\penalty0 2119 -- 2150, 2015.

\bibitem[Khani \& Liang(2021)Khani and Liang]{Khani2021Spurious}
Fereshte Khani and Percy Liang.
\newblock Removing spurious features can hurt accuracy and affect groups disproportionately.
\newblock In \emph{Proceedings of the 2021 ACM Conference on Fairness, Accountability, and Transparency}, FAccT '21, pp.\  196–205, New York, NY, USA, 2021. Association for Computing Machinery.
\newblock ISBN 9781450383097.
\newblock \doi{10.1145/3442188.3445883}.

\bibitem[Lee et~al.(2023)Lee, Moniri, Huang, Dobriban, and Hassani]{lee2023demystifying}
Donghwan Lee, Behrad Moniri, Xinmeng Huang, Edgar Dobriban, and Hamed Hassani.
\newblock Demystifying disagreement-on-the-line in high dimensions.
\newblock In \emph{International Conference on Machine Learning}, pp.\  19053--19093. PMLR, 2023.

\bibitem[Leino et~al.(2019)Leino, Fredrikson, Black, Sen, and Datta]{leino2018feature}
Klas Leino, Matt Fredrikson, Emily Black, Shayak Sen, and Anupam Datta.
\newblock Feature-wise bias amplification.
\newblock In \emph{International Conference on Learning Representations}, 2019.

\bibitem[Maloney et~al.(2022)Maloney, Roberts, and Sully]{maloney2022solvable}
Alexander Maloney, Daniel~A Roberts, and James Sully.
\newblock A solvable model of neural scaling laws.
\newblock \emph{arXiv preprint arXiv:2210.16859}, 2022.

\bibitem[Mannelli et~al.(2024)Mannelli, Gerace, Rostamzadeh, and Saglietti]{mannelli2022unfair}
Stefano~Sarao Mannelli, Federica Gerace, Negar Rostamzadeh, and Luca Saglietti.
\newblock Bias-inducing geometries: exactly solvable data model with fairness implications.
\newblock In \emph{ICML 2024 Workshop on Geometry-grounded Representation Learning and Generative Modeling}, 2024.

\bibitem[Marčenko \& Pastur(1967)Marčenko and Pastur]{MarcenkoPastur}
V.A. Marčenko and Leonid Pastur.
\newblock Distribution of eigenvalues for some sets of random matrices.
\newblock \emph{Math USSR Sb}, 1:\penalty0 457--483, 1967.

\bibitem[Mingo \& Speicher(2017)Mingo and Speicher]{mingo2017free}
James~A. Mingo and Roland Speicher.
\newblock \emph{Free Probability and Random Matrices}, volume~35 of \emph{Fields Institute Monographs}.
\newblock Springer, 2017.

\bibitem[Ovalle et~al.(2023)Ovalle, Goyal, Dhamala, Jaggers, Chang, Galstyan, Zemel, and Gupta]{Ovalle2023TGNB}
Anaelia Ovalle, Palash Goyal, Jwala Dhamala, Zachary Jaggers, Kai-Wei Chang, Aram Galstyan, Richard Zemel, and Rahul Gupta.
\newblock “i’m fully who i am”: Towards centering transgender and non-binary voices to measure biases in open language generation.
\newblock In \emph{Proceedings of the 2023 ACM Conference on Fairness, Accountability, and Transparency}, FAccT '23, pp.\  1246–1266, New York, NY, USA, 2023. Association for Computing Machinery.
\newblock ISBN 9798400701924.
\newblock \doi{10.1145/3593013.3594078}.

\bibitem[Richards et~al.(2021)Richards, Mourtada, and Rosasco]{richards2021asymptotics}
Dominic Richards, Jaouad Mourtada, and Lorenzo Rosasco.
\newblock Asymptotics of ridge (less) regression under general source condition.
\newblock In \emph{International Conference on Artificial Intelligence and Statistics}, pp.\  3889--3897. PMLR, 2021.

\bibitem[Richards et~al.(2024)Richards, Kirichenko, Bouchacourt, and Ibrahim]{richards2023does}
Megan Richards, Polina Kirichenko, Diane Bouchacourt, and Mark Ibrahim.
\newblock Does progress on object recognition benchmarks improve generalization on crowdsourced, global data?
\newblock In \emph{The Twelfth International Conference on Learning Representations}, 2024.

\bibitem[Sagawa et~al.(2020)Sagawa, Raghunathan, Koh, and Liang]{sagawa2020investigation}
Shiori Sagawa, Aditi Raghunathan, Pang~Wei Koh, and Percy Liang.
\newblock An investigation of why overparameterization exacerbates spurious correlations.
\newblock In \emph{International Conference on Machine Learning}, pp.\  8346--8356. PMLR, 2020.

\bibitem[Spigler et~al.(2019)Spigler, Geiger, d’Ascoli, Sagun, Biroli, and Wyart]{spigler2019jamming}
Stefano Spigler, Mario Geiger, St{\'e}phane d’Ascoli, Levent Sagun, Giulio Biroli, and Matthieu Wyart.
\newblock A jamming transition from under-to over-parametrization affects generalization in deep learning.
\newblock \emph{Journal of Physics A: Mathematical and Theoretical}, 52\penalty0 (47):\penalty0 474001, 2019.

\bibitem[Tripuraneni et~al.(2021)Tripuraneni, Adlam, and Pennington]{tripuraneni2021covariate}
Nilesh Tripuraneni, Ben Adlam, and Jeffrey Pennington.
\newblock Overparameterization improves robustness to covariate shift in high dimensions.
\newblock In M.~Ranzato, A.~Beygelzimer, Y.~Dauphin, P.S. Liang, and J.~Wortman Vaughan (eds.), \emph{Advances in Neural Information Processing Systems}, volume~34, pp.\  13883--13897. Curran Associates, Inc., 2021.

\bibitem[Wang \& Russakovsky(2021)Wang and Russakovsky]{wang2021directionala}
Angelina Wang and Olga Russakovsky.
\newblock Directional bias amplification.
\newblock In Marina Meila and Tong Zhang (eds.), \emph{Proceedings of the 38th International Conference on Machine Learning}, volume 139 of \emph{Proceedings of Machine Learning Research}, pp.\  10882--10893. PMLR, 18--24 Jul 2021.

\bibitem[Wyllie et~al.(2024)Wyllie, Shumailov, and Papernot]{wyllie2024fairness}
Sierra Wyllie, Ilia Shumailov, and Nicolas Papernot.
\newblock Fairness feedback loops: training on synthetic data amplifies bias.
\newblock In \emph{The 2024 ACM Conference on Fairness, Accountability, and Transparency}, FAccT '24, pp.\  2113--2147, New York, NY, USA, 2024. Association for Computing Machinery.
\newblock ISBN 9798400704505.
\newblock \doi{10.1145/3630106.3659029}.

\bibitem[Yehudai \& Shamir(2019)Yehudai and Shamir]{Yehudai2019RandomFeatures}
Gilad Yehudai and Ohad Shamir.
\newblock On the power and limitations of random features for understanding neural networks.
\newblock In H.~Wallach, H.~Larochelle, A.~Beygelzimer, F.~d\textquotesingle Alch\'{e}-Buc, E.~Fox, and R.~Garnett (eds.), \emph{Advances in Neural Information Processing Systems}, volume~32. Curran Associates, Inc., 2019.

\bibitem[Zech et~al.(2018)Zech, Badgeley, Liu, Costa, Titano, and Oermann]{zech2018variable}
John~R. Zech, Marcus~A. Badgeley, Manway Liu, Anthony~B. Costa, Joseph~J. Titano, and Eric~Karl Oermann.
\newblock Variable generalization performance of a deep learning model to detect pneumonia in chest radiographs: A cross-sectional study.
\newblock \emph{PLOS Medicine}, 15\penalty0 (11):\penalty0 e1002683, nov 2018.
\newblock ISSN 1549-1676.
\newblock \doi{10.1371/journal.pmed.1002683}.

\bibitem[Zhao et~al.(2017)Zhao, Wang, Yatskar, Ordonez, and Chang]{zhao-etal-2017-men}
Jieyu Zhao, Tianlu Wang, Mark Yatskar, Vicente Ordonez, and Kai-Wei Chang.
\newblock Men also like shopping: Reducing gender bias amplification using corpus-level constraints.
\newblock In Martha Palmer, Rebecca Hwa, and Sebastian Riedel (eds.), \emph{Proceedings of the 2017 Conference on Empirical Methods in Natural Language Processing}, pp.\  2979--2989, Copenhagen, Denmark, sep 2017. Association for Computational Linguistics.
\newblock \doi{10.18653/v1/D17-1323}.

\end{thebibliography}
\bibliographystyle{iclr2025_conference}

\clearpage

\appendix
\addcontentsline{toc}{section}{Appendix} %
\part{Appendix} %
\parttoc %

\clearpage

\section{Warm-up: Deriving Marchenko-Pastur Law via Operator-Valued Free Probability Theory}
\label{sec:ovfpt}
\addcontentsline{app}{section}{Deriving Marchenko-Pastur Law via Operator-Valued Free Probability Theory}

We provide a detailed example of how to apply linear pencils and operator-valued free probability theory (OVFPT) to derive the classical Marchenko-Pastur (MP) law \citep{MarcenkoPastur}. Let $S = (1/n)X^\top X \in \mathbb R^{d \times d}$ be the empirical covariance matrix for an $n \times d$ random matrix $X$ with IID entries from ${\cal N} (0,1)$. If $n$ tends to infinity while $d$ is held fixed, then $S$ converges to the population covariance matrix, here $\Sigma = I_d$. If $d$ also tends to infinity, then the limit seizes to exist. It turns out that one can still make sense of the limiting distribution of eigenvalues of $S$ in the case $d/n$ stays constant, i.e.,
\begin{eqnarray}
n,d \to \infty,\, d/n \to \gamma \in (0,\infty).  
\label{eq:asymptotic}
\end{eqnarray}

In particular, we seek to understand the behavior of the random histogram:
\begin{eqnarray}
\widehat \mu_n = \frac{1}{d}\sum_{i=1}^n \widehat \lambda_i,   
\end{eqnarray}
where $\widehat \lambda_1,\ldots,\widehat \lambda_d$ are the eigenvalues of $S$. In the aforementioned limit,
$\widehat \mu_n$ converges to a deterministic law $\mu_{\text{MP}}$ on $\mathbb R$ called the MP law. This is central to the field of random matrix theory (RMT), a primary tool in probability theory, statistical analysis of neural networks, finance, etc. We are interested in an even more powerful tool -- free probability theory (FPT) -- which is powerful enough to give a precise picture of deep learning in certain linearized regimes (e.g., random features, NTK) and interesting phenomena (e.g., triple descent) via analytic calculation.

\subsection{Step 1: Constructing a Linear Pencil}
For any positive $\lambda$,  consider the $2 \times 2$ block matrix $Q$ defined by:
\begin{equation}
    Q =
    \begin{pmatrix}
    I_n & -\frac{X}{\sqrt{n \lambda}}  \\
    \frac{X^\top}{\sqrt{n \lambda}} & I_d
    \end{pmatrix}.
\end{equation}
Let $\ntrace$ be the normalized trace operator on square matrices and set $\varphi = \mathbb E\circ \ntrace$. This gives random $(n+d) \times (n+d)$ matrices the structure of a von Neumann algebra $\mathcal A$. Define a $2 \times 2$ matrix $G=G(Q)$ by:
\begin{equation}
G = \mathbb (\id_2 \otimes \varphi)Q^{-1},\text{ i.e }g_{i,j} = \varphi ([Q^{-1}]_{i,j}) = [\varphi (Q^{-1})]_{i,j}\text{ for all }i,j \in \{1,2\}.
\end{equation}
\begin{mdframed}
Thus, the operator $(\id_2 \otimes \varphi)Q^{-1}$ extracts the expectation of the normalized trace of the blocks of the inverse of the a $2 \times 2$ block matrix $Q$.
\end{mdframed}

Observe that:
\begin{eqnarray}
\mathbb E\,\ntrace (S+\lambda I_d)^{-1} = \frac{g_{2,2}}{\lambda}.
\label{eq:core}
\end{eqnarray}
This is a direct consequence of inverting a $2 \times 2$ block matrix (namely Schur's complement).
The mechanical advantage of Equation \ref{eq:core} is that the \emph{resolvent} $(S+\lambda I_d)^{-1}$ depends quadratically on $X$ while $g_{2,2}$ is defined via $Q$, which is linear in $X$. For this reason, $Q$ is called a \emph{linear pencil} for $(S+\lambda I_d)^{-1}$. The construction of appropriate linear pencils for rational functions of random matrices is a crucial step in leveraging FPT.

\subsection{Step 2: Constructing the Fundamental Equation via Freeness}
For any $B \in M_b(\mathbb C)^+$, define a block matrix $B \otimes 1_{\mathcal A}$ by:
\begin{eqnarray}
[B \otimes 1_{\mathcal A}]_{ij} = 
\begin{cases}
    b_{ij}I_{d_i},&\mbox{ if }d_i = d_j\\
    0,&\mbox{ else}
\end{cases}.
\end{eqnarray}
Here, $b\times b$ is the number of blocks in the linear pencil $Q_X$, that is, $b=2$. Now, observe that we can write $Q = F - Q_X$, where:
\begin{align}
F=\begin{pmatrix}
    I_d & 0\\
    0 & I_n
\end{pmatrix} = I_2 \otimes 1_{\mathcal A}\text{ and }
Q_X = 
\begin{pmatrix}
    0 & \frac{X}{\sqrt{n \lambda}} \\
    -\frac{X^\top}{\sqrt{n \lambda}} & 0
\end{pmatrix}.
\end{align}
One can then express $G = (\id_b \otimes \varphi) Q^{-1} = (\id_b \otimes \varphi)(F -  Q_X)^{-1}$.
From operator-valued FPT, we know that in the proportionate scaling limit given by Equation \ref{eq:asymptotic}, the following fixed-point equation (due to the asymptotic freeness of $Q_X$ and $F$) is satisfied by $G$:
\begin{equation}
     G = (\id_b \otimes \varphi)(F-R \otimes 1_{\mathcal A})^{-1},
    \label{eq:fp}
\end{equation}
where $R=\mathcal R_{Q_X}(G)$, and $R_{Q_X}$ is the R-transform of $Q_X$ which maps $M_b(\mathbb C)^+$ to itself like so:
\begin{equation}
    \mathcal R_{Q_X}(B)_{ij} = \sum_{k,\ell}\sigma(i,k;\ell ,j)\alpha_k b_{k\ell}.
    \label{eq:Rtransform}
\end{equation}
Here, $\sigma(i,k;\ell,j)$ is the covariance between the entries of block $(i,k)$ and block $(\ell,j)$ of $Q_X$, while $\alpha_k$ is the dimension of the block $(k,\ell)$.

\subsection{Step 3: The Final Calculation}
By the structure of $Q_X$, one can compute from Equation \ref{eq:Rtransform}:
\begin{align}
    r_{1,1} &= d\cdot \frac{-1}{n \lambda} g_{2,2}  =-\frac{\gamma}{\lambda}g_{2,2},\\
    r_{1,2} &= 0,\\
    r_{2,1} &= 0,\\
    r_{2,2} &= n\cdot \frac{-1}{n \lambda} g_{1,1} = -\frac{1}{\lambda}g_{1,1}.
\end{align}
Combining this with Equation \ref{eq:fp}, one has:
\begin{equation}
    \begin{split}
G = (\id_2 \otimes \varphi)(Z-R \otimes 1_{\mathcal A})^{-1} = (I_2 - R)^{-1} &= \begin{pmatrix}
    1+(\gamma/\lambda) g_{2,2} & 0\\
    0 & 1+g_{2,2}/\lambda
\end{pmatrix}^{-1}\\
&=
\begin{pmatrix}
    \lambda/(\lambda + \gamma g_{2,2}) & 0\\
    0 & \lambda/(\lambda +g_{1,1})
\end{pmatrix}.       
    \end{split}
\end{equation}
Comparing the matrix entries, this translates to the following scalar equations:
\begin{align}
    g_{1,1} &= \frac{\lambda}{\lambda +\gamma g_{2,2}},\\
    g_{2,2} &= \frac{\lambda}{\lambda +g_{1,1}},\\\
    g_{2,1} &= g_{1,2} = 0.
\end{align}
Plugging the second equation into the first (to eliminate $g_{1,1}$) gives:
$$
g_{2,2} = \dfrac{\lambda}{\lambda + \lambda/(\lambda+\gamma g_{2,2})}.
$$
Setting $m=g_{2,2}/\lambda$ then gives $m=(\lambda + 1/(1+\gamma m))^{-1}$, i.e.,
\begin{eqnarray}
\frac{1}{m} = \lambda + \frac{1}{1+\gamma m},
\end{eqnarray}
which is precisely the functional equation characterizing the Stieltjes transform (evaluated at $\lambda=-z$) of the MP law with shape parameter $\gamma$. By treating $\lambda$ as a complex number and applying the Cauchy-inversion formula, we can recover $\mu_{\text{MP}}$.

\clearpage

\section{Technical Assumptions}
\label{sec:assumptions}

\begin{assumption}
In the case of classical ridge regression, we will work in the following proportionate scaling limit:
\begin{eqnarray}
n,n_1,n_2,d \to \infty,\quad  n_1/n \to p_1,\,n_2/n \to p_2,\quad d/n_1 \to \phi_1,\,d/n_2 \to \phi_2,\,d/n \to \phi,
\label{eq:proportionate}
\end{eqnarray}
for some constants $\phi_1,\phi_2,\phi \in (0,\infty)$. The scalar $\phi$ captures the rate of features to samples. Observe that $\phi=p_1 \phi_1$ and $\phi = p_2 \phi_2$.
\label{ass:scaling}
\end{assumption}

\begin{assumption}
    The per-group covariance matrices $\Sigma_1$ and $\Sigma_2$ and ground-truth weight covariance matrices $\Theta$ and $\Delta$ are all simultaneously diagonalizable; hence, all these matrices commute.
    \label{ass:commute}
\end{assumption}

While Assumption \ref{ass:commute} may appear reductive, our goal is to analyze the bias amplification phenomenon in a sufficient setting that does not introduce complexities due to non-commutativity. Notably, our main theoretical result does not assume isotropic covariance. For example, our theory accommodates diatomic covariance (see Section \ref{sec:min-group-exp}) and power-law covariance (see Appendix \ref{sec:bias-amp-pow-setup}).

\begin{assumption}
In Corollary \ref{corr:bias-amp-phase-transitions}, we assume the following spectral densities exist when $d \to \infty$:
\begin{itemize}
    \item $\nu \in \mathcal P(\mathbb R_+)$ is the limiting spectral density of $\Sigma_2\Sigma_1^{-1}$, of the ratios $\lambda_j^{(2)}/\lambda_j^{(1)}$ of the eigenvalues of the respective covariance matrices,
    \item $\mu \in \mathcal P(\mathbb R_+, \mathbb R_+)$ is the joint limiting density of the spectra of $\Sigma_2\Sigma_1^{-1}$ and $\Sigma_1$,
    \item $\pi \in \mathcal P(\mathbb R_+)$ is the limiting density of the spectrum of $\Delta$.
\end{itemize}
    \label{ass:spectral}
\end{assumption}

\clearpage

\section{Related Work (Continued)}
\label{sec:rw-cont}

\paragraph{High-dimensional analysis of bias.} \citet{mannelli2022unfair} employ the replica method, which is non-rigorous, while we use OVFPT, which is entirely rigorous. Moreover, \citet{mannelli2022unfair, jain2024bias} study the application of linear classification to Gaussian data with isotropic covariance; in contrast, we study the application of regression with random projections (a simplified model of feedforward neural networks) to Gaussian data with more general covariance structure (i.e., covariance matrices that are simultaneously diagonalizable) and noisy labels. This allows us to analyze the effects of these additional factors on bias. We make additional connections between our work and \citet{mannelli2022unfair, jain2024bias} in Section \ref{sec:reg-training-dynamics}.

\paragraph{Bias amplification metrics.} Our definition of $ADD$ is consistent with the conceptualization of bias of \citet{bell2023simplicity}. At a high level, our definition quantifies how many times worse model bias would be if a ML practitioner opted to train a single model on a mixture of data from two groups (i.e., the setting in which bias is observed in practice) vs. separate models for the data from each group (i.e., the setting which corresponds to the bias in the data alone, and thus the a priori amount of bias we would expect in the case of a single model). In sum, we seek to isolate the contribution of the {\em model} to bias when learning from data with different groups.

\clearpage

\section{Warm-Up: Classical Linear Model}
\label{sec:warm-up-classical}

\textbf{Technical Difficulty.}
The analysis of the test errors (e.g., $R_s(\widehat f)$) amounts to the analysis of the trace of rational functions of sums of random matrices. Although the limiting spectral density of sums of random matrices is a classical computation using subordination techniques \citep{MarcenkoPastur, Kargin2015Subordination}, a more involved analysis is required in our case. This difficulty is even greater in the setting of random projections (see Section \ref{sec:randproj}). Thus, we employ OVFPT to compute the exact high-dimensional limits of such quantities. We derive Theorems \ref{thm:edd-theory} and \ref{thm:odd-theory} using OVFPT (in Appendices \ref{sec:odd-proof} and \ref{sec:edd-proof}). Theorem \ref{thm:odd-theory} is a non-trivial generalization of Proposition 3 from \citep{bach2024high}, which can be recovered by taking $p_s \to 1$ (i.e., $p_{s'} \to 0$).

\subsection{Single Model Learned for Both Groups}

We first consider the classical ridge regression model $\widehat f$, which is learned using empirical risk minimization and $\ell_2$-regularization with penalty $\lambda$. The parameter vector $\widehat w \in \mathbb R^d$ of the linear model $\widehat f$ is given by the following %
problem:
\begin{eqnarray}
\widehat w = \arg\min_{w \in \mathbb R^d} L(w)= \sum_{s=1}^2n^{-1}\|X_s w -Y_s\|^2_2 + \lambda \|w\|_2^2.
\label{eq:classical-ridge}
\end{eqnarray}
The unregularized limit  $\lambda \to 0^+$ corresponds to ordinary least-squares (OLS). %
We provide in Theorem \ref{thm:odd-theory} a novel bias-variance decomposition for the test error $R_s(\widehat f)$ for each group $s\in\{1,2\}$. We first present some relevant definitions.

\begin{definition}
    For any group index $s \in \{1,2\}$, we define $(e_1, e_2, u^{(s)}_1, u^{(s)}_2)$ to be the unique positive solution to the following system of fixed-point equations:
    \begin{align}
    1/e_s &= 1+\phi \ntrace \Sigma_s K^{-1},\quad u^{(s)}_k = \phi e_k^2 \ntrace  \Sigma_k(p_1 u^{(s)}_1  \Sigma_1 + p_2 u^{(s)}_2  \Sigma_2 + \Sigma_s) K^{-2},\, k \in \{1,2\},
\end{align}
where $K = p_1 e_1 \Sigma_1 + p_2 e_2 \Sigma_2 + \lambda I_d$ and $\ntrace A := (1/d)\trace A$ is the normalized trace operator.
\end{definition}
The fixed-point equations for $e_s$ are non-linear and often not analytically solvable for general $\Sigma_1, \Sigma_2$. This is typical in RMT.

\begin{theorem}%
Under Assumptions \ref{ass:commute} and \ref{ass:scaling}, it holds that: $ R_s (\widehat f) \simeq B_s (\widehat f) + V_s (\widehat f)$, with
    \begin{align}
    V_s (\widehat f) &= V_s^{(1)} (\widehat f) + V_s^{(2)} (\widehat f), \\
    V_s^{(k)} (\widehat f) &= p_k  \sigma_k^2 \phi \ntrace \Sigma_k \big(e_k \Sigma_s - \lambda u^{(s)}_k I_d + p_{k'} \Sigma_{k'} (e_k u^{(s)}_{k'} - e_{k'} u^{(s)}_k) \big) K^{-2}, \\
    B_s (\widehat f) &= B_s^{(1)} (\widehat f) + B_s^{(3)} (\widehat f) + \begin{cases}
        0, & s = 1, \\
        2B_2^{(2)} (\widehat f), & s= 2,
    \end{cases} \\
    B_s^{(1)} (\widehat f) &= p_{s'}\ntrace  \Delta \Sigma_{s'}  (p_{s'} (1 + p_s u^{(s)}_s) e_{s'}^2 \Sigma_{s'}  \Sigma_s + u^{(s)}_{s'} (p_s e_s \Sigma_s + \lambda I_d)^2) K^{-2}, \\
    B_2^{(2)} (\widehat f) &= p_1\lambda\ntrace\Sigma_1((1+p_2 u_2^{(2)})e_1\Sigma_2-u_1^{(2)}(p_2e_2\Sigma_2+\lambda I_d))K^{-2}, \\
    B_s^{(3)} (\widehat f) &= \lambda^2 \ntrace \Theta_s (p_1 u_1^{(s)} \Sigma_1 + p_2 u_2^{(s)} \Sigma_2 + \Sigma_s) K^{-2},
    \end{align}
    where $1'=2$ and $2'=1$.
\label{thm:odd-theory}
\end{theorem}

\subsection{Separate Model Learned Per Group}

We now treat the case of fitting a separate model $\widehat f_s$ per group. Suppose that the classical ridge regression models $\widehat f_1$ and $\widehat f_2$ are learned using empirical risk minimization and $\ell_2$-regularization with penalties $\lambda_1$ and $\lambda_2$, respectively. In particular, we have the following optimization problem for each group $s$:
\begin{eqnarray}
    \arg\min_{w \in \mathbb R^d} L(w)= \frac{1}{n_s}\sum_{(x_i,y_i) \in \mathcal D^s}(x_i^\top w-y_i)^2 + \lambda_s\|w\|_2^2 = \frac{\|X_s w -Y_s\|^2_2}{n_s} + \lambda_s\|w\|_2^2.
\end{eqnarray}

We first present some relevant definitions.

\begin{definition}
\label{def:kappa}
    Let $\ndof_m^{(s)} (t) = \ntrace \Sigma_s^m \left(\Sigma_s + t I_d \right)^{-m}$, and $\kappa_s$ be the unique positive solution to the equation
    $     \kappa_s - \lambda_s =  \kappa_s \phi_s \ndof^{(s)}_1 (\kappa_s).
    $
\end{definition}

In this setting, we deduce Theorem \ref{thm:edd-theory}.

\begin{theorem}
Under Assumptions \ref{ass:commute} and \ref{ass:scaling}, it holds that:
    \begin{eqnarray}
        R_s (\widehat f_s) \simeq B_s (\widehat f_s) + V_s (\widehat f_s),\text{ with }
    \end{eqnarray}
    \begin{align}
        V_s (\widehat f_s) = \frac{\sigma_s^2 \phi_s \ndof_2^{(s)} (\kappa_s)}{1 - \phi_s \ndof_2^{(s)} (\kappa_s)},\, B_s (\widehat f_s) &= \frac{\kappa_s^2 \ntrace \Theta_s \Sigma_s \left(\Sigma_s + \kappa_s I_d\right)^{-2} }{1 - \phi_s \ndof_2^{(s)} (\kappa_s)}.
        \label{eq:u-rho-lin-3}
    \end{align}
\label{thm:edd-theory}
\end{theorem}

\subsection{Phase Diagram}
We present the bias amplification phase diagram (Figure \ref{fig:phase-diags-vanilla}) predicted by Theorems \ref{thm:odd-theory} and \ref{thm:edd-theory} for the classical ridge regression model. The phase diagram offers insights into how $\phi$ (rate of features to samples) affects bias amplification. To obtain the precise phase diagram, we solve the scalar equations numerically. In the $ODD$ profile, we observe an interpolation threshold at $\phi = 1$. To the right of the threshold, we observe a tail that descends towards 1. To the left of the threshold, the $ODD$ descends below 1 with a local minimum at $\phi \approx 0.25$ before increasing. In contrast, we observe that the $EDD$ continually grows as $\phi$ increases, ascending from a small value, exhibiting an inflection point at $\phi = 0.5$, and plateauing after $\phi = 1$. Accordingly, the $ADD$ increases significantly as $\phi$ decreases (with an intermediate inflection point at $\phi = 0.5$), peaks at $\phi = 1$, and descends towards 1 as $\phi$ increases (i.e., bias remains amplified in this phase). In sum, bias is most amplified when the rate of features to samples $\phi \ll 1$ and $\phi = 1$. Interestingly, bias amplification consistently occurs (i.e., $ADD > 1$) across all observed values of $\phi$.

\begin{figure}[!ht]
    \centering
    \includegraphics[width=0.9\linewidth]{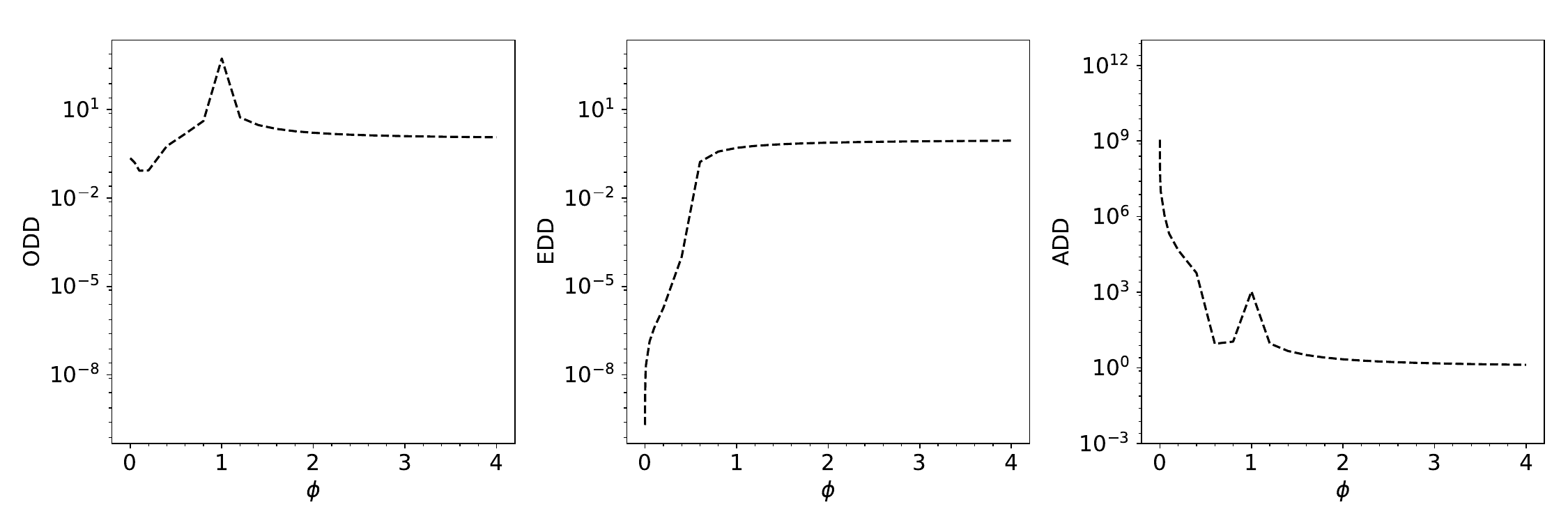}
    \caption{\textbf{$ODD$, $EDD$, and $ADD$ phase diagrams for classical ridge regression.} We plot the bias amplification phase diagrams with respect to $\phi$ (rate of features to samples), as predicted by our theory for ridge regression without random projections (Theorems \ref{thm:odd-theory}, \ref{thm:edd-theory}). Dashed black lines indicate theoretical predictions. We consider isotropic covariance matrices: $\Sigma_1 = 2 I_d, \Sigma_2 = I_d$, $\Theta = 2 I_d$, $\Delta = I_d$. Additionally, $n = 1 \times 10^4, \sigma_1^2 = \sigma_2^2 = 1$. We further choose $\lambda = \lambda_1 = \lambda_2 = 1 \times 10^{-6}$ to approximate the minimum-norm interpolator. We observe that bias amplification can occur even in the balanced data setting, i.e., when $p_1 = p_2 = 1/2$, without spurious correlations.}
    \label{fig:phase-diags-vanilla}
\end{figure}

\clearpage

\section{Proof of Theorem \ref{thm:edd-theory}}
\label{sec:edd-proof}

\begin{proof}
We define $M_s = X_s^\top X_s$ and $E_s = Y_s - X_sw^*_s$. Note that $\widehat w_s=(X_s^\top X_s + n_s \lambda_s I_d)^{-1}X_s^\top (X_sw^*_s + E_s) = (M_s + n_s \lambda_s I_d)^{-1} M_s w^*_s + (M_s + n_s \lambda_s I_d)^{-1} X_s^\top E_s$. We deduce that $R_s(\widehat f_s) = B_s(\widehat f_s) + V_s(\widehat f_s)$, where:
\begin{align}
B_s(\widehat f_s) &= \mathbb E\,\|(M_s + n_s \lambda_s I_d)^{-1} M_s w^*_s -w^*_s\|_{\Sigma_s}^2,\\
V_s(\widehat f_s) &=  \mathbb E\,\|(M_s + n_s \lambda_s I_d)^{-1} X_s^\top E_s \|_{\Sigma_s}^2.
\end{align}

\subsection{Variance Term}

Note that the variance term $V_s(\widehat f)$ of the test error of $\widehat f_s$ evaluated on group $s$ is given by: 
\begin{align}
    V_s(\widehat f_s) &=  \sigma_s^2\mathbb E\, \trace X_s (M_s + n_s \lambda_s I_d)^{-1} \Sigma_s (M_s + n_s \lambda_s I_d)^{-1} X_s^\top \\
    &= \sigma_s^2\mathbb E\, \trace (M_s + n_s \lambda_s I_d)^{-1} M_s (M_s + n_s \lambda_s I_d)^{-1} \Sigma_s.
\end{align}

We can re-express this as:
\begin{align}
n_s V_s(\widehat f_s) &= \sigma_s^2 \mathbb E\trace (H_s+\lambda_s I_d)^{-1} H_s (H_s+\lambda_s I_d)^{-1} \Sigma_s \\
&= \frac{\sigma_s^2}{\lambda_s} \mathbb E\trace (H_s/\lambda_s + I_d)^{-1} (H_s/\lambda_s) (H_s/\lambda_s+ I_d)^{-1} \Sigma_s,
\end{align}
where $H_s = X_s^\top X_s/n_s$ and $X_s=Z_s \Sigma_s^{1/2}$, with $Z_1 \in \mathbb R^{n_1 \times d}$ and $Z_2 \in \mathbb R^{n_2 \times d}$ being independent random matrices with IID entries from ${\cal N} (0, 1)$. Thus, the variance term is proportional to:
\begin{align}
\ntrace (H_s +\lambda_s I_d)^{-1} H_s(H_s+\lambda_s I_d)^{-1}\Sigma_s.
\end{align}
WLOG, we consider the case where $s = 1$. The matrix of interest has a linear pencil representation given by (with zero-based indexing):
\begin{align}
(H_1 / \lambda_1 + I_d)^{-1} (H_1 / \lambda_1) (H_1 / \lambda_1 + I_d)^{-1} \Sigma_1 = Q^{-1}_{0, 8},
\end{align}
where the linear pencil $Q$ is defined as follows:
\begin{gather}
Q = 
\resizebox{0.8\linewidth}{!}{%
$\left(\begin{array}{ccccccccc}
I_{d} & \Sigma_{1}^{\frac{1}{2}} & 0 & 0 & -\Sigma_{1}^{\frac{1}{2}} & 0 & 0 & 0 & 0 \\
0 & I_{d} & -\frac{1}{\sqrt{\lambda_1} \sqrt{n_{1}}} Z_{1}^\top & 0 & 0 & 0 & 0 & 0 & 0 \\
0 & 0 & I_{n_1} & -\frac{1}{\sqrt{\lambda_1} \sqrt{n_{1}}} Z_{1} & 0 & 0 & 0 & 0 & 0 \\
-\Sigma_{1}^{\frac{1}{2}} & 0 & 0 & I_{d} & 0 & 0 & 0 & 0 & 0 \\
0 & 0 & 0 & 0 & I_{d} & -\frac{1}{\sqrt{\lambda_1} \sqrt{n_{1}}} Z_{1}^\top & 0 & 0 & 0 \\
0 & 0 & 0 & 0 & 0 & I_{n_1} & -\frac{1}{\sqrt{\lambda_1} \sqrt{n_{1}}} Z_{1} & 0 & 0 \\
0 & 0 & 0 & 0 & 0 & 0 & I_{d} & -\Sigma_{1}^{\frac{1}{2}} & 0 \\
0 & 0 & 0 & 0 & \Sigma_{1}^{\frac{1}{2}} & 0 & 0 & I_{d} & -\Sigma_1 \\
0 & 0 & 0 & 0 & 0 & 0 & 0 & 0 & I_{d} \\
\end{array}\right)$}.
\end{gather}
We compute $Q$ using the \texttt{NCMinimalDescriptorRealization} function of the NCAlgebra library\footnote{\url{https://github.com/NCAlgebra/NC}}. We further symmetrize $Q$ by constructing the self-adjoint matrix $\overline{Q}$:
\begin{gather}
    \overline{Q} = \left(\begin{array}{cc} 0 & Q^\top \\ Q & 0 \end{array}\right).
\end{gather}
This enables us to apply known formulae for the $R$-transform of Gaussian block matrices \citep{far2006spectra}. We note that $\overline Q^{-1}_{0, 17} = Q^{-1}_{0, 8}$. Taking similar steps as \citet{lee2023demystifying}, we use OVFPT on $\overline Q$. Let $G =  (I_{18} \otimes \mathbb E\,\ntrace) \overline Q^{-1} \in \mathbb R^{18 \times 18}$ be the matrix whose entries are  normalized traces of blocks\footnote{By convention, the trace of a non-square block is zero.} of $\overline{Q}^{-1}$. We provide a detailed example of how to apply OVFPT to derive the MP law in Appendix \ref{sec:ovfpt}.
One can arrive at that, in the asymptotic limit given by Equation \ref{eq:proportionate}, the following holds:
\begin{equation}
\begin{split}
&\E \ntrace (H_1 + \lambda_1 I_d)^{-1} H_1 (H_1 + \lambda_1 I_d)^{-1} \Sigma_1 = \frac{G_{0, 17}}{\lambda_1},\\
&\text{ with }\frac{G_{0, 17}}{\lambda_1} = (G_{5, 14} - G_{2, 14}) \ntrace \left(\Sigma_1 G_{2, 11} + \lambda_1 I_d \right)^{-1} \Sigma_1 \left(\Sigma_1 G_{5, 14} + \lambda_1 I_d \right)^{-1} \Sigma_1.
\end{split}
\end{equation}
We will now obtain the fixed-point equations satisfied by $G_{2, 11}$ and $G_{5, 14}$. We observe that:
\begin{align}
    &G_{2, 11} = - \frac{\lambda_1}{- \lambda_1 + \phi_{1} G_{3, 10}}, \quad G_{3, 10} = -\lambda_1 \ntrace \Sigma_1 (\Sigma_1 G_{2, 11} + \lambda_1 I_d)^{-1} \\
    &\implies G_{2, 11} = \frac{1}{1 + \phi_{1} \ntrace \Sigma_1 (\Sigma_1 G_{2, 11} + \lambda_1 I_d)^{-1}}, \\
    &G_{5, 14} = - \frac{\lambda_1}{- \lambda_1 + \phi_{1} G_{6, 13}}, \quad G_{6, 13} = - \lambda_1 \ntrace \Sigma_1 \left(\Sigma_1 G_{5, 14} + \lambda_1 I_d \right)^{-1} \\
    &\implies G_{5, 14} = \frac{1}{1 + \phi_{1} \ntrace \Sigma_1 \left(\Sigma_1 G_{5, 14} + \lambda_1 I_d \right)^{-1}}.
\end{align}
We recognize that we must have the identification $e_1 = G_{2, 11} = G_{5, 14}$, where $e_1 \geq 0$. Therefore:
\begin{align}
    e_1 &= \frac{e_1}{e_1 + \phi_{1} \ndof^{(1)}_1 (\lambda_1 / e_1)} \\
    \text{i.e., } 1 &= e_1 + \phi_{1} \ndof^{(1)}_1 (\lambda_1 / e_1) = \lambda_1 / \kappa_1 + \phi_{1} \ndof^{(1)}_1 (\kappa_1) \\
    \kappa_1 &= \lambda_1 + \kappa_1 \phi_1 \ndof^{(1)}_1 (\kappa_1),
    \label{eq:kap-lin}
\end{align}
where $\ndof_m^{(s)} (t) = \ntrace \Sigma_s^m \left(\Sigma_s + t I_d \right)^{-m}$ and $\kappa_1 = \lambda_1 / e_1$. Additionally:
\begin{align}
G_{2, 14} &= \frac{\lambda_1 \phi_1 G_{3, 13}}{(-\lambda_1 + \phi_1 G_{3, 10}) (-\lambda_1 + \phi_1 G_{6, 13})} = \phi_1 e_1^2 \frac{G_{3, 13}}{\lambda_1}, \\
\frac{G_{3, 13}}{\lambda_1} &= \ntrace (\Sigma_1 G_{2, 11} + \lambda_1 I_d)^{-2} (\Sigma_1 G_{2, 14} + \lambda_1 I_d) \Sigma_1 \\
&= \frac{G_{2, 14}}{e_1^2} \ndof_2^{(1)} (\kappa_1) + \lambda_1 \ntrace (\Sigma_1 e_1 + \lambda_1 I_d)^{-2} \Sigma_1, \\
\frac{G_{3, 10}}{\lambda_1} &= -\ntrace (\Sigma_1 e_1 + \lambda_1 I_d)^{-1} \Sigma_1.
\end{align}
Then:
\begin{align}
G_{5, 14} - G_{2, 14} &= e_1^2 \left(1 - \phi_1 \frac{G_{3, 10} + G_{3, 13}}{\lambda_1} \right), \\
\frac{G_{3, 10} + G_{3, 13}}{\lambda_1} &= \frac{G_{2, 14}}{e_1^2} \ndof_2^{(1)} (\kappa_1) + \lambda_1 \ntrace (\Sigma_1 e_1 + \lambda_1 I_d)^{-2} \Sigma_1 \\
&-\ntrace (\Sigma_1 e_1 + \lambda_1 I_d)^{-2} (\Sigma_1 e_1 + \lambda_1 I_d) \Sigma_1 \\
&= \frac{G_{2, 14}}{e_1^2} \ndof_2^{(1)} (\kappa_1) - \frac{e_1}{e_1^2} \ndof_2^{(1)} (\kappa_1) \\
&= -\frac{G_{5, 14} - G_{2, 14}}{e_1^2} \ndof_2^{(1)} (\kappa_1).
\end{align}
We define:
\begin{align}
    &c_1 \geq 1, c_1 = \frac{G_{5, 14} - G_{2, 14}}{e_1^2} = 1 + \phi_1 c_1 \ndof_2^{(1)} (\kappa_1), \\
    \text{i.e., }&c_1 = \frac{1}{1 - \phi_1 \ndof_2^{(1)} (\kappa_1)}.
\end{align}
Hence: 
\begin{align}
    \frac{G_{0, 17}}{\lambda_1} = c_1 \ndof_2^{(1)} (\kappa_1) = \frac{\ndof_2^{(1)} (\kappa_1)}{1 - \phi_1 \ndof_2^{(1)} (\kappa_1)}.
\end{align}
In conclusion:
\begin{align}
    &\kappa_1 = \lambda_1 + \kappa_1 \phi_1 \ndof^{(1)}_1 (\kappa_1), \\
    &V_1 (\widehat f_1) = \frac{\sigma_1^2 \phi_1 \ndof_2^{(1)} (\kappa_1)}{1 - \phi_1 \ndof_2^{(1)} (\kappa_1)}.
\end{align}
Following similar steps for $V_2 (\widehat f_2)$, we get:
\begin{align}
    &\kappa_2 = \lambda_2 + \kappa_2 \phi_2 \ndof^{(2)}_1 (\kappa_2), \\
    &V_2 (\widehat f_2) = \frac{\sigma_2^2 \phi_2 \ndof_2^{(2)} (\kappa_2)}{1 - \phi_2 \ndof_2^{(2)} (\kappa_2)}.
\end{align}
To further substantiate our result, let us consider the unregularized case where $\lambda_s = 0$ and $\phi_s < 1$:
\begin{align}
    &\kappa_s = 0, V_s (\widehat f_s) = \frac{\sigma_s^2 \phi_s}{1 - \phi_s}.
\end{align}

From an alternate angle, we know that: 
\begin{align}
    R_s(\widehat f_s) &= \mathbb E\,\|\widehat w_s - w^*_s\|_{\Sigma_s}^2 = \mathbb E\,\|(X_s^\top X_s)^{-1}X_s^\top E_s\|_{\Sigma_s}^2 \\
    &= \sigma_s^2 \mathbb E \trace X_s (X_s^\top X_s)^{-1}\Sigma_s (X_s^\top X_s)^{-1} X_s^\top\\
    &= \sigma_s^2 \mathbb E \trace (X_s^\top X_s)^{-1}\Sigma_s = \frac{\sigma_s^2}{n_s-d-1}\trace I_d = \sigma_s^2\frac{d}{n_s-d-1} \simeq \frac{\sigma_s^2\phi_s}{1-\phi_s},
\end{align}
where we have used Lemma \ref{lm:stats101} below.

\begin{lemma}
    Let $n$ and $d$ be positive integers with $n \ge d+2$. If $Z$ is an $n \times d$ random matrix with IID rows from ${\cal N}(0,\Sigma)$, then:
    \begin{eqnarray}
        \mathbb E (Z^\top Z)^{-1} = \frac{1}{n-d-1}\Sigma^{-1}.
    \end{eqnarray}
    \label{lm:stats101}
\end{lemma}

\subsection{Bias Term}

We can compute the bias term $B_s (\widehat f_s)$ of the test error of $\widehat f_s$ evaluated on group $s$ as:
\begin{align}
    B_s(\widehat f_{s}) &= \mathbb E\,\|(M_s + n_s \lambda_s I_d)^{-1} M_s w^*_s -w^*_s\|_{\Sigma_s}^2 \\
    &= \mathbb E\,\|(M_s + n_s \lambda_s I_d)^{-1} M_s w^*_s - (M_s + n_s \lambda_s I_d)^{-1} (M_s + n_s \lambda_s I_d) w^*_s\|_{\Sigma_s}^2 \\
    &= \mathbb E\,\|(M_s + n_s \lambda_s I_d)^{-1} n_s \lambda_s w^*_s\|_{\Sigma_s}^2 \\
    &= n_s^2 \lambda_s^2 \mathbb E\, \trace (M_s + n_s \lambda_s I_d)^{-1} w^*_s (w^*_s)^\top (M_s + n_s \lambda_s I_d)^{-1} \Sigma_s.
\end{align}

We can re-express this as:
\begin{align}
    \frac{1}{\lambda_s^2} B_s (\widehat f_s) &= \mathbb E\, \ntrace (H_s + \lambda_s I_d)^{-1} \Theta_s (H_s + \lambda_s I_d)^{-1} \Sigma_s \\
    B_s (\widehat f_s) &= \mathbb E\, \ntrace (H_s/\lambda_s + I_d)^{-1} \Theta_s (H_s/\lambda_s + I_d)^{-1} \Sigma_s,
\end{align}
where $\Theta_s = \begin{cases} \Theta, & s = 1 \\ \Theta + \Delta, & s= 2 \end{cases}$. WLOG, we consider the case where $s = 1$. The matrix of interest has a linear pencil representation given by (with zero-based indexing):
\begin{align}
(H_1/\lambda_1 + I_d)^{-1} \Theta (H_1/\lambda_1 + I_d)^{-1} \Sigma_1 = Q^{-1}_{0, 8},
\end{align}
where the linear pencil $Q$ is defined as follows:
\begin{equation}
Q = 
\resizebox{0.8\linewidth}{!}{%
$\left(\begin{array}{ccccccccc}
I_{d} & \Sigma_{1}^{\frac{1}{2}} & 0 & 0 & -\Theta & 0 & 0 & 0 & 0 \\
0 & I_{d} & -\frac{1}{\sqrt{\lambda} \sqrt{n}} Z_{1}^\top & 0 & 0 & 0 & 0 & 0 & 0 \\
0 & 0 & I_{n_1} & -\frac{1}{\sqrt{\lambda} \sqrt{n}} Z_{1} & 0 & 0 & 0 & 0 & 0 \\
-\Sigma_{1}^{\frac{1}{2}} & 0 & 0 & I_{d} & 0 & 0 & 0 & 0 & 0 \\
0 & 0 & 0 & 0 & I_{d} & \Sigma_{1}^{\frac{1}{2}} & 0 & 0 & -\Sigma_1 \\
0 & 0 & 0 & 0 & 0 & I_{d} & -\frac{1}{\sqrt{\lambda} \sqrt{n}} Z_{1}^\top & 0 & 0 \\
0 & 0 & 0 & 0 & 0 & 0 & I_{n_1} & -\frac{1}{\sqrt{\lambda} \sqrt{n}} Z_{1} & 0 \\
0 & 0 & 0 & 0 & -\Sigma_{1}^{\frac{1}{2}} & 0 & 0 & I_{d} & 0 \\
0 & 0 & 0 & 0 & 0 & 0 & 0 & 0 & I_{d} \\
\end{array}\right)$}.
\end{equation}
We note that $\overline Q^{-1}_{0, 17} = Q^{-1}_{0, 8}$. Using OVFPT, we deduce that, in the limit given by Equation \ref{eq:proportionate}, the following holds:
\begin{align}
&\mathbb E\, \ntrace (H_1/\lambda_1 + I_d)^{-1} \Theta (H_1/\lambda_1 + I_d)^{-1} \Sigma_1 = G_{0, 17}, \\
&\text{with } G_{0, 17} = \lambda_1 \ntrace \left(\Sigma_1 G_{2, 11} + \lambda_1 I_d\right)^{-1} (\lambda_1 \Theta + \Sigma_1 G_{2, 15}) \left(\Sigma_1 G_{6, 15} + \lambda_1 I_d \right)^{-1} \Sigma_1.
\end{align}
We will now obtain the fixed-point equations satisfied by $G_{2, 11}$ and $G_{6, 15}$. We observe that:
\begin{align}
&G_{2, 11} = - \frac{\lambda_1}{- \lambda_1 + \phi_{1} G_{3, 10}}, \quad  G_{3, 10} = -\lambda_1 \ntrace \Sigma_1 (\Sigma_1 G_{2, 11} + \lambda_1 I_d)^{-1} \\
&\implies G_{2, 11} = \frac{1}{1 + \phi_{1} \ntrace \Sigma_1 (\Sigma_1 G_{2, 11} + \lambda_1 I_d)^{-1}}, \\
&G_{6, 15} = - \frac{\lambda_1}{- \lambda_1 + \phi_{1} G_{7, 14}}, \quad G_{7, 14} = - \lambda_1 \ntrace \Sigma_1 \left( \Sigma_1 G_{6, 15} + \lambda_1 I_d \right)^{-1} \\
&\implies G_{6, 15} = \frac{1}{1 + \phi_{1} \ntrace \Sigma_1 \left( \Sigma_1 G_{6, 15} + \lambda_1 I_d \right)^{-1}}.
\end{align}

We recognize that we must have the identification $e_1 = G_{2, 11} = G_{6, 15}$, where $e_1 \geq 0$. Therefore:
\begin{align}
    e_1 &= \frac{1}{1 + \phi_{1} \ntrace \Sigma_1 \left(\Sigma_1 e_1 + \lambda_1 I_d \right)^{-1}}, \\
    \text{i.e., }\kappa_1 &= \lambda_1 + \kappa_1 \phi_1 \ndof^{(1)}_1 (\kappa_1).
\end{align}
Additionally:
\begin{align}
G_{2, 15} &= \frac{\lambda_1 \phi_1 G_{3, 14}}{(-\lambda_1 + \phi_1 G_{3, 10}) (-\lambda_1 + \phi_1 G_{7, 14})} = \phi_1 e_1^2 \frac{G_{3, 14}}{\lambda_1}, \\
\frac{G_{3, 14}}{\lambda_1} &= \ntrace (\Sigma_1 G_{2, 11} + \lambda_1 I_d)^{-2} (\Sigma_1 G_{2, 15} + \lambda_1 \Theta) \Sigma_1 \\
&= \frac{G_{2, 15}}{e_1^2} \ndof_2^{(1)} (\kappa_1) + \frac{\lambda_1}{e_1^2} \ntrace (\Sigma_1 + \kappa_1 I_d)^{-2} \Theta \Sigma_1, \\
\implies G_{2, 15} &= \phi_1 G_{2, 15} \ndof_2^{(1)} (\kappa_1) + \lambda_1 \phi_1 \ntrace (\Sigma_1 + \kappa_1 I_d)^{-2} \Theta \Sigma_1, \\
\text{i.e., }G_{2, 15} &= \frac{\lambda_1 \phi_1}{1 - \phi_1 \ndof_2^{(1)} (\kappa_1)} \ntrace (\Sigma_1 + \kappa_1 I_d)^{-2} \Theta \Sigma_1.
\end{align}
Hence:
\begin{align}
    G_{0, 17} &= \kappa_1^2 \ntrace \left(\Sigma_1 + \kappa_1 I_d\right)^{-2} \Theta \Sigma_1 + \kappa_1^2 \ndof_2^{(1)} (\kappa_1) \frac{G_{2, 15}}{\lambda_1} \\
    &= \kappa_1^2 \ntrace \left(\Sigma_1 + \kappa_1 I_d\right)^{-2} \Theta \Sigma_1 + \kappa_1^2 \frac{\phi_1 \ndof_2^{(1)} (\kappa_1)}{1 - \phi_1 \ndof_2^{(1)} (\kappa_1)} \ntrace (\Sigma_1 + \kappa_1 I_d)^{-2} \Theta \Sigma_1 \\
    &= \left(1 + \frac{\phi_1 \ndof_2^{(1)} (\kappa_1)}{1 - \phi_1 \ndof_2^{(1)} (\kappa_1)} \right) \kappa_1^2 \ntrace \left(\Sigma_1 + \kappa_1 I_d\right)^{-2} \Theta \Sigma_1.
\end{align}
In conclusion:
\begin{align}
    B_1 (\widehat f_1) = \frac{\kappa_1^2 \ntrace \left(\Sigma_1 + \kappa_1 I_d\right)^{-2} \Theta \Sigma_1}{1 - \phi_1 \ndof_2^{(1)} (\kappa_1)}.
\end{align}

Following similar steps for $B_2 (\widehat f_2)$, we get:
\begin{align}
    B_2 (\widehat f_2) = \frac{\kappa_2^2 \ntrace \left(\Sigma_2 + \kappa_2 I_d\right)^{-2} (\Theta + \Delta) \Sigma_2}{1 - \phi_2 \ndof_2^{(2)} (\kappa_2)}.
\end{align}

We observe that in the unregularized case (i.e., $\lambda_s = 0$), $\kappa_s = 0$. In this setting, $B_s (\widehat f_s) = 0$ as expected.
\end{proof}

\clearpage

\section{Proof of Theorem \ref{thm:odd-theory}}
\label{sec:odd-proof}

\begin{proof}
We define $M = X^\top X + n \lambda I_d$. Note that one has:
\begin{align}
\widehat w &= M^{-1}(M_1w^*_1+X_1^\top E_1+M_2w^*_2+X_2^\top E_2).
\end{align}
We deduce that $R_s(\widehat f) = B_s(\widehat f) + V_s(\widehat f)$, where:
\begin{align}
B_s(\widehat f) &=  \mathbb E\,\|M^{-1} M_{s'} w^*_{s'} + M^{-1} M_s w^*_s -w^*_s \|_{\Sigma_s}^2, \\
V_s(\widehat f) &= \mathbb E\,\|M^{-1} (X_1^\top E_1+X_2^\top E_2)\|_{\Sigma_s}^2 \\
&= \mathbb E\,\|M^{-1} X_1^\top E_1\|_{\Sigma_s}^2 + \mathbb E\,\|M^{-1} X_2^\top E_2\|_{\Sigma_s}^2,
\end{align}
with $s' = \begin{cases} 2, & s = 1 \\ 1, & s = 2 \end{cases}$.

\subsection{Variance Terms}

Note that $V_s(\widehat f)$ of the test error of $\widehat f$ evaluated on group $s$ is given by: 
\begin{align}
V_s(\widehat f) &= \sigma_1^2\mathbb E\, \trace X_1 M^{-1} \Sigma_s M^{-1} X_1^\top + \sigma_2^2\mathbb E\, \trace X_2 M^{-1} \Sigma_s M^{-1} X_2^\top \\
&= \sigma_1^2\mathbb E\, \trace M^{-1} M_1 M^{-1} \Sigma_s + \sigma_2^2\mathbb E\, \trace M^{-1} M_2 M^{-1} \Sigma_s.
\end{align}
We can re-express this as:
\begin{eqnarray}
nV_s(\widehat f) = \sigma_1^2 \mathbb E\trace (H+\lambda I_d)^{-1} H_1 (H+\lambda I_d)^{-1} \Sigma_s + \sigma_2^2\mathbb E\, \trace (H+\lambda I_d)^{-1} H_2 (H+\lambda I_d)^{-1} \Sigma_s,
\end{eqnarray}
where $H = H_1+H_2$, $H_s = X_s^\top X_s/n$, and $X_s=Z_s\Sigma_s^{1/2}$ with $Z_1 \in \mathbb R^{n_1 \times d}$ and $Z_2 \in \mathbb R^{n_2 \times d}$ being independent random matrices with IID entries from ${\cal N}(0,1)$.

WLOG, we focus on $\trace (H+\lambda I_d)^{-1} H_2 (H+\lambda I_d)^{-1} \Sigma_s$. The matrix of interest has a linear pencil representation given by (with zero-based indexing):
\begin{align}
(H_1 / \lambda + H_2 / \lambda + I_d)^{-1} (H_2 / \lambda) (H_1 / \lambda + H_2 / \lambda + I_d)^{-1} \Sigma_s = Q^{-1}_{1, 8},
\end{align}
where the linear pencil $Q$ is defined as follows:
\begin{equation}
\rotatebox{270}{%
$Q = $
\resizebox{\linewidth}{!}{%
$\left(\begin{array}{ccccccccccccccc}
I_{d} & 0 & -\frac{1}{\sqrt{\lambda} \sqrt{n}} Z_{2}^\top & 0 & 0 & 0 & 0 & 0 & 0 & 0 & 0 & 0 & 0 & 0 & 0 \\
-\Sigma_{2}^{\frac{1}{2}} & I_{d} & 0 & 0 & 0 & 0 & 0 & 0 & 0 & 0 & 0 & 0 & 0 & 0 & 0 \\
0 & 0 & I_{n_2} & -\frac{1}{\sqrt{\lambda} \sqrt{n}} Z_{2} & 0 & 0 & 0 & 0 & 0 & 0 & 0 & 0 & 0 & 0 & 0 \\
0 & 0 & 0 & I_{d} & -\Sigma_{2}^{\frac{1}{2}} & 0 & 0 & 0 & 0 & 0 & 0 & 0 & 0 & 0 & 0 \\
\Sigma_{2}^{\frac{1}{2}} & 0 & 0 & 0 & I_{d} & \Sigma_{1}^{\frac{1}{2}} & 0 & 0 & -\Sigma_s & 0 & 0 & 0 & 0 & 0 & 0 \\
0 & 0 & 0 & 0 & 0 & I_{d} & -\frac{1}{\sqrt{\lambda} \sqrt{n}} Z_{1}^\top & 0 & 0 & 0 & 0 & 0 & 0 & 0 & 0 \\
0 & 0 & 0 & 0 & 0 & 0 & I_{n_1} & -\frac{1}{\sqrt{\lambda} \sqrt{n}} Z_{1} & 0 & 0 & 0 & 0 & 0 & 0 & 0 \\
0 & 0 & 0 & 0 & -\Sigma_{1}^{\frac{1}{2}} & 0 & 0 & I_{d} & 0 & 0 & 0 & 0 & 0 & 0 & 0 \\
0 & 0 & 0 & 0 & 0 & 0 & 0 & 0 & I_{d} & \Sigma_{1}^{\frac{1}{2}} & 0 & 0 & \Sigma_{2}^{\frac{1}{2}} & 0 & 0 \\
0 & 0 & 0 & 0 & 0 & 0 & 0 & 0 & 0 & I_{d} & -\frac{1}{\sqrt{\lambda} \sqrt{n}} Z_{1}^\top & 0 & 0 & 0 & 0 \\
0 & 0 & 0 & 0 & 0 & 0 & 0 & 0 & 0 & 0 & I_{n_1} & -\frac{1}{\sqrt{\lambda} \sqrt{n}} Z_{1} & 0 & 0 & 0 \\
0 & 0 & 0 & 0 & 0 & 0 & 0 & 0 & -\Sigma_{1}^{\frac{1}{2}} & 0 & 0 & I_{d} & 0 & 0 & 0 \\
0 & 0 & 0 & 0 & 0 & 0 & 0 & 0 & 0 & 0 & 0 & 0 & I_{d} & -\frac{1}{\sqrt{\lambda} \sqrt{n}} Z_{2}^\top & 0 \\
0 & 0 & 0 & 0 & 0 & 0 & 0 & 0 & 0 & 0 & 0 & 0 & 0 & I_{n_2} & -\frac{1}{\sqrt{\lambda} \sqrt{n}} Z_{2} \\
0 & 0 & 0 & 0 & 0 & 0 & 0 & 0 & -\Sigma_{2}^{\frac{1}{2}} & 0 & 0 & 0 & 0 & 0 & I_{d} \\
\end{array}\right)$}.}
\end{equation}
Using OVFPT, we deduce that, in the limit given by Equation \ref{eq:proportionate}, the following holds:
\begin{align}
\mathbb E\,\ntrace (H_1+H_2+\lambda I_d)^{-1}H_2(H_1+H_2+\lambda I_d)^{-1}\Sigma_s = \frac{G_{1,23}}{\lambda},
\end{align}
with:
\begin{align}
&\frac{G_{1,23}}{\lambda} = \lambda^{-1} \ntrace p_2 \Sigma_2 (\lambda \Sigma_s G_{0, 15} + \lambda G_{0, 27} I_d - p_1 \Sigma_1 G_{0, 15} G_{5, 24} + p_1 \Sigma_1 G_{0, 27} G_{5, 20}) \\
&\cdot (p_1 \Sigma_1 G_{5, 20} + p_2 \Sigma_2 G_{0, 15} + \lambda I_d)^{-2}.
\end{align}
By identifying identical entries of $\overline Q^{-1}$, we must have that $\frac{G_{5, 20}}{\lambda} = \frac{G_{6, 21}}{\lambda} = \frac{G_{10, 25}}{\lambda}, \frac{G_{0, 15}}{\lambda} = \frac{G_{2, 17}}{\lambda} = \frac{G_{13, 28}}{\lambda}$. For $G_{6, 21}$ and $G_{2, 17}$, we observe that:
\begin{align}
    &G_{6, 21} = - \frac{\lambda}{- \lambda + \phi G_{7, 20}}, \quad 
    G_{7, 20} = -\lambda \ntrace \Sigma_1 \left(p_{1} \Sigma_1 G_{6, 21} + p_{2} \Sigma_2 G_{2, 17} + \lambda I_d\right)^{-1} \\
    &\implies G_{6, 21} = \frac{1}{1 + \phi \ntrace \Sigma_1 \left(p_{1} \Sigma_1 G_{6, 21} + p_{2} \Sigma_2 G_{2, 17} + \lambda I_d\right)^{-1}}, \\
    &G_{2, 17} = - \frac{\lambda}{- \lambda + \phi G_{3, 15}}, \quad 
    G_{3, 15} = - \lambda \ntrace \Sigma_2 \left(p_{1} \Sigma_1 G_{6, 21} + p_{2} \Sigma_2 G_{2, 17} + \lambda I_d\right)^{-1} \\
    &\implies G_{2, 17} = \frac{1}{1 + \phi \ntrace \Sigma_2 \left(p_{1} \Sigma_1 G_{6, 21} + p_{2} \Sigma_2 G_{2, 17} + \lambda I_d\right)^{-1}}.
\end{align}

We define $\eta_1 = \frac{G_{6, 21}}{\lambda}, \eta_2 = \frac{G_{2, 17}}{\lambda}$, with $\eta_1 \geq 0, \eta_2 \geq 0$. Therefore:
\begin{align}
\eta_s &= \frac{1}{\lambda+\phi \ntrace \Sigma_s K^{-1}},
\end{align}
where $K = \eta_1 p_1 \Sigma_1 + \eta_2 p_2 \Sigma_2 + I_d$. Additionally, by identifying identical entries of $\overline Q^{-1}$, we must have that $G_{5, 24} = G_{6, 25}, G_{0, 27} = G_{2, 28}$. We observe that:
\begin{align}
    G_{10, 25} &= \frac{-\lambda}{-\lambda + \phi G_{11, 24}}, \\
    G_{6, 25} &= \frac{\lambda \phi G_{7, 24}}{(-\lambda + \phi G_{7, 20}) (-\lambda + \phi G_{11, 24})} = \phi \lambda^2 \eta_1^2 \frac{G_{7, 24}}{\lambda}, \\
    \frac{G_{7, 24}}{\lambda} &= \lambda^{-2} \ntrace K^{-2} (p_1 \Sigma_1 G_{6, 25} + p_2 \Sigma_2 G_{2, 28} - \lambda \Sigma_s) \Sigma_1, \\
    \implies G_{6, 25} &= \phi \eta_1^2 \ntrace K^{-2} (p_1 \Sigma_1 G_{6, 25} + p_2 \Sigma_2 G_{2, 28} - \lambda \Sigma_s) \Sigma_1, \\
    G_{13, 28} &= \frac{-\lambda}{-\lambda + \phi G_{14, 27}}, \\
    G_{2, 28} &= \frac{\lambda \phi G_{3, 27}}{(-\lambda + \phi G_{3, 15}) (-\lambda + \phi G_{14, 27})} = \phi \lambda^2 \eta_2^2 \frac{G_{3, 27}}{\lambda}, \\
    \frac{G_{3, 27}}{\lambda} &= \lambda^{-2} \ntrace K^{-2} (p_1 \Sigma_1 G_{6, 25} + p_2 \Sigma_2 G_{2, 28} - \lambda \Sigma_s) \Sigma_2, \\
    \implies G_{2, 28} &= \phi \eta_2^2 \ntrace K^{-2} (p_1 \Sigma_1 G_{6, 25} + p_2 \Sigma_2 G_{2, 28} - \lambda \Sigma_s) \Sigma_2.
\end{align}
We now define $v^{(s)}_1 = -G_{6, 25}, v^{(s)}_2 = -G_{2, 28}$, with $v^{(s)}_1 \geq 0, v^{(s)}_2 \geq 0$. Therefore, $v^{(s)}_1, v^{(s)}_2$ obey the following system of equations:
\begin{align}
    v^{(s)}_k &= \phi \eta_k^2 \ntrace K^{-2} (v^{(s)}_1 p_1 \Sigma_1 + v^{(s)}_2 p_2 \Sigma_2 + \lambda \Sigma_s) \Sigma_k.
\end{align}

We further define $u^{(s)}_k = \frac{v^{(s)}_k}{\lambda}$. Putting all the pieces together:
\begin{align}
\frac{G_{1,23}}{\lambda} = \lambda^{-1} \ntrace p_2 \Sigma_2 \big(\eta_2 \Sigma_s - u^{(s)}_2 I_d + p_1 \Sigma_1 (\eta_2 u^{(s)}_1 - \eta_1 u^{(s)}_2) \big) K^{-2}.
\end{align}

By symmetry, in conclusion:
\begin{gather}
    V_s (\widehat f) = V_s^{(1)} (\widehat f) + V_s^{(2)} (\widehat f), \\
    V_s^{(k)} (\widehat f) = \lambda^{-1} \phi \sigma_k^2 \ntrace p_k \Sigma_k \big(\eta_k \Sigma_s - u^{(s)}_k I_d + p_{k'} \Sigma_{k'} (\eta_k u^{(s)}_{k'} - \eta_{k'} u^{(s)}_k) \big) K^{-2},
\end{gather}
with $k' = \begin{cases} 2, & k = 1 \\ 1, & k = 2 \end{cases}$.

We now corroborate our result in the limit $p_2 \to 1$ (i.e., $p_1 \to 0$) and $s = 2$. We observe that:
\begin{align}
    \phi \to \phi_2&, \lambda \to \lambda_2, \\
    V_2^{(1)} (\widehat f) &= 0, \\
    \frac{V_2^{(2)} (\widehat f)}{\lambda^{-1} \phi_2 \sigma_2^2} &= \ntrace \Sigma_2 (\eta_2 \Sigma_2 - u_2^{(2)} I_d) K^{-2} \\
    v_2^{(2)} &= \phi_2 \eta_2^2 \ntrace K^{-2} (v_2^{(2)} \Sigma_2 + \lambda_2 \Sigma_2) \Sigma_2 \\
    &= \phi_2 (v_2^{(2)} + \lambda_2) \ndof_2^{(2)} (\kappa_2), \\
    u_2^{(2)} &= \frac{\phi_2 \ndof_2^{(2)} (\kappa_2)}{1 - \phi_2 \ndof_2^{(2)} (\kappa_2)}, \\
    \frac{V_2^{(2)} (\widehat f)}{\lambda^{-1} \phi_2 \sigma_2^2} &= \kappa_2 \ndof_2^{(2)} (\kappa_2) - u_2^{(2)} \ntrace \Sigma_2 (\eta_2 \Sigma_2 + I_d)^{-2} \\
    &= \kappa_2 \ndof_2^{(2)} (\kappa_2) - \kappa_2^2 u_2^{(2)} \ntrace \Sigma_2 (\Sigma_2 + \kappa_2 I_d)^{-2} \\
    &= \kappa_2 \ndof_2^{(2)} (\kappa_2) - \kappa_2  u_2^{(2)} (\ndof_1^{(2)} (\kappa_2) - \ndof_2^{(2)} (\kappa_2)) \\
    &= \kappa_2 (1 + u_2^{(2)}) \ndof_2^{(2)} (\kappa_2) - \kappa_2 u_2^{(2)} \ndof_1^{(2)} (\kappa_2) \\
    &= \frac{\kappa_2 - \kappa_2 \phi_2 \ndof_1^{(2)} (\kappa_2)}{1 - \phi_2 \ndof_2^{(2)} (\kappa_2)} \cdot \ndof_2^{(2)} (\kappa_2) \\
    &= \frac{\lambda \ndof_2^{(2)} (\kappa_2)}{1 - \phi_2 \ndof_2^{(2)} (\kappa_2)}, \\
    V_2^{(2)} (\widehat f) &= \frac{\sigma_2^2 \phi_2 \ndof_2^{(2)} (\kappa_2)}{1 - \phi_2 \ndof_2^{(2)} (\kappa_2)},
\end{align}
which exactly recovers the result for $V_2 (\widehat f_2)$ as expected.

\subsection{Bias Terms}
Recall that:
\begin{align}
B_s(\widehat f) &=  \mathbb E\,\|M^{-1} M_{s'} w^*_{s'} + M^{-1} M_s w^*_s -w^*_s \|_{\Sigma_s}^2.
\end{align}
Now, observe that $M^{-1}M_1 w^*_1-w^*_1 = M^{-1}M_1 w^*_1-M^{-1}Mw^*_1=-M^{-1} M_2 w^*_1 - n \lambda M^{-1} w_1^*$. Let $\delta = w^*_2-w^*_1$. Then:
\begin{align}
    B_s(\widehat f) &= \mathbb E \|M^{-1} M_{s'} (-1)^{s - 1} \delta - n \lambda M^{-1} w_s^* \|^2_{\Sigma_s} \\
    &= \mathbb E \trace \delta^\top M_{s'} M^{-1}\Sigma_s M^{-1} M_{s'} \delta \\
    &- 2 (-1)^{s - 1} n \lambda \mathbb E \trace 
 \delta^\top M_{s'} M^{-1} \Sigma_s M^{-1} w_s^* \\
    &+ n^2 \lambda^2 \mathbb E \trace (w_s^*)^\top M^{-1} \Sigma_s M^{-1} w_s^* \\
    &= B_s^{(1)} (\widehat f) - 2 (-1)^{s - 1} B_s^{(2)} (\widehat f) + B_s^{(3)} (\widehat f),
\end{align}
where:
\begin{align}
B_s^{(1)} (\widehat f) &= \mathbb E \ntrace (H_1/\lambda + H_2/\lambda + I_d)^{-1} (H_{s'}/\lambda) \Delta (H_{s'}/\lambda) (H_1/\lambda + H_2/\lambda + I_d)^{-1} \Sigma_s, \\
B_s^{(2)} (\widehat f) &= \mathbb E \trace \delta^\top (H_{s'}/\lambda) (H_1/\lambda + H_2/\lambda + I_d)^{-1} \Sigma_s (H_1/\lambda + H_2/\lambda + I_d)^{-1} w_s^*, \\
B_s^{(3)} (\widehat f) &= \mathbb E \ntrace (H_1/\lambda + H_2/\lambda + I_d)^{-1} \Theta_s (H_1/\lambda + H_2/\lambda + I_d)^{-1} \Sigma_s.
\end{align}
Because $\delta$ and $w_1^*$ are independent and sampled from zero-centered distributions:
\begin{align}
    B_1^{(2)} (\widehat f) &= 0, \\
    B_2^{(2)} (\widehat f) &= \mathbb E \ntrace (H_1/\lambda + H_2/\lambda + I_d)^{-1} \Delta (H_1/\lambda) (H_1/\lambda + H_2/\lambda + I_d)^{-1} \Sigma_2.
\end{align}

WLOG, for $B_s^{(1)}$, we focus on the case $s = 1$. The matrix of interest has a linear pencil representation given by (with zero-based indexing):
\begin{align}
    (H_1/\lambda + H_2/\lambda + I_d)^{-1} (H_2/\lambda) \Delta (H_2/\lambda) (H_1/\lambda + H_2/\lambda + I_d)^{-1} \Sigma_1 = Q^{-1}_{1, 16},
\end{align}
where the linear pencil $Q$ is defined as follows:
\begin{equation}
\rotatebox{270}{%
$Q = $
\resizebox{\linewidth}{!}{%
$\left(\begin{array}{ccccccccccccccccc}
I_{d} & 0 & -\Sigma_{2}^{\frac{1}{2}} & 0 & 0 & 0 & 0 & 0 & 0 & 0 & 0 & 0 & 0 & 0 & 0 & 0 & 0 \\
-\Delta & I_{d} & 0 & 0 & 0 & 0 & 0 & 0 & 0 & 0 & 0 & 0 & 0 & 0 & 0 & 0 & 0 \\
0 & 0 & I_{d} & -\frac{1}{\sqrt{\lambda} \sqrt{n}} Z_{2}^\top & 0 & 0 & 0 & 0 & 0 & 0 & 0 & 0 & 0 & 0 & 0 & 0 & 0 \\
0 & 0 & 0 & I_{n_2} & -\frac{1}{\sqrt{\lambda} \sqrt{n}} Z_{2} & 0 & 0 & 0 & 0 & 0 & 0 & 0 & 0 & 0 & 0 & 0 & 0 \\
0 & 0 & 0 & 0 & I_{d} & -\Sigma_{2}^{\frac{1}{2}} & 0 & 0 & 0 & 0 & 0 & 0 & 0 & 0 & 0 & 0 & 0 \\
0 & 0 & \Sigma_{2}^{\frac{1}{2}} & 0 & 0 & I_{d} & \Sigma_{1}^{\frac{1}{2}} & 0 & 0 & -\Sigma_1 & 0 & 0 & 0 & 0 & 0 & 0 & 0 \\
0 & 0 & 0 & 0 & 0 & 0 & I_{d} & -\frac{1}{\sqrt{\lambda} \sqrt{n}} Z_{1}^\top & 0 & 0 & 0 & 0 & 0 & 0 & 0 & 0 & 0 \\
0 & 0 & 0 & 0 & 0 & 0 & 0 & I_{n_1} & -\frac{1}{\sqrt{\lambda} \sqrt{n}} Z_{1} & 0 & 0 & 0 & 0 & 0 & 0 & 0 & 0 \\
0 & 0 & 0 & 0 & 0 & -\Sigma_{1}^{\frac{1}{2}} & 0 & 0 & I_{d} & 0 & 0 & 0 & 0 & 0 & 0 & 0 & 0 \\
0 & 0 & 0 & 0 & 0 & 0 & 0 & 0 & 0 & I_{d} & \Sigma_{1}^{\frac{1}{2}} & 0 & 0 & \Sigma_{2}^{\frac{1}{2}} & 0 & 0 & 0 \\
0 & 0 & 0 & 0 & 0 & 0 & 0 & 0 & 0 & 0 & I_{d} & -\frac{1}{\sqrt{\lambda} \sqrt{n}} Z_{1}^\top & 0 & 0 & 0 & 0 & 0 \\
0 & 0 & 0 & 0 & 0 & 0 & 0 & 0 & 0 & 0 & 0 & I_{n_1} & -\frac{1}{\sqrt{\lambda} \sqrt{n}} Z_{1} & 0 & 0 & 0 & 0 \\
0 & 0 & 0 & 0 & 0 & 0 & 0 & 0 & 0 & -\Sigma_{1}^{\frac{1}{2}} & 0 & 0 & I_{d} & 0 & 0 & 0 & 0 \\
0 & 0 & 0 & 0 & 0 & 0 & 0 & 0 & 0 & 0 & 0 & 0 & 0 & I_{d} & -\frac{1}{\sqrt{\lambda} \sqrt{n}} Z_{2}^\top & 0 & 0 \\
0 & 0 & 0 & 0 & 0 & 0 & 0 & 0 & 0 & 0 & 0 & 0 & 0 & 0 & I_{n_2} & -\frac{1}{\sqrt{\lambda} \sqrt{n}} Z_{2} & 0 \\
0 & 0 & 0 & 0 & 0 & 0 & 0 & 0 & 0 & -\Sigma_{2}^{\frac{1}{2}} & 0 & 0 & 0 & 0 & 0 & I_{d} & \Sigma_{2}^{\frac{1}{2}} \\
0 & 0 & 0 & 0 & 0 & 0 & 0 & 0 & 0 & 0 & 0 & 0 & 0 & 0 & 0 & 0 & I_{d} \\
\end{array}\right)$}.}
\end{equation} 

Using OVFPT, we deduce that, in the limit given by Equation \ref{eq:proportionate}, the following holds:
\begin{align}
    \mathbb E \ntrace (H_1/\lambda + H_2/\lambda + I_d)^{-1} (H_2/\lambda) \Delta (H_2/\lambda) (H_1/\lambda + H_2/\lambda + I_d)^{-1} \Sigma_1 = G_{1, 33}, 
\end{align}
with:
\begin{equation}
\begin{split}    
&G_{1, 33} \\
&= \lambda^{-1} \ntrace p_2 \Sigma_2 \Delta (p_2 \Sigma_2 G^2_{2, 19} (\lambda - p_1 G_{6, 27}) \Sigma_1 - G_{2, 30} (p_1 \Sigma_1 G_{6, 23} + \lambda I_d)^2) \\
&\cdot (p_1 \Sigma_1 G_{6, 23} + p_2 \Sigma_2 G_{2, 19} + \lambda I_d)^{-2}.
\end{split}
\end{equation}

By identifying identical entries of $\overline Q^{-1}$, we must have that $\eta_1 = \frac{G_{6, 23}}{\lambda} = \frac{G_{7, 24}}{\lambda} = \frac{G_{11, 28}}{\lambda}, \eta_2 = \frac{G_{2, 19}}{\lambda} = \frac{G_{3, 20}}{\lambda} = \frac{G_{14, 31}}{\lambda}$. For $G_{7, 24}$ and $G_{3, 20}$, we observe that:
\begin{align}
    &G_{7, 24} = - \frac{\lambda}{- \lambda + \phi G_{8, 23}}, \quad 
    G_{8, 23} = -\lambda \ntrace \Sigma_1 \left(p_{1} \Sigma_1 G_{7, 24} + p_{2} \Sigma_2 G_{3, 20} + \lambda I_d\right)^{-1} \\
    &\implies G_{7, 24} = \frac{1}{1 + \phi \ntrace \Sigma_1 \left(p_{1} \Sigma_1 G_{7, 24} + p_{2} \Sigma_2 G_{3, 20} + \lambda I_d\right)^{-1}}, \\
    &G_{3, 20} = - \frac{\lambda}{- \lambda + \phi G_{4, 19}}, \quad 
    G_{4, 19} = - \lambda \ntrace \Sigma_2 \left(p_{1} \Sigma_1 G_{7, 24} + p_{2} \Sigma_2 G_{3, 20} + \lambda I_d\right)^{-1} \\
    &\implies G_{3, 20} = \frac{1}{1 + \phi \ntrace \Sigma_2 \left(p_{1} \Sigma_1 G_{7, 24} + p_{2} \Sigma_2 G_{3, 20} + \lambda I_d\right)^{-1}}.
\end{align}

By again identifying identical entries of $\overline Q^{-1}$, we further have that $v^{(1)}_1 = -G_{6, 27} = -G_{7, 28}, v^{(1)}_2 = -G_{2, 30} = -G_{3, 31}$. We observe that:
\begin{align}
G_{7, 28} &= \phi \lambda^2 \eta_1^2 \frac{G_{8, 27}}{\lambda}, \\
\frac{G_{8, 27}}{\lambda} &= \lambda^{-2} \ntrace K^{-2} (p_1 \Sigma_1 G_{7, 28} + p_2 \Sigma_2 G_{3, 31} - \lambda \Sigma_1) \Sigma_1 \\
\implies v^{(1)}_1 &= \phi \eta_1^2 \ntrace K^{-2} (v^{(s)}_1 p_1 \Sigma_1  + v^{(s)}_2 p_2 \Sigma_2 + \lambda \Sigma_1) \Sigma_1, \\
G_{3, 31} &= \phi \lambda^2 \eta_2^2 \frac{G_{4, 30}}{\lambda}, \\
\frac{G_{4, 30}}{\lambda} &= \lambda^{-2} \ntrace K^{-2} (p_1 \Sigma_1 G_{7, 28} + p_2 \Sigma_2 G_{3, 31} -\lambda \Sigma_1) \Sigma_2, \\
\implies v^{(1)}_2 &= \phi \eta_2^2 \ntrace K^{-2} (v^{(s)}_1 p_1 \Sigma_1  + v^{(s)}_2 p_2 \Sigma_2 + \lambda \Sigma_1) \Sigma_2.
\end{align}

Putting all the pieces together:
\begin{align}
B_1^{(1)} (\widehat f) &= \ntrace p_2 \Sigma_2 \Delta (p_2 \eta_2^2 \Sigma_2 (1 + p_1 u^{(s)}_1) \Sigma_1 + u^{(s)}_2 (p_1 \eta_1 \Sigma_1 + I_d)^2) K^{-2}.
\end{align}

In conclusion:
\begin{align}
B_s^{(1)} (\widehat f) &= \ntrace p_{s'} \Sigma_{s'} \Delta (p_{s'} \eta_{s'}^2 \Sigma_{s'} (1 + p_s u^{(s)}_s) \Sigma_s + u^{(s)}_{s'} (p_s \eta_s \Sigma_s + I_d)^2) K^{-2}.
\end{align}

Now, switching our focus to $B_2^{(2)} (\widehat f)$, the matrix of interest has a linear pencil representation given by (with zero-based indexing):
\begin{align}
   (H_1/\lambda + H_2/\lambda + I_d)^{-1} \Delta (H_1/\lambda) (H_1/\lambda + H_2/\lambda + I_d)^{-1} \Sigma_2 = Q^{-1}_{0, 15},
\end{align}
where the linear pencil $Q$ is defined as follows:
\begin{equation}
\rotatebox{270}{%
$Q = $
\resizebox{\linewidth}{!}{%
$\left(\begin{array}{cccccccccccccccc}
I_{d} & 0 & -\Sigma_{1}^{\frac{1}{2}} & 0 & 0 & 0 & 0 & 0 & 0 & 0 & 0 & 0 & 0 & 0 & 0 & 0 \\
-\Delta & I_{d} & 0 & 0 & 0 & 0 & 0 & 0 & 0 & 0 & 0 & 0 & 0 & 0 & 0 & 0 \\
0 & 0 & I_{d} & -\frac{1}{\sqrt{\lambda} \sqrt{n}} Z_{1}^\top & 0 & 0 & 0 & 0 & 0 & 0 & 0 & 0 & 0 & 0 & 0 & 0 \\
0 & 0 & 0 & I_{n_1} & -\frac{1}{\sqrt{\lambda} \sqrt{n}} Z_{1} & 0 & 0 & 0 & 0 & 0 & 0 & 0 & 0 & 0 & 0 & 0 \\
0 & 0 & 0 & 0 & I_{d} & -\Sigma_{1}^{\frac{1}{2}} & 0 & 0 & 0 & 0 & 0 & 0 & 0 & 0 & 0 & 0 \\
0 & 0 & \Sigma_{1}^{\frac{1}{2}} & 0 & 0 & I_{d} & \Sigma_{2}^{\frac{1}{2}} & 0 & 0 & -\Sigma_2 & 0 & 0 & 0 & 0 & 0 & 0 \\
0 & 0 & 0 & 0 & 0 & 0 & I_{d} & -\frac{1}{\sqrt{\lambda} \sqrt{n}} Z_{2}^\top & 0 & 0 & 0 & 0 & 0 & 0 & 0 & 0 \\
0 & 0 & 0 & 0 & 0 & 0 & 0 & I_{n_2} & -\frac{1}{\sqrt{\lambda} \sqrt{n}} Z_{2} & 0 & 0 & 0 & 0 & 0 & 0 & 0 \\
0 & 0 & 0 & 0 & 0 & -\Sigma_{2}^{\frac{1}{2}} & 0 & 0 & I_{d} & 0 & 0 & 0 & 0 & 0 & 0 & 0 \\
0 & 0 & 0 & 0 & 0 & 0 & 0 & 0 & 0 & I_{d} & \Sigma_{1}^{\frac{1}{2}} & 0 & 0 & \Sigma_{2}^{\frac{1}{2}} & 0 & 0 \\
0 & 0 & 0 & 0 & 0 & 0 & 0 & 0 & 0 & 0 & I_{d} & -\frac{1}{\sqrt{\lambda} \sqrt{n}} Z_{1}^\top & 0 & 0 & 0 & 0 \\
0 & 0 & 0 & 0 & 0 & 0 & 0 & 0 & 0 & 0 & 0 & I_{n_1} & -\frac{1}{\sqrt{\lambda} \sqrt{n}} Z_{1} & 0 & 0 & 0 \\
0 & 0 & 0 & 0 & 0 & 0 & 0 & 0 & 0 & -\Sigma_{1}^{\frac{1}{2}} & 0 & 0 & I_{d} & 0 & 0 & 0 \\
0 & 0 & 0 & 0 & 0 & 0 & 0 & 0 & 0 & 0 & 0 & 0 & 0 & I_{d} & -\frac{1}{\sqrt{\lambda} \sqrt{n}} Z_{2}^\top & 0 \\
0 & 0 & 0 & 0 & 0 & 0 & 0 & 0 & 0 & 0 & 0 & 0 & 0 & 0 & I_{n_2} & -\frac{1}{\sqrt{\lambda} \sqrt{n}} Z_{2} \\
0 & 0 & 0 & 0 & 0 & 0 & 0 & 0 & 0 & -\Sigma_{2}^{\frac{1}{2}} & 0 & 0 & 0 & 0 & 0 & I_{d} \\
\end{array}\right)$}.}
\end{equation}

Like before, the following holds:
\begin{align}
    &\mathbb E \ntrace (H_1/\lambda + H_2/\lambda + I_d)^{-1} \Delta (H_1/\lambda) (H_1/\lambda + H_2/\lambda + I_d)^{-1} \Sigma_2 = G_{1, 25},
\end{align}
with:
\begin{equation}
\begin{split}    
&G_{1, 25} \\
&= \ntrace p_1 \Sigma_1 \Delta (\lambda \Sigma_2 G_{2, 18} + \lambda G_{2, 26} I_d - p_2 \Sigma_2 G_{2, 18} G_{6, 29} + p_2 \Sigma_2 G_{2, 26} G_{6, 22}) \\
&\cdot (p_1 \Sigma_1 G_{2, 18} + p_2 \Sigma_2 G_{6, 22} + \lambda I_d)^{-2}
\end{split}
\end{equation}

By identifying identical entries of $\overline Q^{-1}$, we must have that $\eta_1 = \frac{G_{2, 18}}{\lambda} = \frac{G_{3, 19}}{\lambda} = \frac{G_{11, 27}}{\lambda}, \eta_2 = \frac{G_{6, 22}}{\lambda} = \frac{G_{7, 23}}{\lambda} = \frac{G_{14, 30}}{\lambda}$. For $G_{3, 19}$ and $G_{7, 23}$, we observe that:
\begin{align}
    &G_{3, 19} = - \frac{\lambda}{- \lambda + \phi G_{4, 18}}, \quad 
    G_{4, 18} = - \lambda \ntrace \Sigma_1 \left(p_{1} \Sigma_1 G_{3, 19} + p_{2} \Sigma_2 G_{7, 23} + \lambda I_d\right)^{-1} \\
    &\implies G_{3, 19} = \frac{1}{1 + \phi \ntrace \Sigma_1 \left(p_{1} \Sigma_1 G_{3, 19} + p_{2} \Sigma_2 G_{7, 23} + \lambda I_d\right)^{-1}}, \\
    &G_{7, 23} = - \frac{\lambda}{- \lambda + \phi G_{8, 22}},\quad
    G_{8, 22} = - \lambda \ntrace \Sigma_2 \left(p_{1} \Sigma_1 G_{3, 19} + p_{2} \Sigma_2 G_{7, 23} + \lambda I_d\right)^{-1} \\
    &\implies G_{7, 23} = \frac{1}{1 + \phi \ntrace \Sigma_2 \left(p_{1} \Sigma_1 G_{3, 19} + p_{2} \Sigma_2 G_{7, 23} + \lambda I_d\right)^{-1}}.
\end{align}

By again identifying identical entries of $\overline Q^{-1}$, we further have that $v^{(2)}_1 = -G_{2, 26} = -G_{3, 27}, v^{(2)}_2 = -G_{6, 29} = -G_{7, 30}$. We observe that:
\begin{align}
G_{3, 27} &= \phi \lambda^2 \eta_1^2 \frac{G_{4, 26}}{\lambda}, \\
\frac{G_{4, 26}}{\lambda} &= \lambda^{-2} \ntrace K^{-2} (p_1 \Sigma_1 G_{3, 27} + p_2 \Sigma_2 G_{7, 30} - \lambda \Sigma_2) \Sigma_1, \\
\implies v^{(2)}_1 &= \phi \eta_1^2 \ntrace K^{-2} (v^{(2)}_1 p_1 \Sigma_1 + v^{(2)}_2 p_2 \Sigma_2 + \lambda \Sigma_2) \Sigma_1, \\
G_{7, 30} &= \phi \lambda^2 \eta_2^2 \frac{G_{8, 29}}{\lambda}, \\
\frac{G_{8, 29}}{\lambda} &= \lambda^{-2} \ntrace K^{-2} (p_1 \Sigma_1 G_{3, 27} + p_2 \Sigma_2 G_{7, 30} - \lambda \Sigma_2) \Sigma_2, \\
\implies v^{(2)}_2 &= \phi \eta_2^2 \ntrace K^{-2} (v^{(2)}_1 p_1 \Sigma_1 + v^{(2)}_2 p_2 \Sigma_2 + \lambda \Sigma_2) \Sigma_2.
\end{align}

Putting all the pieces together:
\begin{align}
    B_2^{(1)} (\widehat f) &= 0, \\
    B_2^{(2)} (\widehat f) &= \ntrace p_1 \Sigma_1 \Delta \big(\eta_1 \Sigma_2 - u_1^{(2)} I_d + p_2 \Sigma_2 (\eta_1 u_2^{(2)} - \eta_2 u_1^{(2)}) \big) K^{-2}.
\end{align}

Finally, switching our focus to $B_1^{(3)} (\widehat f)$, the matrix of interest has a linear pencil representation given by (with zero-based indexing):
\begin{align}
    (H_1/\lambda + H_2/\lambda + I_d)^{-1} \Theta (H_1/\lambda + H_2/\lambda + I_d)^{-1} \Sigma_1 = Q^{-1}_{1, 8},
\end{align}
where the linear pencil $Q$ is defined as follows:
\begin{equation}
\rotatebox{270}{%
$Q = $
\resizebox{\linewidth}{!}{%
$\left(\begin{array}{ccccccccccccccc}
I_{d} & 0 & \Sigma_{1}^{\frac{1}{2}} & 0 & 0 & \Sigma_{2}^{\frac{1}{2}} & 0 & 0 & -\Sigma_1 & 0 & 0 & 0 & 0 & 0 & 0 \\
-\Theta & I_{d} & 0 & 0 & 0 & 0 & 0 & 0 & 0 & 0 & 0 & 0 & 0 & 0 & 0 \\
0 & 0 & I_{d} & -\frac{1}{\sqrt{\lambda} \sqrt{n}} Z_{1}^\top & 0 & 0 & 0 & 0 & 0 & 0 & 0 & 0 & 0 & 0 & 0 \\
0 & 0 & 0 & I_{n_1} & -\frac{1}{\sqrt{\lambda} \sqrt{n}} Z_{1} & 0 & 0 & 0 & 0 & 0 & 0 & 0 & 0 & 0 & 0 \\
-\Sigma_{1}^{\frac{1}{2}} & 0 & 0 & 0 & I_{d} & 0 & 0 & 0 & 0 & 0 & 0 & 0 & 0 & 0 & 0 \\
0 & 0 & 0 & 0 & 0 & I_{d} & -\frac{1}{\sqrt{\lambda} \sqrt{n}} Z_{2}^\top & 0 & 0 & 0 & 0 & 0 & 0 & 0 & 0 \\
0 & 0 & 0 & 0 & 0 & 0 & I_{n_2} & -\frac{1}{\sqrt{\lambda} \sqrt{n}} Z_{2} & 0 & 0 & 0 & 0 & 0 & 0 & 0 \\
-\Sigma_{2}^{\frac{1}{2}} & 0 & 0 & 0 & 0 & 0 & 0 & I_{d} & 0 & 0 & 0 & 0 & 0 & 0 & 0 \\
0 & 0 & 0 & 0 & 0 & 0 & 0 & 0 & I_{d} & \Sigma_{1}^{\frac{1}{2}} & 0 & 0 & \Sigma_{2}^{\frac{1}{2}} & 0 & 0 \\
0 & 0 & 0 & 0 & 0 & 0 & 0 & 0 & 0 & I_{d} & -\frac{1}{\sqrt{\lambda} \sqrt{n}} Z_{1}^\top & 0 & 0 & 0 & 0 \\
0 & 0 & 0 & 0 & 0 & 0 & 0 & 0 & 0 & 0 & I_{n_1} & -\frac{1}{\sqrt{\lambda} \sqrt{n}} Z_{1} & 0 & 0 & 0 \\
0 & 0 & 0 & 0 & 0 & 0 & 0 & 0 & -\Sigma_{1}^{\frac{1}{2}} & 0 & 0 & I_{d} & 0 & 0 & 0 \\
0 & 0 & 0 & 0 & 0 & 0 & 0 & 0 & 0 & 0 & 0 & 0 & I_{d} & -\frac{1}{\sqrt{\lambda} \sqrt{n}} Z_{2}^\top & 0 \\
0 & 0 & 0 & 0 & 0 & 0 & 0 & 0 & 0 & 0 & 0 & 0 & 0 & I_{n_2} & -\frac{1}{\sqrt{\lambda} \sqrt{n}} Z_{2} \\
0 & 0 & 0 & 0 & 0 & 0 & 0 & 0 & -\Sigma_{2}^{\frac{1}{2}} & 0 & 0 & 0 & 0 & 0 & I_{d} \\
\end{array}\right)$}.}
\end{equation}
The following holds:
\begin{align}
    (H_1/\lambda + H_2/\lambda + I_d)^{-1} \Theta (H_1/\lambda + H_2/\lambda + I_d)^{-1} \Sigma_1 = G_{1, 23},
\end{align}
with $G_{1, 23} = \lambda \ntrace \Theta (-p_1 \Sigma_1 G_{2, 24} - p_2 \Sigma_2 G_{5, 27} + \lambda \Sigma_1) (p_1 \Sigma_1 G_{2, 17} + p_2 \Sigma_2 G_{5, 20} + \lambda I_d)^{-2}$.

By identifying identical entries of $\overline Q^{-1}$ and following similar steps as before, we must have the identification $\eta_1 = \frac{G_{2, 17}}{\lambda}, \eta_2 = \frac{G_{5, 20}}{\lambda}$, as well as $v^{(1)}_1 = -G_{2, 24}, v^{(1)}_2 = -G_{5, 27}$. Therefore, in conclusion:
\begin{align}
    B_s^{(3)} (\widehat f) = \ntrace \Theta_s (p_1 u_1^{(s)} \Sigma_1 + p_2 u_2^{(s)} \Sigma_2 + \Sigma_s) K^{-2}.
\end{align}

In the limit $p_s \to 1$ (i.e., $p_{s'} \to 0$), we observe that:
\begin{align}
\phi &\to \phi_s, \lambda \to \lambda_s, \\
B_s^{(1)} (\widehat f) &\to 0, \\
B_s^{(2)} (\widehat f) &\to 0, \\
B_s^{(3)} (\widehat f) &= \ntrace \Theta_s (u_s^{(s)} + 1) \Sigma_s K^{-2}, \\
v_s^{(s)} &= \phi_s \eta_s^2 \ntrace K^{-2} (v_s^{(s)} + \lambda_s) \Sigma_s^2 \\
&= \phi_s (v_s^{(s)} + \lambda_s) \ndof_2^{(s)} (\kappa_s) \\
u_s^{(s)} &= \frac{\phi_s \ndof_2^{(s)} (\kappa_s)}{1 - \phi_s \ndof_2^{(s)} (\kappa_s)}, \\
B_s^{(3)} (\widehat f) &= \frac{\kappa_s^2 \ntrace \Theta_s \Sigma_s (\Sigma_s + \kappa_s I_d)^{-2}}{1 - \phi_s \ndof_2^{(s)} (\kappa_s)}, \\
B_s (\widehat f) &\to B_s^{(3)} (\widehat f),
\end{align}
which matches up exactly with $B_s (\widehat f_s)$ as expected.
\end{proof}

\clearpage

\section{Proof of Theorem \ref{thm:odd-theory-rand-proj}}
\label{ref:rand-proj-proof}

\begin{proof}
The gradient of the loss $L$ is given by:
\begin{align*}
\nabla L(\eta) &= \sum_s S^\top X_s^\top (X_s S\eta-Y_s)/n + \lambda\eta =\sum_s S^\top M_s S\eta -\sum_s S^\top X_s^\top Y_s/n + \lambda \eta \\
&= H\eta-\sum_s S^\top X_s^\top Y_s/n,
\end{align*}
where $H=S^\top M S + \lambda I_m \in \mathbb R^{m \times m}$, with $M=M_1+M_2$ and $M_s = X_s^\top X_s/n$.
Thus, setting $R = H^{-1}$, we may write:
\begin{align*}
    \widehat w &= S\widehat\eta = S R S^\top (X_1^\top Y_1 + X_2^\top Y_2)/n\\
    &= SRS^\top (M_1 w^*_1+M_2 w^*_2) +  SRS^\top X_1^\top E_1/n + SR S^\top X_2^\top E_2/n.
\end{align*}

We deduce the following bias-variance decomposition:
\begin{align*}
    \mathbb E \|\widehat w-w^*_s\|_{\Sigma_s}^2 &= B_s (\widehat f) + V_s (\widehat f),\text{ where }\\
    V_s (\widehat f) &= V^{(1)}_s (\widehat f) + V^{(2)}_s (\widehat f),\text{ with }V^{(j)}_s (\widehat f) = \sigma_j^2 \phi \mathbb{E} \ntrace M_j S R S^\top \Sigma_s S R S^\top,\\
    B_s (\widehat f) &= \mathbb E \|SRS^\top (M_1 w^*_1+M_2 w^*_2)-w^*_s\|_{\Sigma_s}^2.
\end{align*}

We can further decompose $B_s (\widehat f)$, first considering the case $s = 1$. We define $\delta = w_2^* - w_1^*$.
\begin{align*}
&\mathbb E \|SRS^\top (M_1 w^*_1+M_2 w^*_2)-w^*_1\|_{\Sigma_1}^2 \\
&= \mathbb E\|(SRS^\top (M_1+M_2) - I_d) w^*_1 + SRS^\top M_2\delta\|_{\Sigma_1}^2\\
&= \mathbb E\|(SRS^\top M - I_d)w_1^*\|_{\Sigma_1}^2 + \mathbb E\|SRS^\top M_2 \delta\|_{\Sigma_1}^2\\
&=\mathbb{E} \ntrace \Theta (M S R S^\top - I_d) \Sigma_1 (SRS^\top M-I_d) + \mathbb{E} \ntrace\Delta M_2 S R S^\top \Sigma_1 S R S^\top M_2\\
&=\mathbb{E} \ntrace\Theta\Sigma_1+ \mathbb{E} 
 \ntrace\Theta M SRS^\top \Sigma_1 SRS^\top M -2\mathbb{E} \ntrace \Theta \Sigma_1 SRS^\top M + \mathbb{E} \ntrace\Delta M_2 S R S^\top \Sigma_1 S R S^\top M_2. 
\end{align*}

We can similarly decompose $B_2$:
\begin{align*}
&\mathbb E \|SRS^\top (M_1 w^*_1+M_2 w^*_2)-w^*_2\|_{\Sigma_2}^2 \\
&= \mathbb E \|SRS^\top (M_1 w^*_1+M_2 w^*_2)-w^*_2\|_{\Sigma_2}^2  \\
&= \mathbb E\|(SRS^\top (M_1 + M_2) - I_d) w^*_2 - SRS^\top M_1\delta\|_{\Sigma_2}^2\\
&= \mathbb E\|(SRS^\top M - I_d)w_2^*\|_{\Sigma_2}^2 + \mathbb E\|SRS^\top M_1 \delta\|_{\Sigma_2}^2 - 2 \mathbb{E} \trace (w_2^*)^\top (M SRS^\top - I_d) \Sigma_2 SRS^\top M_1 \delta  \\
&=\mathbb{E} \ntrace \Theta_2 (M SRS^\top -I_d)\Sigma_2 (S R S^\top M - I_d) + \mathbb{E} \ntrace\Delta M_1 S R S^\top \Sigma_2 S R S^\top M_1 \\
&- 2 \mathbb{E} \ntrace \Delta (M SRS^\top - I_d) \Sigma_2 SRS^\top M_1 \\
&=\mathbb{E} \ntrace\Theta_2\Sigma_2+ \mathbb{E} 
 \ntrace\Theta_2 M SRS^\top \Sigma_2 SRS^\top M -2\mathbb{E} \ntrace \Theta_2 \Sigma_2 SRS^\top M \\
 &+ \mathbb{E} \ntrace\Delta M_1 S R S^\top \Sigma_2 S R S^\top M_1 - 2 \mathbb{E} \ntrace \Delta M SRS^\top \Sigma_2 SRS^\top M_1 + 2 \mathbb{E} \ntrace \Delta \Sigma_2 SRS^\top M_1.
\end{align*}

Furthermore, we observe that:
\begin{align}
&\mathbb{E} \ntrace AMSRS^\top BSRS^\top M \\
&= \mathbb{E} \ntrace AM_1SRS^\top BSRS^\top M_1 + \mathbb{E} \ntrace AM_2SRS^\top BSRS^\top M_2 + 2 \mathbb{E} \ntrace AM_1SRS^\top BSRS^\top M_2,\\
&\mathbb{E} \ntrace A SRS^\top M = \mathbb{E} \ntrace A SRS^\top M_1 + \mathbb{E} \ntrace A SRS^\top M_2.
\end{align}

Hence, we desire deterministic equivalents for the following expressions:
\begin{align}
    r_j^{(1)} (A) &= A S \overline R S^\top \overline M_j, \\
    r_{j}^{(2)} (A, B) &= A \overline M_j S \overline R S^\top B S \overline R S^\top, \\
    r_{j}^{(3)} (A, B) &= A \overline M_j S \overline R S^\top B S \overline R S^\top \overline M_j, \\
    r_{j}^{(4)} (A, B) &= A \overline M_j S \overline R S^\top B S \overline R S^\top \overline M_{j'},
\end{align}
where:
\begin{gather}
    \overline M_j = \Sigma_j^{1/2} Z^\top_j Z_j \Sigma_j^{1/2}, \overline R = (S^\top \overline M S + I_m)^{-1}, \overline M = \overline M_1 + \overline M_2, \\
    \overline M_j = M_j / \lambda, \overline R = \lambda R, \overline M = M / \lambda.
\end{gather}

In summary:
\begin{align}
    V_s^{(j)} (\widehat f) &= \sigma_j^2 \phi \lambda^{-1} 
 \mathbb{E} \ntrace r_j^{(2)} (I_d, \Sigma_s), \\
    B_s (\widehat f) &= \ntrace \Theta_s \Sigma_s \\
    &+ \mathbb{E} \ntrace r_1^{(3)} (\Theta_s, \Sigma_s) + \mathbb{E} \ntrace r_2^{(3)} (\Theta_s, \Sigma_s) + 2 \mathbb{E} \ntrace r_1^{(4)} (\Theta_s, \Sigma_s) \\
    &- 2 \mathbb{E} \ntrace r_1^{(1)} (\Theta_s \Sigma_s) - 2 \mathbb{E} \ntrace r_2^{(1)} (\Theta_s \Sigma_s) \\
    &+ \mathbb{E} \ntrace r_{s'}^{(3)} (\Delta, \Sigma_s) \\
    &- 2 \begin{cases}
    0, & s = 1, \\
    \mathbb{E} \ntrace r_1^{(3)} (\Delta, \Sigma_2) + \mathbb{E} \ntrace r_2^{(4)} (\Delta, \Sigma_2) - \mathbb{E} \ntrace r_1^{(1)} (\Delta \Sigma_2), & s = 2
    \end{cases}.
\end{align}

\subsection{Computing \texorpdfstring{$\mathbb{E} \ntrace r_j^{(1)}$}{}}

WLOG, we focus on $r_1^{(1)}$. The matrix of interest has a linear pencil representation given by (with zero-based indexing):
\begin{align}
    r_1^{(1)} = Q_{1, 10}^{-1},
\end{align}
where the linear pencil $Q$ is defined as follows:
\begin{equation}
Q = 
\resizebox{0.8\linewidth}{!}{%
$\left(\begin{array}{ccccccccccc}
I_{d} & 0 & -S & 0 & 0 & 0 & 0 & 0 & 0 & 0 & 0 \\
-A & I_{d} & 0 & 0 & 0 & 0 & 0 & 0 & 0 & 0 & 0 \\
0 & 0 & I_{m} & S^\top & 0 & 0 & 0 & 0 & 0 & 0 & 0 \\
0 & 0 & 0 & I_{d} & -\Sigma_{1}^{\frac{1}{2}} & 0 & 0 & -\Sigma_{2}^{\frac{1}{2}} & 0 & 0 & 0 \\
0 & 0 & 0 & 0 & I_{d} & -\frac{1}{\sqrt{\lambda}} Z_{1}^\top & 0 & 0 & 0 & 0 & 0 \\
0 & 0 & 0 & 0 & 0 & I_{n_1} & -\frac{1}{\sqrt{\lambda}} Z_{1} & 0 & 0 & 0 & 0 \\
-\Sigma_{1}^{\frac{1}{2}} & 0 & 0 & 0 & 0 & 0 & I_{d} & 0 & 0 & 0 & \Sigma_{1}^{\frac{1}{2}} \\
0 & 0 & 0 & 0 & 0 & 0 & 0 & I_{d} & -\frac{1}{\sqrt{\lambda}} Z_{2}^\top & 0 & 0 \\
0 & 0 & 0 & 0 & 0 & 0 & 0 & 0 & I_{n_2} & -\frac{1}{\sqrt{\lambda}} Z_{2} & 0 \\
-\Sigma_{2}^{\frac{1}{2}} & 0 & 0 & 0 & 0 & 0 & 0 & 0 & 0 & I_{d} & 0 \\
0 & 0 & 0 & 0 & 0 & 0 & 0 & 0 & 0 & 0 & I_{d} \\
\end{array}\right)$}.
\end{equation}

Using the tools of OVFPT, the following holds:
\begin{align}
\mathbb{E} \ntrace r_1^{(1)} = G_{1, 21},
\end{align} 
with:
\begin{align}
    G_{1, 21} = \ntrace \gamma p_1 \Sigma_1 A G_{2, 13} G_{5, 16} (\gamma G_{2, 13} (p_1 \Sigma_1 G_{5, 16} + p_2 G_{8, 19}) + \lambda I_d)^{-1}.
\end{align}
For $G_{5, 16}$ and $G_{8, 19}$, we observe that:
\begin{align}
&G_{5, 16} = \frac{-\lambda}{-\lambda + \phi G_{6, 15}}, \quad G_{6, 15} = -\lambda \gamma G_{2, 13} \ntrace \Sigma_1 (\gamma G_{2, 13} (p_1 \Sigma_1 G_{5, 16} + p_2 \Sigma_2 G_{8, 19}) + \lambda I_d)^{-1}, \\
&\implies G_{5, 16} = \frac{1}{1 + \psi G_{2, 13} \ntrace \Sigma_1 (\gamma G_{2, 13} (p_1 \Sigma_1 G_{5, 16} + p_2 \Sigma_2 G_{8, 19}) + \lambda I_d)^{-1}}, \\
&G_{8, 19} = \frac{-\lambda}{-\lambda + \phi G_{9, 18}}, \quad G_{9, 18} = -\lambda \gamma G_{2, 13} \ntrace \Sigma_2 (\gamma G_{2, 13} (p_1 \Sigma_1 G_{5, 16} + p_2 \Sigma_2 G_{8, 19}) + \lambda I_d)^{-1}, \\
&G_{8, 19} = \frac{1}{1 + \psi G_{2, 13} \ntrace \Sigma_2 (\gamma G_{2, 13} (p_1 \Sigma_1 G_{5, 16} + p_2 \Sigma_2 G_{8, 19}) + \lambda I_d)^{-1}}.
\end{align}

We define $e_1 = G_{5, 16}, e_2 = G_{8, 19}$, with $e_1 \geq 0, e_2 \geq 0$. We further observe that:
\begin{align}
G_{2, 13} &= \frac{1}{1 + G_{3, 11}}, \\
G_{3, 11} &= \ntrace (p_1 \Sigma_1 G_{5, 16} + p_2 \Sigma_2 G_{8, 19}) (\gamma G_{2, 13} (p_1 \Sigma_1 G_{5, 16} + p_2 \Sigma_2 G_{8, 19}) + \lambda I_d)^{-1}.
\end{align}

We define $\tau = G_{2, 13} \geq 0$. We further define $L = p_1 e_1 \Sigma_1 + p_2 e_2 \Sigma_2, K = \gamma \tau L + \lambda I_d$. Therefore, we have the following system of equations:
\begin{align}
e_s &= \frac{1}{1 + \psi \tau \ntrace \Sigma_s K^{-1}}, \tau = \frac{1}{1 + \ntrace L K^{-1}}.
\end{align}
In conclusion:
\begin{align}
    \mathbb{E} \ntrace r_j^{(1)} = p_j \gamma e_j \tau \ntrace A \Sigma_j K^{-1}. 
\end{align}

\subsection{Computing \texorpdfstring{$\mathbb{E} \ntrace r_j^{(2)}$}{}}

WLOG, we focus on $r_1^{(2)}$. The matrix of interest has a linear pencil representation given by (with zero-based indexing):
\begin{align}
    r_1^{(2)} = -Q^{-1}_{1, 13},
\end{align}
where the linear pencil $Q$ is defined as follows:
\begin{equation}
\rotatebox{270}{%
$Q = $
\resizebox{\linewidth}{!}{%
$\left(\begin{array}{cccccccccccccccccccc}
I_{d} & 0 & -\Sigma_{1}^{\frac{1}{2}} & 0 & 0 & 0 & 0 & 0 & 0 & 0 & 0 & 0 & 0 & 0 & 0 & 0 & 0 & 0 & 0 & 0 \\
-A & I_{d} & 0 & 0 & 0 & 0 & 0 & 0 & 0 & 0 & 0 & 0 & 0 & 0 & 0 & 0 & 0 & 0 & 0 & 0 \\
0 & 0 & I_{d} & -\frac{1}{\sqrt{\lambda}} Z_{1}^\top & 0 & 0 & 0 & 0 & 0 & 0 & 0 & 0 & 0 & 0 & 0 & 0 & 0 & 0 & 0 & 0 \\
0 & 0 & 0 & I_{n_1} & -\frac{1}{\sqrt{\lambda}} Z_{1} & 0 & 0 & 0 & 0 & 0 & 0 & 0 & 0 & 0 & 0 & 0 & 0 & 0 & 0 & 0 \\
0 & 0 & 0 & 0 & I_{d} & -\Sigma_{1}^{\frac{1}{2}} & 0 & 0 & 0 & 0 & 0 & 0 & 0 & 0 & 0 & 0 & 0 & 0 & 0 & 0 \\
0 & 0 & 0 & 0 & 0 & I_{d} & -S & 0 & 0 & 0 & 0 & 0 & 0 & 0 & 0 & 0 & 0 & 0 & 0 & 0 \\
0 & 0 & 0 & 0 & 0 & 0 & I_{m} & S^\top & 0 & 0 & 0 & 0 & 0 & 0 & 0 & 0 & 0 & 0 & 0 & 0 \\
0 & 0 & -\Sigma_{1}^{\frac{1}{2}} & 0 & 0 & 0 & 0 & I_{d} & -\Sigma_{2}^{\frac{1}{2}} & 0 & 0 & B & 0 & 0 & 0 & 0 & 0 & 0 & 0 & 0 \\
0 & 0 & 0 & 0 & 0 & 0 & 0 & 0 & I_{d} & -\frac{1}{\sqrt{\lambda}} Z_{2}^\top & 0 & 0 & 0 & 0 & 0 & 0 & 0 & 0 & 0 & 0 \\
0 & 0 & 0 & 0 & 0 & 0 & 0 & 0 & 0 & I_{n_2} & -\frac{1}{\sqrt{\lambda}} Z_{2} & 0 & 0 & 0 & 0 & 0 & 0 & 0 & 0 & 0 \\
0 & 0 & 0 & 0 & 0 & -\Sigma_{2}^{\frac{1}{2}} & 0 & 0 & 0 & 0 & I_{d} & 0 & 0 & 0 & 0 & 0 & 0 & 0 & 0 & 0 \\
0 & 0 & 0 & 0 & 0 & 0 & 0 & 0 & 0 & 0 & 0 & I_{d} & -S & 0 & 0 & 0 & 0 & 0 & 0 & 0 \\
0 & 0 & 0 & 0 & 0 & 0 & 0 & 0 & 0 & 0 & 0 & 0 & I_{m} & S^\top & 0 & 0 & 0 & 0 & 0 & 0 \\
0 & 0 & 0 & 0 & 0 & 0 & 0 & 0 & 0 & 0 & 0 & 0 & 0 & I_{d} & -\Sigma_{1}^{\frac{1}{2}} & 0 & 0 & -\Sigma_{2}^{\frac{1}{2}} & 0 & 0 \\
0 & 0 & 0 & 0 & 0 & 0 & 0 & 0 & 0 & 0 & 0 & 0 & 0 & 0 & I_{d} & -\frac{1}{\sqrt{\lambda}} Z_{1}^\top & 0 & 0 & 0 & 0 \\
0 & 0 & 0 & 0 & 0 & 0 & 0 & 0 & 0 & 0 & 0 & 0 & 0 & 0 & 0 & I_{n_1} & -\frac{1}{\sqrt{\lambda}} Z_{1} & 0 & 0 & 0 \\
0 & 0 & 0 & 0 & 0 & 0 & 0 & 0 & 0 & 0 & 0 & -\Sigma_{1}^{\frac{1}{2}} & 0 & 0 & 0 & 0 & I_{d} & 0 & 0 & 0 \\
0 & 0 & 0 & 0 & 0 & 0 & 0 & 0 & 0 & 0 & 0 & 0 & 0 & 0 & 0 & 0 & 0 & I_{d} & -\frac{1}{\sqrt{\lambda}} Z_{2}^\top & 0 \\
0 & 0 & 0 & 0 & 0 & 0 & 0 & 0 & 0 & 0 & 0 & 0 & 0 & 0 & 0 & 0 & 0 & 0 & I_{n_2} & -\frac{1}{\sqrt{\lambda}} Z_{2} \\
0 & 0 & 0 & 0 & 0 & 0 & 0 & 0 & 0 & 0 & 0 & -\Sigma_{2}^{\frac{1}{2}} & 0 & 0 & 0 & 0 & 0 & 0 & 0 & I_{d} \\
\end{array}\right)$}}.
\end{equation}

The following holds:
\begin{align}
    \mathbb{E} \ntrace r_1^{(2)} = -G_{1, 33},
\end{align}
with:
\begin{align}
    G_{1, 33} &= -p_1 \ntrace A \Sigma_1 P_1 P^{-1}_2, \\
    P_1 &= \gamma \lambda B G_{3, 23} G_{6, 26} G_{12, 32} - \gamma p_2 \Sigma_2 G_{3, 23} G_{6, 26} G_{9, 38} G_{12, 32} \\
    &+ \gamma p_2 \Sigma_2 G_{3, 35} G_{6, 26} G_{9, 29} G_{12, 32} + \lambda G_{3, 23} G_{6, 32} I_d + \lambda G_{3, 15} G_{12, 32} I_d, \\
    P_2 &= (\gamma G_{6, 26} (p_1 \Sigma_1 G_{3, 23} + \gamma p_2 \Sigma_2 G_{9, 29}) + \lambda I_d) \\
    &\cdot (\gamma G_{12, 32} (p_1 \Sigma_1 G_{15, 35} + p_2 \Sigma_2 G_{18, 38}) + \lambda I_d).
\end{align}

Following similar steps as before and recognizing identifications, we arrive at that:
\begin{align}
    e_1 &= G_{3, 23} = G_{15, 35}, \\
    e_2 &= G_{9, 29} = G_{18, 38}, \\
    \tau &= G_{6, 26} = G_{12, 32}.
\end{align}

We now focus on the remaining terms. We observe that:
\begin{align}
    G_{3, 35} &= \phi e_1^2 \frac{G_{4, 14}}{\lambda} \\
    \frac{G_{4, 14}}{\lambda} &= \gamma \ntrace \Sigma_1 (\gamma \tau^2 (p_1 \Sigma_1 G_{3, 35} + p_2 \Sigma_2 G_{9, 38} - \lambda B) - \lambda G_{6, 32} I_d) K^{-2}, \\
    G_{9, 38} &= \phi e_2^2 \frac{G_{10, 37}}{\lambda}, \\
    \frac{G_{10, 37}}{\lambda} &= \gamma \ntrace \Sigma_2 (\gamma \tau^2 (p_1 \Sigma_1 G_{3, 35} + p_2 \Sigma_2 G_{9, 38} - \lambda B) - \lambda G_{6, 32} I_d) K^{-2}.
\end{align}

We define $u_1 = -\frac{G_{3, 35}}{\lambda}, u_2 = -\frac{G_{9, 38}}{\lambda}$, with $u_1 \leq 0, u_2 \leq 0$. We further define $D = p_1 u_1 \Sigma_1 + p_2 u_2 \Sigma_2 + B$. We now observe that:
\begin{align}
G_{6, 32} &= -\frac{G_{7, 31}}{(G_{7, 25} + 1)(G_{13, 31} + 1)} = -\tau^2 G_{7, 31}, \\
G_{7, 31} &= -\ntrace (\gamma G_{6, 32} L^2 + \lambda^2 D) K^{-2}.
\end{align}

Defining $\rho = G_{6, 32}$, we must have the following system of equations:
\begin{align}
    u_s &= \psi e_s^2 \ntrace \Sigma_s (\gamma \tau^2 D + \rho I_d) K^{-2}, \\
    \rho &= \tau^2 \ntrace (\gamma \rho L^2 + \lambda^2 D) K^{-2}.
\end{align}

In conclusion:
\begin{align}
    P_2 &= K^2, \\
    -P_1 &= \lambda \gamma e_1 \tau^2 B + \lambda \gamma \tau^2 p_2 \Sigma_2 (e_1 u_2 - e_2 u_1) + \lambda e_1 \rho I_d - \lambda^2 u_1 \tau I_d, \\
    \mathbb{E} \ntrace r_j^{(2)} &= \lambda p_j \gamma \ntrace A \Sigma_j (\gamma e_j \tau^2 B + \gamma \tau^2 p_{j'} \Sigma_{j'} (e_j u_{j'} - e_{j'} u_j) + e_j \rho I_d - \lambda u_j \tau I_d) K^{-2}.
\end{align}

\subsection{Computing \texorpdfstring{$\mathbb{E} \ntrace 
 r_j^{(3)}$}{}}

WLOG, we focus on $r_1^{(3)}$. The matrix of interest has a linear pencil representation given by (with zero-based indexing):
\begin{align}
    r_1^{(3)} = Q^{-1}_{1, 20},
\end{align}
where the linear pencil $Q$ is defined as follows:
\begin{equation}
\rotatebox{270}{%
$Q = $
\resizebox{\linewidth}{!}{%
$\left(\begin{array}{ccccccccccccccccccccc}
I_{d} & 0 & -\Sigma_{1}^{\frac{1}{2}} & 0 & 0 & 0 & 0 & 0 & 0 & 0 & 0 & 0 & 0 & 0 & 0 & 0 & 0 & 0 & 0 & 0 & 0 \\
-A & I_{d} & 0 & 0 & 0 & 0 & 0 & 0 & 0 & 0 & 0 & 0 & 0 & 0 & 0 & 0 & 0 & 0 & 0 & 0 & 0 \\
0 & 0 & I_{d} & -\frac{1}{\sqrt{\lambda}} Z_{1}^\top & 0 & 0 & 0 & 0 & 0 & 0 & 0 & 0 & 0 & 0 & 0 & 0 & 0 & 0 & 0 & 0 & 0 \\
0 & 0 & 0 & I_{n_1} & -\frac{1}{\sqrt{\lambda}} Z_{1} & 0 & 0 & 0 & 0 & 0 & 0 & 0 & 0 & 0 & 0 & 0 & 0 & 0 & 0 & 0 & 0 \\
0 & 0 & 0 & 0 & I_{d} & -\Sigma_{1}^{\frac{1}{2}} & 0 & 0 & 0 & 0 & 0 & 0 & 0 & 0 & 0 & 0 & 0 & 0 & 0 & 0 & 0 \\
0 & 0 & 0 & 0 & 0 & I_{d} & -S & 0 & 0 & 0 & 0 & 0 & 0 & 0 & 0 & 0 & 0 & 0 & 0 & 0 & 0 \\
0 & 0 & 0 & 0 & 0 & 0 & I_{m} & S^\top & 0 & 0 & 0 & 0 & 0 & 0 & 0 & 0 & 0 & 0 & 0 & 0 & 0 \\
0 & 0 & -\Sigma_{1}^{\frac{1}{2}} & 0 & 0 & 0 & 0 & I_{d} & -\Sigma_{2}^{\frac{1}{2}} & 0 & 0 & B & 0 & 0 & 0 & 0 & 0 & 0 & 0 & 0 & 0 \\
0 & 0 & 0 & 0 & 0 & 0 & 0 & 0 & I_{d} & -\frac{1}{\sqrt{\lambda}} Z_{2}^\top & 0 & 0 & 0 & 0 & 0 & 0 & 0 & 0 & 0 & 0 & 0 \\
0 & 0 & 0 & 0 & 0 & 0 & 0 & 0 & 0 & I_{n_2} & -\frac{1}{\sqrt{\lambda}} Z_{2} & 0 & 0 & 0 & 0 & 0 & 0 & 0 & 0 & 0 & 0 \\
0 & 0 & 0 & 0 & 0 & -\Sigma_{2}^{\frac{1}{2}} & 0 & 0 & 0 & 0 & I_{d} & 0 & 0 & 0 & 0 & 0 & 0 & 0 & 0 & 0 & 0 \\
0 & 0 & 0 & 0 & 0 & 0 & 0 & 0 & 0 & 0 & 0 & I_{d} & -S & 0 & 0 & 0 & 0 & 0 & 0 & 0 & 0 \\
0 & 0 & 0 & 0 & 0 & 0 & 0 & 0 & 0 & 0 & 0 & 0 & I_{m} & S^\top & 0 & 0 & 0 & 0 & 0 & 0 & 0 \\
0 & 0 & 0 & 0 & 0 & 0 & 0 & 0 & 0 & 0 & 0 & 0 & 0 & I_{d} & -\Sigma_{1}^{\frac{1}{2}} & 0 & 0 & -\Sigma_{2}^{\frac{1}{2}} & 0 & 0 & 0 \\
0 & 0 & 0 & 0 & 0 & 0 & 0 & 0 & 0 & 0 & 0 & 0 & 0 & 0 & I_{d} & -\frac{1}{\sqrt{\lambda}} Z_{1}^\top & 0 & 0 & 0 & 0 & 0 \\
0 & 0 & 0 & 0 & 0 & 0 & 0 & 0 & 0 & 0 & 0 & 0 & 0 & 0 & 0 & I_{n_1} & -\frac{1}{\sqrt{\lambda}} Z_{1} & 0 & 0 & 0 & 0 \\
0 & 0 & 0 & 0 & 0 & 0 & 0 & 0 & 0 & 0 & 0 & -\Sigma_{1}^{\frac{1}{2}} & 0 & 0 & 0 & 0 & I_{d} & 0 & 0 & 0 & \Sigma_{1}^{\frac{1}{2}} \\
0 & 0 & 0 & 0 & 0 & 0 & 0 & 0 & 0 & 0 & 0 & 0 & 0 & 0 & 0 & 0 & 0 & I_{d} & -\frac{1}{\sqrt{\lambda}} Z_{2}^\top & 0 & 0 \\
0 & 0 & 0 & 0 & 0 & 0 & 0 & 0 & 0 & 0 & 0 & 0 & 0 & 0 & 0 & 0 & 0 & 0 & I_{n_2} & -\frac{1}{\sqrt{\lambda}} Z_{2} & 0 \\
0 & 0 & 0 & 0 & 0 & 0 & 0 & 0 & 0 & 0 & 0 & -\Sigma_{2}^{\frac{1}{2}} & 0 & 0 & 0 & 0 & 0 & 0 & 0 & I_{d} & 0 \\
0 & 0 & 0 & 0 & 0 & 0 & 0 & 0 & 0 & 0 & 0 & 0 & 0 & 0 & 0 & 0 & 0 & 0 & 0 & 0 & I_{d} \\
\end{array}\right)$}}.
\end{equation}

It holds that $\mathbb{E} \ntrace r_1^{(3)} = G_{1, 41}$. We immediately observe that:
\begin{align}
    e_1 &= G_{3, 24}, G_{15, 36}, \\
    e_2 &= G_{9, 30}, G_{18, 39}, \\
    \tau &= G_{6, 27}, G_{12, 33}, \\
    u_1 &= -\frac{G_{3, 36}}{\lambda}, \\
    u_2 &= -\frac{G_{9, 39}}{\lambda}, \\
    \rho &= G_{6, 33}.
\end{align}

In conclusion:
\begin{align}
    \mathbb{E} \ntrace r_j^{(3)} = p_j \ntrace A \Sigma_j (\gamma e_j^2 p_j \Sigma_j (\gamma \tau^2 u_{j'} p_{j'} \Sigma_{j'} + \gamma \tau^2 B + \rho I_d) + u_j (\gamma e_{j'} \tau p_{j'} \Sigma_{j'} + \lambda I_d)^2) K^{-2}.
\end{align}

\subsection{Computing \texorpdfstring{$\mathbb{E} \ntrace r_j^{(4)}$}{}}

WLOG, we focus on $r_1^{(4)}$. The matrix of interest has a linear pencil representation given by (with zero-based indexing):
\begin{align}
    r_1^{(4)} = Q^{-1}_{1, 20},
\end{align}
where the linear pencil $Q$ is defined as follows:
\begin{equation}
\rotatebox{270}{%
$Q = $
\resizebox{\linewidth}{!}{%
$\left(\begin{array}{ccccccccccccccccccccc}
I_{d} & 0 & -\Sigma_{1}^{\frac{1}{2}} & 0 & 0 & 0 & 0 & 0 & 0 & 0 & 0 & 0 & 0 & 0 & 0 & 0 & 0 & 0 & 0 & 0 & 0 \\
-A & I_{d} & 0 & 0 & 0 & 0 & 0 & 0 & 0 & 0 & 0 & 0 & 0 & 0 & 0 & 0 & 0 & 0 & 0 & 0 & 0 \\
0 & 0 & I_{d} & -\frac{1}{\sqrt{\lambda}} Z_{1}^\top & 0 & 0 & 0 & 0 & 0 & 0 & 0 & 0 & 0 & 0 & 0 & 0 & 0 & 0 & 0 & 0 & 0 \\
0 & 0 & 0 & I_{n_1} & -\frac{1}{\sqrt{\lambda}} Z_{1} & 0 & 0 & 0 & 0 & 0 & 0 & 0 & 0 & 0 & 0 & 0 & 0 & 0 & 0 & 0 & 0 \\
0 & 0 & 0 & 0 & I_{d} & -\Sigma_{1}^{\frac{1}{2}} & 0 & 0 & 0 & 0 & 0 & 0 & 0 & 0 & 0 & 0 & 0 & 0 & 0 & 0 & 0 \\
0 & 0 & 0 & 0 & 0 & I_{d} & -S & 0 & 0 & 0 & 0 & 0 & 0 & 0 & 0 & 0 & 0 & 0 & 0 & 0 & 0 \\
0 & 0 & 0 & 0 & 0 & 0 & I_{m} & S^\top & 0 & 0 & 0 & 0 & 0 & 0 & 0 & 0 & 0 & 0 & 0 & 0 & 0 \\
0 & 0 & -\Sigma_{1}^{\frac{1}{2}} & 0 & 0 & 0 & 0 & I_{d} & -\Sigma_{2}^{\frac{1}{2}} & 0 & 0 & B & 0 & 0 & 0 & 0 & 0 & 0 & 0 & 0 & 0 \\
0 & 0 & 0 & 0 & 0 & 0 & 0 & 0 & I_{d} & -\frac{1}{\sqrt{\lambda}} Z_{2}^\top & 0 & 0 & 0 & 0 & 0 & 0 & 0 & 0 & 0 & 0 & 0 \\
0 & 0 & 0 & 0 & 0 & 0 & 0 & 0 & 0 & I_{n_2} & -\frac{1}{\sqrt{\lambda}} Z_{2} & 0 & 0 & 0 & 0 & 0 & 0 & 0 & 0 & 0 & 0 \\
0 & 0 & 0 & 0 & 0 & -\Sigma_{2}^{\frac{1}{2}} & 0 & 0 & 0 & 0 & I_{d} & 0 & 0 & 0 & 0 & 0 & 0 & 0 & 0 & 0 & 0 \\
0 & 0 & 0 & 0 & 0 & 0 & 0 & 0 & 0 & 0 & 0 & I_{d} & -S & 0 & 0 & 0 & 0 & 0 & 0 & 0 & 0 \\
0 & 0 & 0 & 0 & 0 & 0 & 0 & 0 & 0 & 0 & 0 & 0 & I_{m} & S^\top & 0 & 0 & 0 & 0 & 0 & 0 & 0 \\
0 & 0 & 0 & 0 & 0 & 0 & 0 & 0 & 0 & 0 & 0 & 0 & 0 & I_{d} & -\Sigma_{1}^{\frac{1}{2}} & 0 & 0 & -\Sigma_{2}^{\frac{1}{2}} & 0 & 0 & 0 \\
0 & 0 & 0 & 0 & 0 & 0 & 0 & 0 & 0 & 0 & 0 & 0 & 0 & 0 & I_{d} & -\frac{1}{\sqrt{\lambda}} Z_{1}^\top & 0 & 0 & 0 & 0 & 0 \\
0 & 0 & 0 & 0 & 0 & 0 & 0 & 0 & 0 & 0 & 0 & 0 & 0 & 0 & 0 & I_{n_1} & -\frac{1}{\sqrt{\lambda}} Z_{1} & 0 & 0 & 0 & 0 \\
0 & 0 & 0 & 0 & 0 & 0 & 0 & 0 & 0 & 0 & 0 & -\Sigma_{1}^{\frac{1}{2}} & 0 & 0 & 0 & 0 & I_{d} & 0 & 0 & 0 & 0 \\
0 & 0 & 0 & 0 & 0 & 0 & 0 & 0 & 0 & 0 & 0 & 0 & 0 & 0 & 0 & 0 & 0 & I_{d} & -\frac{1}{\sqrt{\lambda}} Z_{2}^\top & 0 & 0 \\
0 & 0 & 0 & 0 & 0 & 0 & 0 & 0 & 0 & 0 & 0 & 0 & 0 & 0 & 0 & 0 & 0 & 0 & I_{n_2} & -\frac{1}{\sqrt{\lambda}} Z_{2} & 0 \\
0 & 0 & 0 & 0 & 0 & 0 & 0 & 0 & 0 & 0 & 0 & -\Sigma_{2}^{\frac{1}{2}} & 0 & 0 & 0 & 0 & 0 & 0 & 0 & I_{d} & \Sigma_{2}^{\frac{1}{2}} \\
0 & 0 & 0 & 0 & 0 & 0 & 0 & 0 & 0 & 0 & 0 & 0 & 0 & 0 & 0 & 0 & 0 & 0 & 0 & 0 & I_{d} \\
\end{array}\right)$}.}
\end{equation}

It holds that $\mathbb{E} \ntrace r_1^{(4)} = G_{1, 41}$. We immediately observe that:
\begin{align}
    e_1 &= G_{3, 24}, G_{15, 36}, \\
    e_2 &= G_{9, 30}, G_{18, 39}, \\
    \tau &= G_{6, 27}, G_{12, 33}, \\
    u_1 &= -\frac{G_{3, 36}}{\lambda}, \\
    u_2 &= -\frac{G_{9, 39}}{\lambda}, \\
    \rho &= G_{6, 33}.
\end{align}

In conclusion:
\begin{align}
    \mathbb{E} \ntrace r_j^{(4)} &= p_j \gamma p_{j'} \ntrace \Sigma_j \Sigma_{j'} A (\gamma \tau^2 (B e_j e_{j'} - p_j \Sigma_j e_j^2 u_{j'} - p_{j'} \Sigma_{j'} e_{j'}^2 u_j) \\
    &- \lambda \tau (e_j u_{j'} + e_{j'} u_j) I_d + e_j e_{j'} \rho I_d) K^{-2}.
\end{align}
\end{proof}

\clearpage

\section{Theorem \ref{thm:edd-theory-rand-proj}}
\label{sec:edd-random-proj-proof}

\begin{definition}
\label{def:cosmo-constants-edd}
Let $(e_1,e_2,\tau_1,\tau_2,u_1,u_2,\rho_1,\rho_2)$ is be unique positive solution to the following system of fixed-point equations:
\begin{gather}
    e_s = \frac{1}{1 + \psi_s \tau_s \ntrace \Sigma_s (\gamma \tau_s e_s \Sigma_s + \lambda_s I_d)^{-1}},\text{ for }s\in\{1,2\} \\
    \tau_s = \frac{1}{1 + \ntrace e_s \Sigma_s (\gamma \tau_s e_s \Sigma_s + \lambda_s I_d)^{-1}},\text{ for }s\in\{1,2\} \\
    u_s = \psi_s e_s^2 \ntrace \Sigma_s (\gamma \tau_s^2 (u_s + 1) \Sigma_s + \rho_s I_d) (\gamma \tau_s e_s \Sigma_s + \lambda_s I_d)^{-2},\text{ for }s\in\{1,2\} \\
    \rho_s = \tau_s^2 \ntrace (\gamma \rho_s (e_s \Sigma_s)^2 + \lambda_s^2 (u_s + 1) \Sigma_s) (\gamma \tau_s e_s \Sigma_s + \lambda_s I_d)^{-2},\text{ for }s\in\{1,2\}.
\end{gather}

For deterministic $d \times d$ PSD matrices $A$ and $B$, we define the following auxiliary quantities:
\begin{align}
    h_j^{(1)} (A) &:= \gamma e_j \tau_j \ntrace A \Sigma_j (\gamma \tau_j e_j \Sigma_j + \lambda_j I_d)^{-1}, \\
    h_j^{(2)} (A) &:= \gamma \ntrace A \Sigma_j (\gamma e_j \tau_j^2 \Sigma_j + e_j \rho_j I_d - \lambda_j u_j \tau_j I_d) (\gamma \tau_j e_j \Sigma_j + \lambda_j I_d)^{-2}, \\
    h_j^{(3)} (A) &:= \ntrace A \Sigma_j (\gamma e_j^2 \Sigma_j (\gamma \tau_j^2 \Sigma_j + \rho_j I_d) + \lambda_j^2 u_j I_d) (\gamma \tau_j e_j \Sigma_j + \lambda_j I_d)^{-2}.
\end{align}
\end{definition}

Under Assumptions \ref{ass:commute} and \ref{ass:scaling-random-proj}, it holds that:
    \begin{align}
        R_s (\widehat f_s) &\simeq B_s (\widehat f_s) + V_s (\widehat f_s),
        \text{ with }    V_s (\widehat f_s) = \lim_{p_s \to 1} V_s (\widehat f),\quad  B_s (\widehat f_s) = \lim_{p_s \to 1} B_s (\widehat f).
    \end{align}
More explicitly:
\begin{align}
    V_s (\widehat f_s) &= \sigma_s^2 \phi_s h_s^{(2)} (I_d), \quad
    B_s (\widehat f_s) = \ntrace \Theta_s \Sigma_s + h_s^{(3)} (\Theta_s) - 2 h_s^{(1)} (\Theta_s \Sigma_s).
    \end{align}

\begin{proof}
Theorem \ref{thm:edd-theory-rand-proj} follows from Theorem \ref{thm:odd-theory-rand-proj} in the limit $p_s \to 1$ (i.e., $p_{s'} \to 0$).
\end{proof}

The scalars $e_s, \tau_s, u_s, \rho_s$ can be intuitively interpreted in the setting where a separate model is learned for each group. For ridge regression with random projections and $\lambda \to 0^+$, we show in Equations \ref{eq:e-lin} and \ref{eq:tau-lin} that $e_s, \tau_s$ are related to the normalized first-order degrees of freedom $I_{1, 1}$ of the population covariance matrix $\Sigma_s$. $e_s$ captures the effect of the feature rate $\phi_s$ while $\tau$ captures the effect of the parameterization rate $\gamma$. Similarly, for classical ridge regression, we show in Equation \ref{eq:kap-lin} that $e_s$ is related to the normalized first-order degrees of freedom $\ndof_1^{(s)}$. On the other hand, $u_s$ and $\rho_s$ can be understood as pseudo-variances. Indeed, for ridge regression with random projections and $\lambda \to 0^+$, Equations \ref{eq:u-rho-lin-1} and \ref{eq:u-rho-lin-2} show that $u_s, \rho_s, V_s$ are all related to the normalized second-order degrees of freedom $I_{2, 2}$ of $\Sigma_s$.

\clearpage

\section{Solving Fixed-Point Equations for Theorem \ref{thm:odd-theory}}
\label{sec:fixed-point-complexity}

\subsection{Proportional Covariance Matrices}

When $\lambda \to 0^+$, it is not possible to analytically solve the fixed-point equations for the constants in Definition \ref{def:cosmo-constants} for general $\Sigma_1, \Sigma_2$. As such, we consider a more tractable case where the covariance matrices are proportional, i.e., $\Sigma_1 = a_1 \Sigma$ and $\Sigma_2 = a_2 \Sigma$, for some $\Sigma \in \mathbb R^{d \times d}$.

We define $\theta = \frac{\lambda}{\gamma \tau (a_1 p_1 e_1 + a_2 p_2 e_2)}$ and $\eta = \ntrace \Sigma (\Sigma + \theta I_d)^{-1}$. Then, we have that:
\begin{align}
    1/e_s = e_s' &= 1 + \psi \tau \ntrace \Sigma_s K^{-1} = 1 + \frac{\phi a_s \eta}{a_1 p_1 e_1 + a_2 p_2 e_2}, \\
    1/\tau = \tau' &= 1+\ntrace L K^{-1} = 1+(\eta/\gamma) \tau' = \frac{1}{1 - \eta / \gamma}.
\end{align}

If $\theta_0 = 0$, then $\eta_0 = 1$. Therefore, $e_s' \to 1 + \frac{\phi a_s}{a_1 p_1 e_1 + a_2 p_2 e_2}$, which is a quadratic fixed-point equation. Accounting for the constraint that $e_s > 0$, the fixed-point equation requires that $\phi < 1$. Moreover, $\tau \to 1 - 1 / \gamma$, which requires that $\gamma > 1$. We further observe that $\rho \to (\tau^2 \ntrace \gamma L^2 K^{-2}) \rho$, which implies that $\rho \to 0$. We can then see that, for $c \in \{a_1, a_2\}$:
\begin{align}
    u_s &\to \phi \gamma^2 \tau^2 e_s^2 a_s (a_1 p_1 u_1 + a_2 p_2 u_2 + c) \ntrace \Sigma^2 K^{-2} \\
    &= \frac{\phi e_s^2 a_s (a_1 p_1 u_1 + a_2 p_2 u_2 + c)}{(a_1 p_1 e_1 + a_2 p_2 e_2)^2},
\end{align}
which is a linear fixed-point equation in $u_s$. In contrast, if $\theta_0 > 1$, we have $e_s' = 1 + \frac{\psi \tau a_s \eta \theta}{\lambda}$ and the equation:
\begin{align}
    \gamma \theta &= \frac{\lambda}{(1 - \eta / \gamma) \left(\frac{a_1 p_1}{1 + \frac{\psi (1 - \eta / \gamma) a_1 \eta \theta}{\lambda}} + \frac{a_2 p_2}{1 + \frac{\psi (1 - \eta / \gamma) a_2 \eta \theta}{\lambda}}\right)},
\end{align}
which is a quartic equation in $\eta$. This highlights the difficulties of rigorously isolating the effects of different components on bias amplification. We empirically investigate how different components (e.g., covariance structures, group sizes) affect bias amplification and minority-group bias in Sections \ref{sec:bias-amp-exp} and \ref{sec:min-group-exp}, and extensively validate that our theory predicts these implications.

\subsection{The General Regularized Case}
We now consider the case where the covariance structure is the same for both groups, i.e., $\Sigma_1 = \Sigma_2 = \Sigma$. The calculations presented in this subsection are shared with Appendix H.2 of \citet{dohmatob2024strong}, of which two authors of this work are also authors.

In this setting, it is clear that $e_1=e_2=e$ and $u_1=u_2=u$, where $(\tau,e,u,\rho)$ now satisfy:
\begin{align}
    1/e &= 1+\psi\tau\ntrace\Sigma K^{-1},\quad 1/\tau = 1+\ntrace K_0 K^{-1},\text{ where }K_0:=e\Sigma,\,K := \gamma\tau K_0+\lambda I_d,\\
    u &= \psi e^2 \ntrace\Sigma_1(\gamma\tau^2L'+ \rho I_d)K^{-2},\,\,
    \rho = \tau^2 \ntrace(\gamma \rho K_0^2 + \lambda^2 L')K^{-2},\,\,L' := (1+u)\Sigma.
\end{align}

We first introduce some notation related to the degrees of freedom of $\Sigma$ before proceeding with our theoretical result.

\begin{definition}
    Let $\dof_m (t) = \trace \Sigma^m \left(\Sigma + t I_d \right)^{-m}$ for any positive integer $m$. Furthermore, define $I_{a,b}(t) = \ntrace\Sigma^a(\Sigma + t I_d)^{-b}$ for any positive integers $a,b$.
\end{definition}

\begin{lemma}
\label{lm:uomegaprime}
The scalars $u$ and $\rho' = \rho/(\gamma \tau^2)$ solve the following pair of linear equations:
\begin{equation}
    \begin{split}
    u &=\phi I_{2,2}(\theta)(1+u)+\phi I_{1,2}(\theta) \rho',\\
    \gamma \rho'  &= I_{2,2}(\theta)\rho' + \theta^2 I_{1,2}(\theta)(1+u).
    \end{split}
    \label{eq:uomegaprime}
\end{equation}
Furthermore, the solutions can be explicitly represented as: 
\begin{equation}
    \begin{split}
    u &= \frac{\phi z}{\gamma-\phi z-I_{2,2}(\theta)},\quad 
    \rho' = \frac{\theta^2 I_{2,2}(\theta)}{\gamma-\phi z-I_{2,2}(\theta)},
    \end{split}
\end{equation}
where  $z = I_{2,2}(\theta)(\gamma-I_{2,2}(\theta)) + \theta^2 I_{1,2}(\theta)^2.$

In particular, in the limit $\gamma \to \infty$, it holds that:
\begin{eqnarray}
\theta \simeq \kappa,\quad \rho' \to 0,\quad u \simeq \frac{\phi I_{2,2}(\kappa)}{1-\phi I_{2,2}(\kappa)} \simeq  \frac{\dof_2(\kappa)/n}{1-\dof_2(\kappa)/n},
\end{eqnarray}
where $\kappa>0$ is uniquely satisfies the fixed-point equation $\kappa-\lambda=\kappa \trace\Sigma(\Sigma + \kappa I_d)^{-1}/n$.
\end{lemma}
\begin{proof}
The equations defining these scalars are:
\begin{align}
    u &= \psi e^2 \ntrace\Sigma(\gamma\tau^2 L' + \rho I_d)K^{-2},\\
    \rho &= \tau^2\ntrace(\gamma\rho K_0^2 + \lambda^2 L')K^{-2},
\end{align}
where $K_0=e\Sigma$, $K=\gamma\tau K_0 + \lambda I_d$, and $ L' := u\Sigma + B$. Further, since $B=\Sigma$, we have $ L'=(1+u)\Sigma$. Now, we can rewrite the previous equations like so
\begin{align*}
    u &= \psi e^2\ntrace\Sigma(\gamma\tau^2(1+u)\Sigma + \rho I_d)K^{-2}=\phi\gamma^2 \tau^2 e^2(1+u)\ntrace\Sigma^2K^{-2} + \phi\gamma e^2\rho\ntrace \Sigma K^{-2},\\
    \rho &= \tau^2 \ntrace(\gamma\rho e^2\Sigma^2 + \lambda^2(1+u)\Sigma)K^{-2}=\gamma\tau^2 e^2 \rho \ntrace\Sigma^2 K^{-2} + \lambda^2\tau^2(1+u)\ntrace\Sigma K^{-2}.
\end{align*}
This can be equivalently written as: 
\begin{align}
    u &= \phi (1+u)\gamma^2\tau^2 e^2\ntrace\Sigma^2K^{-2} + \phi \rho' \gamma^2 \tau^2 e^2\ntrace \Sigma K^{-2},\\
    \gamma \rho' &= \rho' \gamma^2 \tau^2 e^2 \ntrace\Sigma^2 K^{-2} + (1+u)\lambda^2\ntrace\Sigma K^{-2}.
\end{align}
Now, observe that:
\begin{align}
 \tau^2 e^2 \ntrace\Sigma^2 K^{-2} &= \ntrace\Sigma^2(\Sigma + \theta I_d)^{-2}/\gamma^2=I_{2,2}(\theta)/\gamma^2,\\
  \tau^2 e^2 \ntrace\Sigma K^{-2} &= \ntrace\Sigma(\Sigma + \theta I_d)^{-2}/\gamma^2=I_{1,2}(\theta)/\gamma^2,\\
\lambda^2 \ntrace\Sigma K^{-2} &= \theta^2 \ntrace \Sigma(\Sigma + \theta I_d)^{-2}=\theta^2 I_{1,2}(\theta),\\
 e^2 \ntrace\Sigma K^{-2} &= \ntrace\Sigma(\Sigma + \theta I_d)^{-2}/(\gamma \tau)^2=I_{1,2}(\theta)/(\gamma\tau)^2,\\*
  \tau^2 \ntrace\Sigma K^{-2} &= \ntrace\Sigma(\Sigma + \theta I_d)^{-2}/(\gamma e)^2=I_{1,2}(\theta)/(\gamma e)^2,
\end{align}
where we have used the definition $\theta = \lambda/(\gamma\tau e)$.
Thus, $u$ and $\rho$ have limiting values which solve the system of linear equations:
\begin{align*}
    u &= \psi \gamma \cdot \gamma^{-2}I_{2,2}(\theta)  (1+u) + \psi \gamma\cdot \gamma^{-2}I_{1,2}\rho'=\phi I_{2,2}(\theta)(1+u)+\phi I_{1,2}(\theta) \rho',\\
    \gamma \rho'  &= I_{2,2}(\theta) \rho' + \theta^2 I_{1,2}(\theta)(1+u)
    =I_{2,2}(\theta)\rho' + \theta^2 I_{1,2}(\theta)(1+u),
\end{align*}
where we have used the identity $\phi\gamma = \psi$. These correspond exactly to the equations given in the lemma. This proves the first part.

For the second part, indeed, $\tau=1-\eta_0/\gamma \to 1$ in the limit $\gamma \to \infty$, and so $\theta \simeq \lambda/(\gamma e)$ which verifies the equation:
$$
\theta \simeq \lambda + \lambda \psi \ntrace\Sigma(\gamma e\Sigma + \lambda)^{-1} = \lambda + \phi\cdot\frac{\lambda}{\gamma e}\ntrace\Sigma(\Sigma + \frac{\lambda}{\gamma e}I_d)^{-1} \simeq \lambda + \theta \trace\Sigma(\Sigma + \theta I_d)^{-1}/n,
$$
i.e., $\theta \simeq \lambda + \theta \dof_1(\theta)/n$ and $\theta>0$.
By comparing with the equation $\kappa-\lambda=\kappa \dof_1(\kappa)/n$  satisfied by $\kappa>0$ in Definition \ref{def:kappa}, we conclude $\theta \simeq \kappa$.

Now, Equation \ref{eq:uomegaprime} becomes $\rho' = 0$, and $u=\phi I_{2,2}(\kappa)(1+u)$, i.e.,
$$
u = \frac{\phi I_{2,2}(\kappa)}{1-\phi I_{2,2}(\kappa)} \simeq \frac{\dof_2(\kappa)/n}{1-\dof_2(\kappa)/n},
$$
as claimed.
\end{proof}

\subsection{Unregularized Limit}
The calculations presented in this subsection are shared with Appendix H.2 of \citet{dohmatob2024strong}, of which two authors of this work are also authors.

Define the following auxiliary quantities:
\begin{eqnarray}
    \theta := \frac{\lambda}{\gamma\tau e},\quad \chi := \frac{\lambda}{\tau},\quad  \kappa := \frac{\lambda}{e}.
\end{eqnarray}
where $\tau$, $e$, $u$, and $\rho$ are as previously defined.

\begin{lemma}
In the limit $\lambda \to 0^+$, we have the following analytic formulae:
    \begin{align}
        \chi &\to \chi_0 = (1-\psi)_+\cdot\gamma\theta_0,\\
        \kappa &\to \kappa_0 = (\psi-1)_+\cdot\theta_0/\phi,\\
        \tau &\to \tau_0 = 1-\eta_0/\gamma,\\
        e &\to e_0 = 1-\phi \eta_0,
    \end{align}
where $\theta_0$ is the unique positive solution of the fixed-point equation $\eta_0 = I_{1,1}(\theta_0)$.
\label{lm:key}
\end{lemma}
\begin{proof}
Observe that $K_0=e\Sigma$ and $K=\gamma\tau K_0 + \lambda I_d = \gamma\tau e \cdot (\Sigma + \theta I_d)$.
Defining $\eta := I_{1,1}(\theta)$, one can then rewrite the equations defining $e$ and $\tau$ as follows:
\begin{align}
    e' &= \frac{\lambda}{e} = \lambda +\psi\tau \lambda \ntrace\Sigma K^{-1} = \lambda + \frac{\psi\tau \lambda }{\gamma\tau e}\ntrace\Sigma(\Sigma + \theta I_d)^{-1} = \lambda +\phi \eta e',\\
    \tau' &= \frac{\lambda}{\tau} = \lambda +\lambda\ntrace K_0 K^{-1} = \lambda + \frac{\lambda e}{\gamma\tau e}\ntrace\Sigma(\Sigma + \theta I_d)^{-1} = \lambda +(\eta/\gamma)\tau'.
\end{align}
We deduce that:
\begin{align}
\label{eq:toto}
    e' &= \frac{\lambda}{1-\phi \eta},\quad
    \tau' = \frac{\lambda}{1-\eta/\gamma},\quad
    \tau'e' = \lambda \gamma\theta.
\end{align}
In particular, the above means that $\eta \le \min(\gamma,1/\phi)$. The last part of equations Equation \ref{eq:toto} can be rewritten as follows:
\begin{eqnarray}
\frac{\lambda}{(1-\phi\eta)(1-\eta/\gamma)} = \gamma\theta,\text{ i.e., }\phi\eta^2-(\phi\gamma + 1)\eta+\gamma-\frac{\lambda}{\theta} = 0.
\end{eqnarray}
This is a quadratic equation for $\eta$ as a function of $\lambda$ and $\theta$, with roots
\begin{eqnarray}
    \eta^\pm = \frac{\phi\gamma + 1 \pm \sqrt{(\phi\gamma + 1)^2-4(\phi\gamma-(\phi/\theta)\lambda)}}{2\phi} = \frac{\psi + 1 \pm \sqrt{(\psi + 1)^2-4(\psi-\phi/\theta')}}{2\phi}.
\end{eqnarray}
Now, for small $\lambda > 0$ and $\psi \ne 1$, we can do a Taylor expansion to get:
\begin{align*}
    \eta^\pm &\simeq \frac{\psi+1 \pm |\psi-1|}{2\phi} \pm \frac{1}{\theta |\psi-1|}\lambda  + O(\lambda^2).
\end{align*}
More explicitly:
\begin{align*}
    \eta^+ &\simeq O(\lambda^2)
    \begin{cases}
        1/\phi+\lambda/((1-\psi)\theta),&\mbox{ if }\psi < 1\\
        \gamma+\lambda/((\psi-1)\theta),&\mbox{ if }\psi > 1
    \end{cases},\\
    \eta^- &\simeq O(\lambda^2) +  \begin{cases}
        \gamma-\lambda/((1-\psi)\theta),&\mbox{ if }\psi < 1\\
        1/\phi-\lambda /((\psi-1)\theta),&\mbox{ if }\psi > 1\\
    \end{cases}.
\end{align*}
Because $\eta \le \min(1,1/\phi,\gamma)$, we must have the expansion:
\begin{equation}
    \begin{split}
    \eta &\simeq O(\lambda^2) + \begin{cases}
    \gamma-\lambda/((1-\psi)\theta),&\mbox{ if }\psi < 1\\
    1/\phi+\lambda/((\psi-1)\theta),&\mbox{ if }\psi > 1
    \end{cases},\\
    &= \eta_0 - \frac{1}{(1-\psi)\theta_0}\lambda + O(\lambda^2),
    \end{split}
\end{equation}
provided $\theta_0>0$, i.e $\eta_0 \ne 1$.
in this regime, we obtain:
\begin{align*}
     \tau' = \frac{\lambda}{1-\eta/\gamma} &\simeq \begin{cases}
        \lambda/(1-1+\lambda/((1-\psi)\gamma\theta_0))=(1-\psi)\gamma \theta_0,&\mbox{ if }\psi \le 1\\
        \lambda/(1-1/\psi+o(1)) \to 0,&\mbox{ if }\psi > 1
    \end{cases}, \\
     e' = \frac{\lambda}{1-\phi\eta} &\simeq \begin{cases}
        \lambda/(1-\psi+o(1)) \to 0,&\mbox{ if }\psi \le 1\\
        \lambda/(1-1+\lambda\phi/((\psi-1)\theta_0) \to (\psi-1)\theta_0/\phi,&\mbox{ if }\psi > 1
    \end{cases},\\
    \tau = 1-\eta/\gamma &\simeq 1-\eta_0/\gamma=(1-1/\psi)_+,\\
    e = 1-\phi\eta &\simeq 1-\phi\eta_0=(1-\psi)_+.
\end{align*}

On the other hand, if $\theta_0=0$ (which only happens if $\psi < 1$ and $\gamma > 1$, or $\psi \ge 1$ and $\phi \le 1$), it is easy to see from Equation \ref{eq:toto} that we must have $\tau' \to 0$, $e' \to 0$, $\tau \to 1-1/\gamma$, $e \to 1-\phi \ge 0$. 
\end{proof}

\clearpage

\section{Corollary \ref{corr:unregularized-edd}}
\label{sec:unregularized-edd-proof}

As a special case of Theorem \ref{thm:odd-theory-rand-proj}, we recover Corollary \ref{corr:unregularized-edd}, which aligns with Proposition 4 from \citep{bach2024high}. Theorem \ref{thm:odd-theory-rand-proj} is a non-trivial generalization of Proposition 4.

Corollary \ref{corr:unregularized-edd} captures how the covariance matrix affects the test risk of a model through the normalized second and first-order degrees of freedom of $\Sigma_s$. Corollary \ref{corr:unregularized-edd} also reveals that in the underparameterized regime ($\psi_s < 1$), the bias and variance of the test risk of the model strictly increase as a function of $\psi_s$ (rate of parameters to samples); the test risk of the model explodes (i.e., there is catastrophic overfitting \citep{bach2024high}) when $\psi_s$ gets close to 1. In the overparameterized regime ($\psi_s > 1$), the bias and variance of the test risk decrease as $\psi_s$ increases.

\begin{corollary}
Under Assumptions \ref{ass:commute} and \ref{ass:scaling-random-proj}, it holds in the unregularized setting $\lambda_s \to 0^+$ that 
\begin{align}
B_s (\widehat f_s) &= \begin{cases}
    \frac{\theta_0 \ntrace \Theta_s \Sigma_s (\Sigma_s + \theta_0 I_d)^{-1}}{1 - \psi_s}, &\gamma,\psi_s < 1\\
    0, &\psi_s < 1, \gamma \geq 1 \text{ or } 1 \leq \psi_s \leq \gamma\\
    \frac{\theta_0^2 \ntrace \Theta_s \Sigma_s (\Sigma_s + \theta_0 I_d)^{-2}}{1 - \phi_s I_{2, 2} (\theta_0)} + \frac{\theta_0 \ntrace \Theta_s \Sigma_s (\Sigma_s + \theta_0 I_d)^{-1}}{\psi_s - 1}, &\psi_s \geq 1, \psi_s \geq \gamma
\end{cases}, \\
V_s (\widehat f_s) &= \begin{cases}
    \frac{\sigma_s^2 \psi_s}{1 - \psi_s}, &\gamma,\psi_s < 1\\
    \frac{\sigma_s^2 \phi_s}{1 - \phi_s}, &\psi_s < 1, \gamma \geq 1 \text{ or } 1 \leq \psi_s \leq \gamma\\ 
    \frac{\sigma_s^2 \phi_s I_{2, 2} (\theta_0)}{1 - \phi_s I_{2, 2} (\theta_0)} + \frac{\sigma_s^2}{\psi_s - 1}, &\psi_s \geq 1, \psi_s \geq \gamma
\end{cases},
\end{align}
where $I_{a,b}(t) = \ntrace\Sigma^a(\Sigma + t I_d)^{-b}$ for any positive integers $a,b$; and $\theta_0$ is the unique solution to the following non-linear equation:
\begin{gather}
     I_{1, 1} (\theta_0) = \begin{cases}
        \gamma,&\gamma,\psi_s < 1\\
        1,&\psi_s < 1, \gamma \geq 1 \text{ or } 1 \leq \psi_s \leq \gamma \\
        1/\phi_s,&\psi_s \geq 1, \psi_s \geq \gamma
    \end{cases}.
\end{gather}
\label{corr:unregularized-edd}
\end{corollary}

\begin{proof}
Define $e' = 1/e_s \geq 0$, $\tau' = 1/\tau_s \geq 0$, $\theta = \lambda_s \tau' e'/\gamma$, and $\eta = I_{1, 1} (\theta) \in [0, 1]$. One can then express $e'$ and $\tau'$ as:
\begin{align}
    e' &= 1+\psi\tau_s\ntrace\Sigma (\gamma\tau_s e_s \Sigma + \lambda_s I_d)^{-1} = 1+\phi_s \eta e',\\
    \tau' &= 1+\ntrace e_s\Sigma (\gamma\tau_s e_s \Sigma + \lambda_s I_d)^{-1} = 1+(\eta/\gamma)\tau'.
\end{align}
We deduce that:
\begin{align}
    e' &= \frac{1}{1-\phi_s \eta}\label{eq:eprime},
    \\ \tau' &= \frac{1}{1-\eta/\gamma}\label{eq:tauprime}, \\
    \lambda\tau'e' &= \gamma\theta.
    \label{eq:theta}
\end{align}
We define the following limiting values:
\begin{align}
    \lim_{\lambda_s \to 0^+} \theta \to \theta_0,& \lim_{\lambda_s \to 0^+} \eta \to \eta_0, \\
    \lim_{\lambda_s \to 0^+} e_s \to e_0,&\lim_{\lambda_s \to 0^+} \tau_s \to \tau_0, \\
    \lim_{\lambda_s \to 0^+} u_s \to u_0,& \lim_{\lambda_s \to 0^+} \rho_s \to \rho_0.
\end{align}

There are now two cases to consider.

\subsection{Case 1: \texorpdfstring{$\theta_0=0$}{}}
This implies $\eta_0=1$. Therefore, by simple computation, $e_0 = 1/e'_0 = 1-\phi_s \eta_0 = 1-\phi_s$ and $\tau_0 = 1/\tau'_0 = 1-1/\gamma$. This requires $\phi_s \le 1$ and $\gamma \ge 1$.

\subsection{Case 2: \texorpdfstring{$\theta_0>0$}{}}
Equation \ref{eq:theta} can be re-written as:
\begin{eqnarray}
\label{eq:quad}
\frac{\lambda_s}{(1-\phi_s\eta)(1-\eta/\gamma)} = \gamma\theta,\text{ i.e., }\phi_s\eta^2-(\psi_s + 1)\eta+\gamma-\frac{\lambda_s}{\theta} = 0.
\end{eqnarray}
We solve this quadratic equation for $\eta$, arriving at the solutions:
\begin{eqnarray}
    \eta^\pm = \frac{\psi_s + 1 \pm \sqrt{(\psi_s + 1)^2-4(\psi_s -(\phi_s/\theta)\lambda_s)}}{2\phi_s} = \frac{\psi_s + 1 \pm \sqrt{(\psi_s + 1)^2-4(\psi_s-(\phi_s/\theta)\lambda_s)}}{2\phi_s}.
\end{eqnarray}
Taking the limit of $\eta^\pm$ as $\lambda_s \to 0^+$ gives:
\begin{equation}
\begin{split}
    \eta^+ \to \frac{\psi_s+1 + |\psi_s-1|}{2\phi_s} &=
    \begin{cases}
        \psi_s/\phi_s=\gamma,&\mbox{ if }\psi_s \ge 1,\\
        1/\phi_s,&\mbox{ if }\psi_s < 1,
    \end{cases}\\
    \eta^- \to \frac{\psi_s+1 - |\psi_s-1|}{2\phi_s} &=
    \begin{cases}
        1/\phi_s,&\mbox{ if }\psi_s \ge 1,\\
        \psi_s/\phi_s=\gamma,&\mbox{ if }\psi_s < 1.
    \end{cases}
\end{split}
\end{equation}

Recall that we have the following constraints:
\begin{itemize}
    \item $e' \geq 0, \tau' \geq 0$.
    \item $\eta \in [0, 1]$.
\end{itemize}

We can show that $\eta_0 = 1/\phi_s$ is incompatible with $\psi_s<1$. Indeed, otherwise we would have $\tau'_0=1/(1-\eta_0/\gamma) = 1/(1-1/\psi_s)<0$. Similarly, if $\psi_s > 1$, we would have $e_0 = 1 - \phi_s \gamma = 1 - \psi_s < 0$. Therefore, $\eta_0 = \eta^-$. Furthermore, if $\psi_s, \gamma < 1$, it must be that $\theta_0 > 0$ and $\eta_0 = \gamma$. Instead, if $\psi_s < 1, \gamma \geq 1$, we must have that $\phi_s \leq 1$, and therefore, $\theta_0 = 0$ and $\eta_0 = 1$. Similarly, if $\psi_s \geq 1, \gamma \geq 1$, and $\phi_s \leq 1$ (i.e., $1 \leq \psi_s \leq \gamma$), we must have that $\theta_0 = 0$ and $\eta_0 = 1$. In all other cases where $\psi_s \geq 1$, it must be that $\eta_0 = 1 / \phi_s$ (which additionally requires $\phi_s \geq 1$ or $\psi_s \geq \gamma$). Succinctly:
\begin{eqnarray}
    \eta_0 = \begin{cases}
        \gamma,&\gamma,\psi_s < 1\\
        1,&\psi_s < 1, \gamma \geq 1 \text{ or } 1 \leq \psi_s \leq \gamma \\
        1/\phi_s,&\psi_s \geq 1, \psi_s \geq \gamma
    \end{cases}.
\end{eqnarray}

Plugging this into Equation \ref{eq:eprime} and Equation \ref{eq:tauprime} gives:
\begin{align}
    e_0 &= 1 - \phi_s \eta_0 = 1 - \phi_s I_{1, 1} (\theta_0),
    \label{eq:e-lin}
\end{align}
\begin{align}
    \tau_0 &= 1 - \eta_0 / \gamma = 1 - I_{1, 1} (\theta_0) / \gamma.
    \label{eq:tau-lin}
\end{align}

We will now solve for $u_0$ and $\rho_0 / \tau_0^2$. We can re-write $u_s$ and $\rho_s / \tau_s^2$ as:
\begin{align}
    \rho_s / \tau_s^2 &= \gamma^{-1} (\rho_s / \tau_s^2) I_{2, 2} (\theta) +  \theta^2 (u_s + 1) I_{1, 2} (\theta), \\
    \tau_s^2 u_s &= \tau_s^2 \phi_s (u_s + 1) I_{2, 2} (\theta) + \phi_s \gamma^{-1} \rho_s I_{1, 2} (\theta).
\end{align}

Solving for $u_0$ and $\rho_0 / \tau_0^2$ yields:
\begin{align}
    \label{eq:u-rho-lin-1}
    u_0 &= \frac{\phi \zeta}{\gamma-\phi \zeta-I_{2,2}(\theta_0)},\,
    \rho_0 / \tau_0^2 = \frac{\gamma \theta_0^2 I_{2,2}(\theta_0)}{\gamma-\phi \zeta-I_{1,2}(\theta_0)},\\
    \text{where }\zeta 
    &= I_{2,2}(\theta_0)(\gamma-I_{2,2}(\theta_0)) + \theta_0^2 I_{1,2}(\theta_0)^2.
\end{align}

We can then see for the variance term that:
\begin{align}
V_s (\widehat f_s) &= \sigma_s^2 \phi_s \gamma \ntrace \Sigma_s (\gamma e_s \tau_s^2 \Sigma_s + e_s \rho_s I_d - \lambda_s u_s \tau_s I_d) (\gamma \tau_s e_s)^{-2} (\Sigma_s + \theta I_d)^{-2} \\
&= \sigma_s^2 \phi_s (1 / e_s) \ntrace \Sigma_s^2 (\Sigma_s + \theta I_d)^{-2} + (\sigma_s^2 \phi_s / \gamma) (1 / e_s) (\rho_s / \tau_s^2) \ntrace \Sigma_s (\Sigma_s + \theta I_d)^{-2} \\
&- \sigma_s^2 \phi_s (u_s) (1 / e_s) \theta \ntrace \Sigma_s (\Sigma_s + \theta I_d)^{-2} \\
&= \sigma_s^2 \phi_s I_{2, 2} (\theta) / e_s + \sigma_s^2 \phi_s (\rho_s / \tau_s^2) I_{1, 2} (\theta) / (\gamma e_s) - \sigma_s^2 \phi_s u_s \theta I_{1, 2} (\theta) / e_s \\
&\to \frac{\sigma_s^2 \phi_s I_{2, 2} (\theta_0) - \sigma_s^2 \phi_s u_0 \theta_0 I_{1, 2} (\theta_0)}{1 - \phi_s I_{1, 1} (\theta_0)} + \frac{\sigma_s^2 \phi_s \rho_0 / \tau_0^2}{\gamma (1 - \phi_s I_{1, 1} (\theta_0))} \\
&= -\frac{\sigma_s^2 \phi_s \xi}{\phi_s \xi + I_{2, 2} (\theta_0) - \gamma},
\end{align}
where $\xi = I_{1, 1}^2 (\theta_0)  - 2 I_{1, 1} (\theta_0) I_{2, 2} (\theta_0) + I_{2, 2} (\theta_0) \gamma$ and we have used the fact that $I_{1, 2} (\theta) = (I_{1, 1} (\theta) - I_{2, 2} (\theta)) / \theta$. Plugging in $I_{1, 1} (\theta_0) = \eta_0$, we have that:
\begin{align}
    V_s (\widehat f_s) &\to \begin{cases}
        \frac{\sigma_s^2 \psi_s}{1 - \psi_s}, &\gamma,\psi_s < 1\\
        \frac{\sigma_s^2 \phi_s}{1 - \phi_s}, &\psi_s < 1, \gamma \geq 1 \text{ or } 1 \leq \psi_s \leq \gamma\\
        \frac{\sigma_s^2 \phi_s I_{2, 2} (\theta_0)}{1 - \phi_s I_{2, 2} (\theta_0)} + \frac{\sigma_s^2}{\psi_s - 1}, &\psi_s \geq 1, \psi_s \geq \gamma
    \end{cases},
    \label{eq:u-rho-lin-2}
\end{align}
where we have used that $I_{2, 2} (\theta_0) = I_{2, 2} (0) = 1$ in the second case.

Likewise, for the bias term, we obtain:
\begin{align}
    B_s (\widehat f_s) &= \ntrace \Theta_s \Sigma_s + \ntrace \Theta_s \Sigma_s (\gamma e_s^2 \Sigma_s (\gamma \tau_s^2 \Sigma_s + \rho_s I_d) + \lambda_s^2 u_s I_d) (\gamma \tau_s e_s \Sigma_s + \lambda_s I_d)^{-2} \\
    &- 2 \gamma e_s \tau_s \ntrace \Theta_s \Sigma_s^2 (\gamma \tau_s e_s \Sigma_s + \lambda_s I_d)^{-1} \\
    &\to \ntrace \Theta_s \Sigma_s (\Sigma_s^2 + 2 \theta_0 \Sigma_s + \theta_0^2 I_d) (\Sigma_s + \theta_0 I_d)^{-2} \\
    &+ \ntrace \Theta_s \Sigma_s (\Sigma_s^2) (\Sigma_s + \theta_0 I_d)^{-2} \\
    &+ \ntrace \Theta_s \Sigma_s ((\rho_0 / \tau_0^2) \Sigma_s / \gamma) (\Sigma_s + \theta_0 I_d)^{-2} \\
    &+ \ntrace \Theta_s \Sigma_s (\theta_0^2 u_0 I_d) (\Sigma_s + \theta_0 I_d)^{-2} \\
    &+ \ntrace \Theta_s \Sigma_s (-2 \Sigma_s^2 - 2 \theta_0 \Sigma_s) (\Sigma_s + \theta_0 I_d)^{-2} \\
    &= \theta_0^2 (u_0 + 1) \ntrace \Theta_s \Sigma_s (\Sigma_s + \theta_0 I_d)^{-2} + (1 / \gamma) (\rho_0 / \tau_0^2) \ntrace \Theta_s \Sigma_s^2 (\Sigma_s + \theta_0 I_d)^{-2}.
\end{align}

Again, plugging in $I_{1, 1} (\theta_0) = \eta_0$, we have that:
\begin{align}
    B_s (\widehat f_s) &\to \begin{cases}
        \frac{\theta_0 \ntrace \Theta_s \Sigma_s (\Sigma_s + \theta_0 I_d)^{-1}}{1 - \psi_s}, &\gamma,\psi_s < 1\\
        \frac{\theta_0^2 \ntrace \Theta_s \Sigma_s (\Sigma_s + \theta_0 I_d)^{-2}}{1 - \phi_s} = 0, &\psi_s < 1, \gamma \geq 1 \text{ or } 1 \leq \psi_s \leq \gamma\\
        \frac{\theta_0^2 \ntrace \Theta_s \Sigma_s (\Sigma_s + \theta_0 I_d)^{-2}}{1 - \phi_s I_{2, 2} (\theta_0)} + \frac{\theta_0 \ntrace \Theta_s \Sigma_s (\Sigma_s + \theta_0 I_d)^{-1}}{\psi_s - 1}, &\psi_s \geq 1, \psi_s \geq \gamma
    \end{cases},
\end{align}
where we have used that $\ntrace \Theta_s \Sigma_s^2 (\Sigma_s + \theta_0 I_d)^{-2} = \ntrace \Theta_s \Sigma_s (\Sigma_s + \theta_0 I_d)^{-1} - \theta_0 \ntrace \Theta_s \Sigma_s (\Sigma_s + \theta_0 I_d)^{-2}$ and in the second case, $\theta_0 = 0$ and $I_{2, 2} (\theta_0) = 1$.
\end{proof}

\clearpage

\section{Experimental Details}

\subsection{Synthetic Experiments}
\label{sec:common-exp-details}

Across all experiments on synthetic data, we choose $n = 400$. We further use 5 runs to estimate test risks (e.g., $\mathbb E R_s (\widehat f), \mathbb E R_s (\widehat f_s)$), and 5 runs to capture the variance of the estimators, for a total of 25 runs. We use 10,000 samples to estimate test risks.

Our experiments validate that bias amplification occurs even in low-dimensional regimes. In Sections \ref{sec:bias-amp-exp} and \ref{sec:min-group-exp}, and Appendices \ref{sec:bias-amp-plots}, \ref{sec:bias-amp-training}, and \ref{sec:minority-plots}, we show that our theory predicts bias amplification for models trained on only $n = 400$ samples. The high-dimensional regime is commonly studied in ML theory and statistical physics (as we mention in Section \ref{sec:rw-main}), as it makes precise analysis more tractable.

\paragraph{Setup for Section \ref{sec:bias-amp-iso-setup}.} We further choose $\lambda = 1 \times 10^{-6}$ to approximate the minimum-norm interpolator; we henceforth set $\lambda = \lambda_1 = \lambda_2$ for simplicity. We modulate $a_1, a_2, \sigma_1^2, \sigma_2^2$, as well as $\psi$ (rate of parameters to samples) and $\phi$ (rate of features to samples) to understand the effects of model size, number of features, and sample size on bias amplification. We consider diverse and dense values of these variables to obtain a clear picture of when and how models amplify bias.

\paragraph{Setup for Section \ref{sec:spurious-diatomic-setup}.} The first $\pi d$ features represent common {\em core} features of groups 1 and 2 while the latter $(1 - \pi) d$ features capture unshared {\em extraneous} features for group 2 (e.g., spurious features). Intuitively, this setting can model: (1) learning from data from two groups where one group suffers from spurious features \citep{sagawa2020investigation}, or (2) learning from a mixture of raw data (i.e., with spurious features) and clean data (i.e., without spurious features) for a single population \citep{Khani2021Spurious}. We ask: Does our theory predict how the inclusion of different amounts of extraneous features affect the test risk of a minority group (compared to the majority group) when a single model is trained on data from both groups vs. a separate model is trained per group?

Although \citet{sagawa2020investigation} consider classification instead of regression, to mirror their experimental setting, we pick $p_1 = 0.9$ (i.e., group 1 is much larger than group 2) and $\Theta = I_d, \Delta = 0$ (i.e., $w_1^* = w_2^*$). We additionally choose $\lambda = 1 \times 10^{-6}$ and $\sigma_1^2 = \sigma_2^2 = 1$. We modulate $a_1, b_2$, as well as $\psi$ (rate of parameters to samples) and $\phi$ (rate of features to samples). Notably, this setting also captures learning problems with $o(d)$ overlapping core and extraneous features in our asymptotic scaling limit. An extremization of this setting is choosing $\Sigma_1 = a_1 I_{\pi d} \oplus 0 I_{(1-\pi)d}, \Sigma_2 = 0 I_{\pi d} \oplus b_2 I_{(1-\pi)d}$, where groups 1 and 2 have no overlapping features.

The experiments in Section \ref{sec:min-group-exp} validate that our analysis does not rely on conditional dependence heterogeneity. That is, we empirically verify that our theory still holds and predicts bias amplification occurs even when $w^*_1 = w^*_2$ (see Figure \ref{fig:spurious-plots}). In essence, the structure and eigenspectra of the covariance matrices of the two groups still contribute to bias amplification even when the ground-truth weights for the groups are the same. In our theory, we only allow the possibility of $w^*_1 \neq w^*_2$ to be as general as possible. In practice, labeling rule heterogeneity may be leveraged, for example, to train a mixture of experts that is regularized to deamplify bias.

\paragraph{Extraneous vs. Spurious Features.} \RevisionFour{Our usage of extraneous features (i.e., features that are different across groups and correlated with labels) differs from classical definitions of spuriousness (i.e., non-causal correlations between features and labels) \citep{bell2024multipledimensionsspuriousnessmachine}; indeed, the extraneous features are used to generate the labels of the minority group. For example, \citet{Khani2021Spurious} model both the labels and spurious features in linear regression as being separately generated by the core features, such that the labels and spurious features are associated but not causally related. However, this setup is not encompassed by our modeling assumptions, as it entails that the ground-truth parameter and feature covariance matrices are not jointly diagonalizable. In contrast, \citet{sagawa2020investigation} study spurious correlations in classification. At a high level, \citet{sagawa2020investigation} create four subgroups of a population with different combinations of class labels $y \in \{-1, 1\}$ and group labels $a \in \{-1, 1\}$. The core and spurious features are then sampled from normal distributions parameterized by $y, a$ (respectively) and different variance levels. By setting the spurious features to have a significantly lower variance than the core features and making $y$ and $a$ highly associated (i.e., imbalanced groups), the authors coerce models to perform classification as a function of primarily the spurious features of the majority group, which does not generalize to the minority group. To capture this spirit, our setup uses imbalanced groups, and the data for the majority group provides no learning signal for the extraneous features; this coerces models to perform regression as a function of primarily the core features, without learning appropriate parameters for the extraneous features, and thus generalize poorly to the minority group.}

\subsection{Colored MNIST Experiments}
\label{sec:colored-mnist-details}

\textbf{Train-test split.} Colored MNIST has a total of 60k instances. Each image is $28 \times 28 \times 3$ pixels. We use the prescribed 0.67-0.33 train-test split. We do not perform validation of hyperparameters, which we mostly adopt\footnote{\url{https://colab.research.google.com/github/reiinakano/invariant-risk-minimization/blob/master/invariant_risk_minimization_colored_mnist.ipynb}}.

\textbf{Model architecture.} By default, our CNN architecture consists of: (1) a convolutional layer (3 in-channels, 20 out-channels, kernel size of 5, stride of 1); (2) a max pooling layer (kernel size of 2, stride of 2); (3) a second convolutional layer (20 in-channels, 50 out-channels, kernel size of 5, stride of 1); (4) a second max-pooling layer (kernel size of 2, stride of 2); (5) a fully-connected layer ($\mathbb R^{800} \to \mathbb R^{500}$); and (6) a second fully-connected layer ($\mathbb R^{500} \to \mathbb R^1$).

\textbf{Model training.} We train each model with a batch size of 250 for a single epoch with respect to groups (i.e., 80 training steps given there are two groups). We use a cross-entropy loss and the Adam optimizer with learning rate 0.01. We run all experiments on a single NVIDIA L40S. We report our results over 10 random seeds.

\clearpage

\section{Bias Amplification Plots}
\label{sec:bias-amp-plots}

\begin{figure}[!ht]
    \centering
    \includegraphics[width=\linewidth]{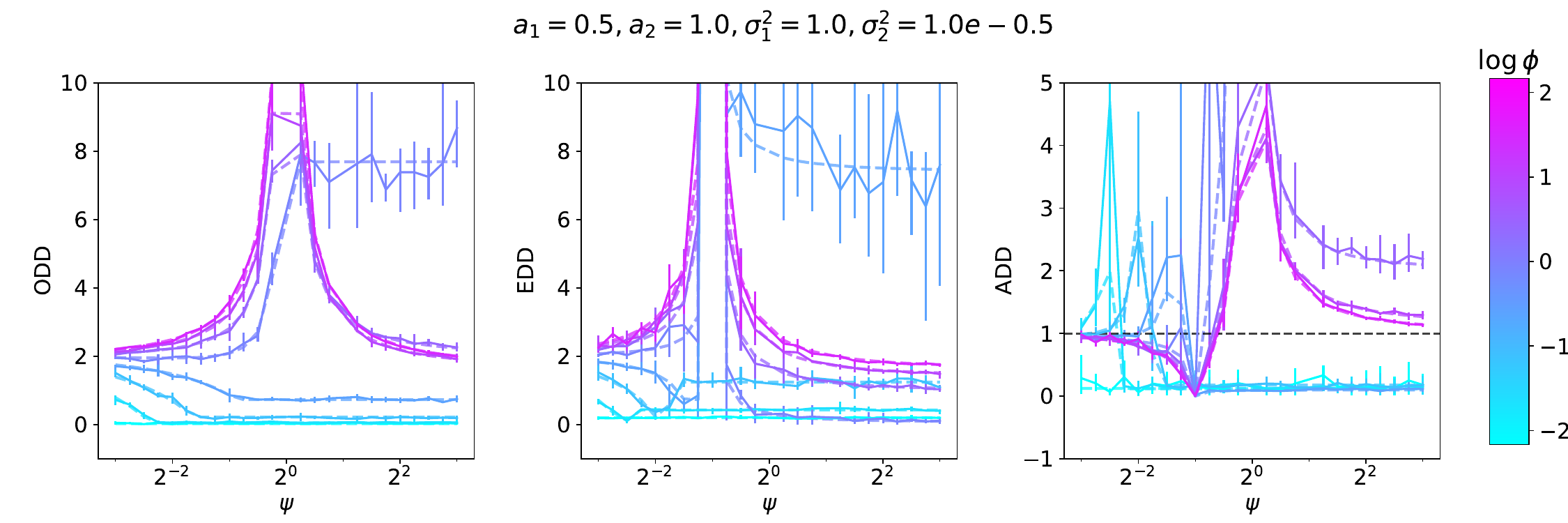}
    \includegraphics[width=\linewidth]{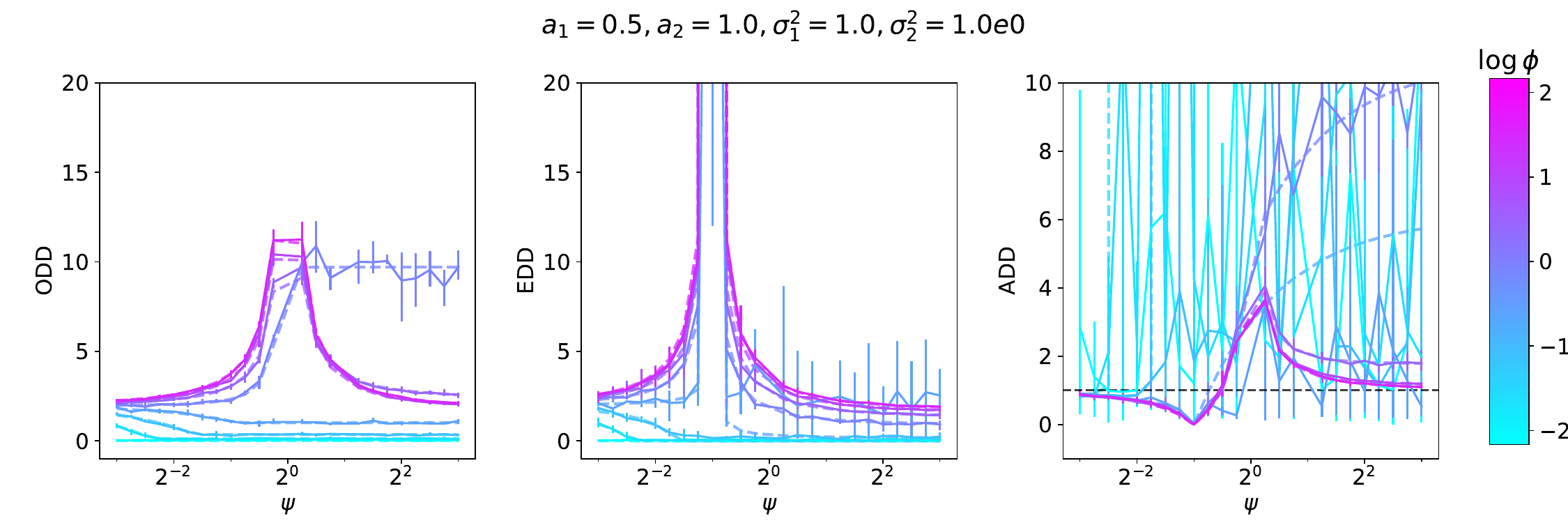}
    \includegraphics[width=\linewidth]{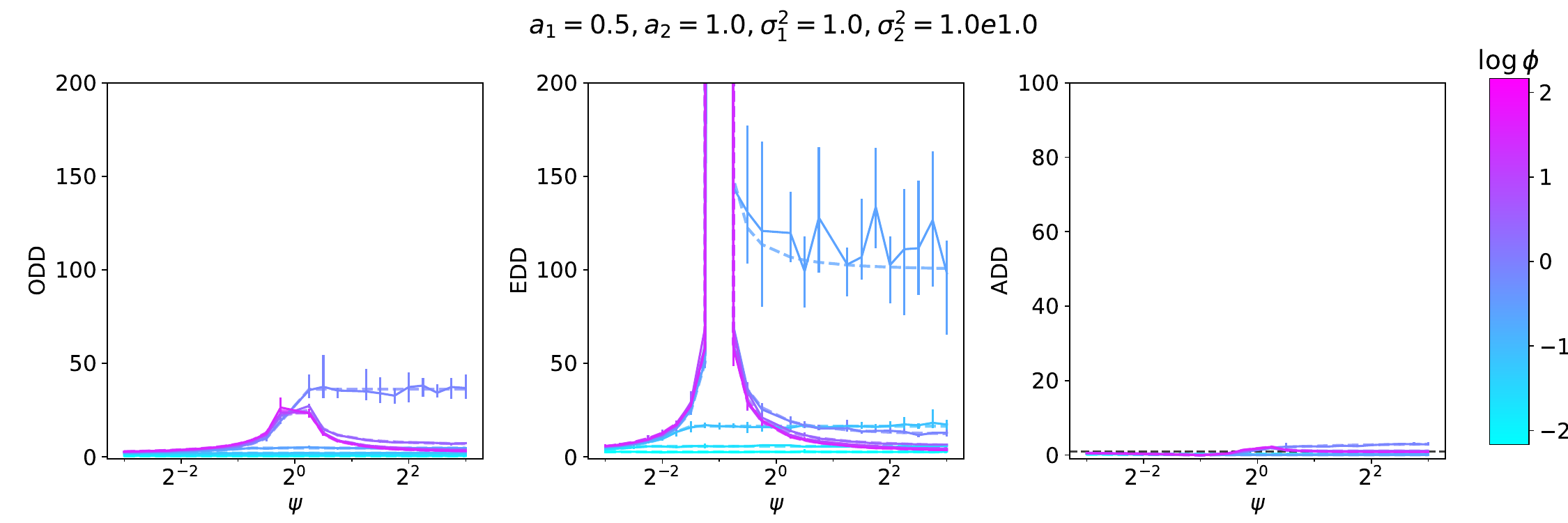}
   \caption{We empirically demonstrate that bias amplification occurs and validate our theory (Theorems \ref{thm:odd-theory-rand-proj} and \ref{thm:edd-theory-rand-proj}) for $ODD$, $EDD$, and $ADD$ under the setup described in Section \ref{sec:bias-amp-iso-setup}. The solid lines capture empirical values while the corresponding lower-opacity dashed lines represent what our theory predicts. We plot $ODD$ and $EDD$ on the same scale for easy comparison, and include a black dashed line at $ADD = 1$ to contrast bias amplification vs. deamplification. The error bars capture the range of the estimators over 25 random seeds.}
\end{figure}

\clearpage

\begin{figure}[!ht]
    \centering
    \includegraphics[width=\linewidth]{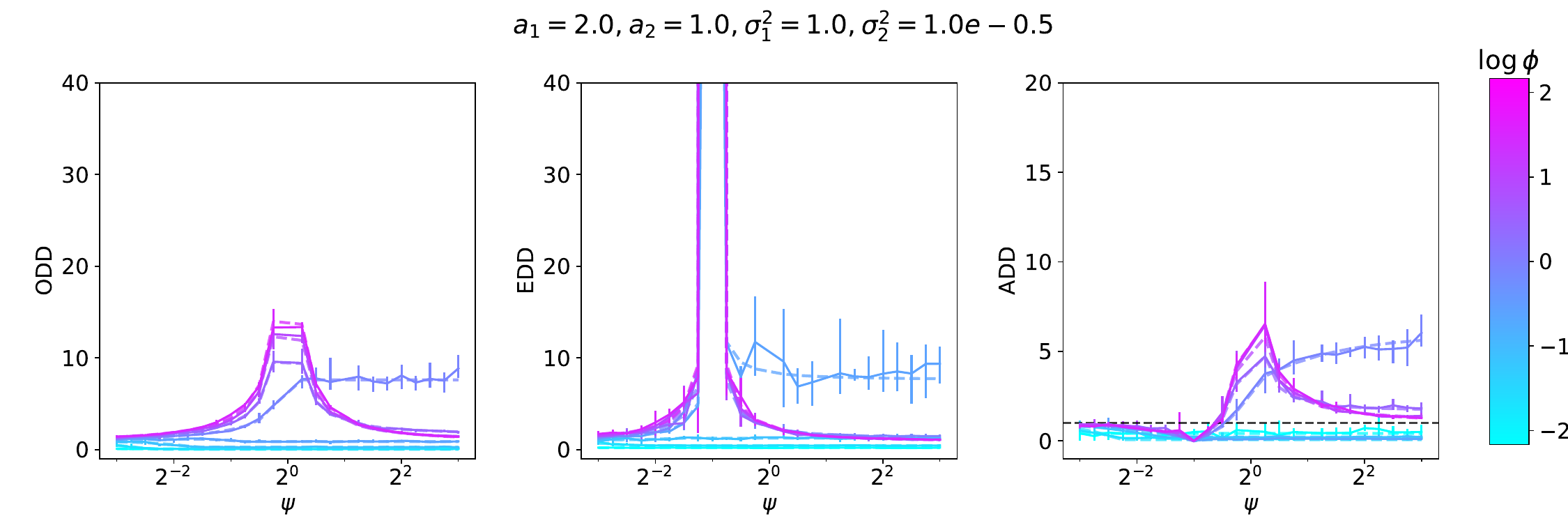}
    \includegraphics[width=\linewidth]{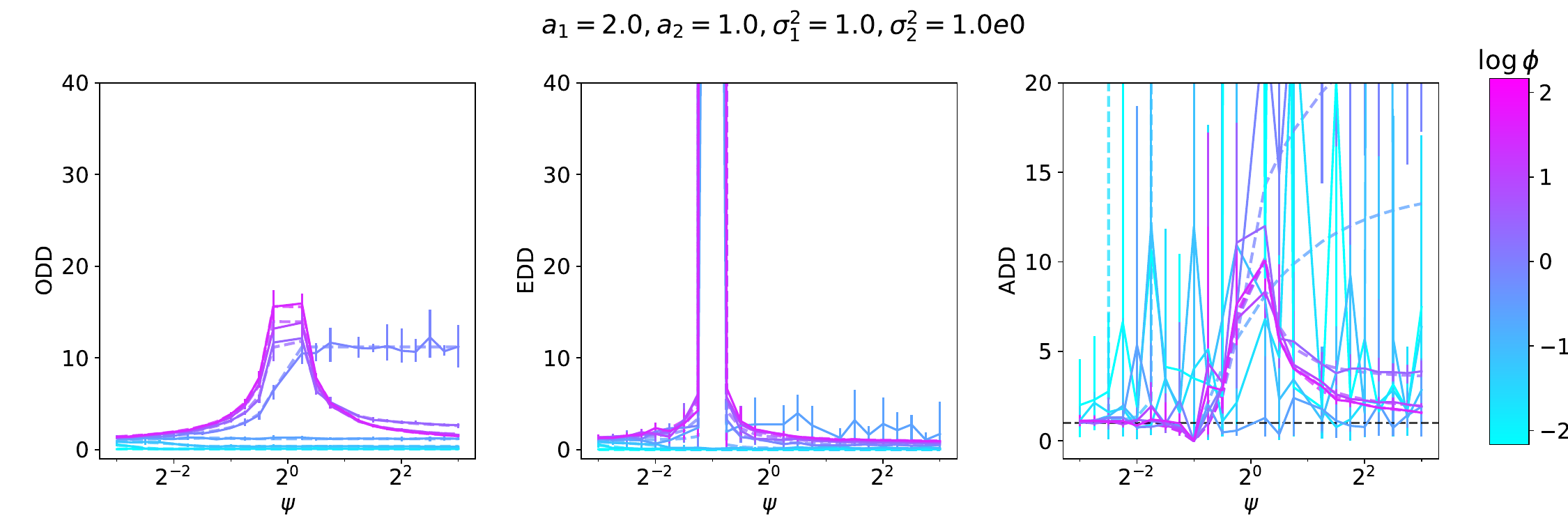}
    \includegraphics[width=\linewidth]{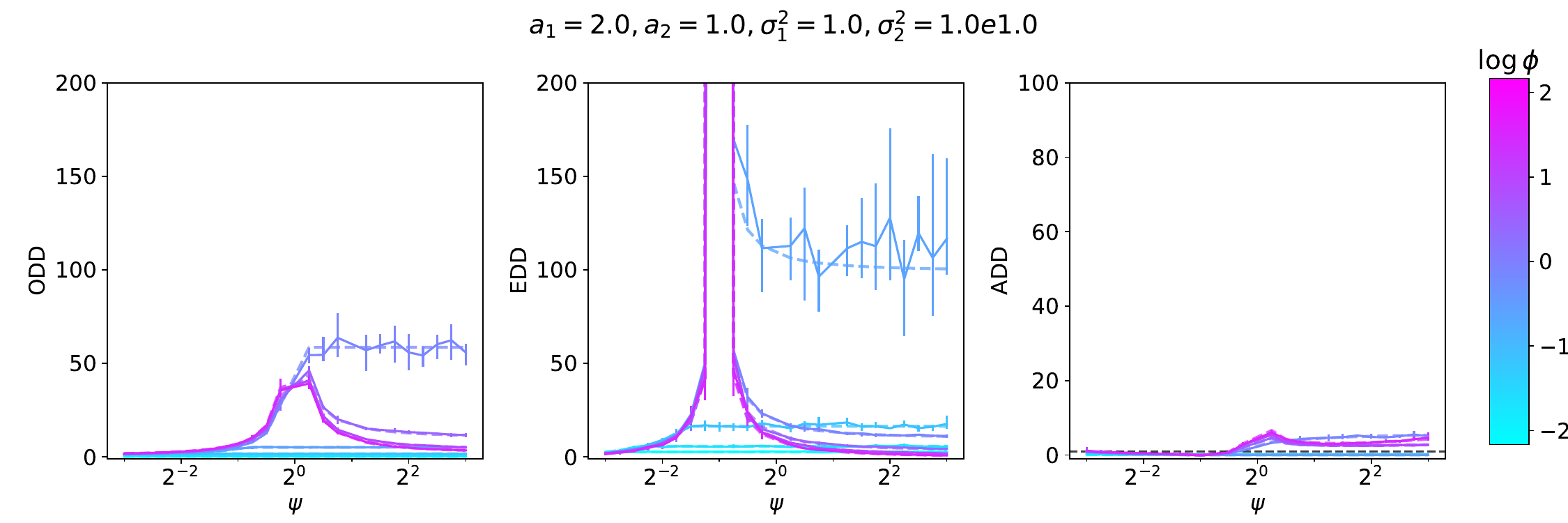}
    \caption{We empirically demonstrate that bias amplification occurs and validate our theory (Theorems \ref{thm:odd-theory-rand-proj} and \ref{thm:edd-theory-rand-proj}) for $ODD$, $EDD$, and $ADD$ under the setup described in Section \ref{sec:bias-amp-iso-setup}. The solid lines capture empirical values while the corresponding lower-opacity dashed lines represent what our theory predicts. We plot $ODD$ and $EDD$ on the same scale for easy comparison, and include a black dashed line at $ADD = 1$ to contrast bias amplification vs. deamplification. The error bars capture the range of the estimators over 25 random seeds.}
\end{figure}

\clearpage

\begin{figure}[!ht]
    \centering
    \includegraphics[width=\linewidth]{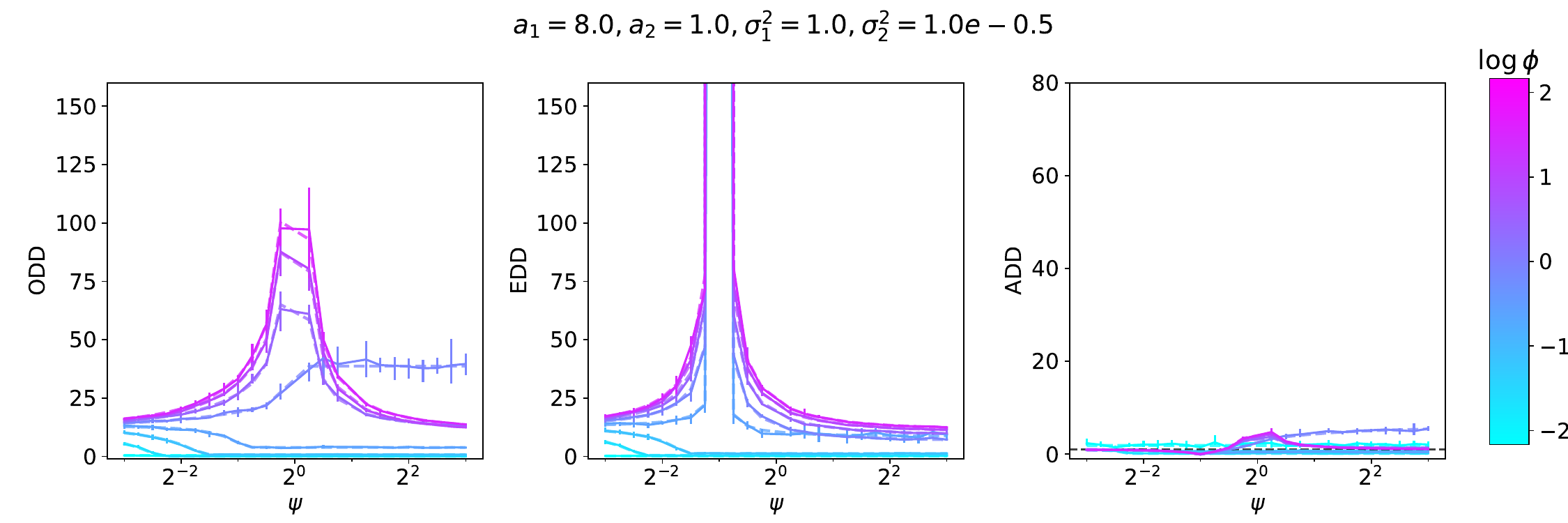}
    \includegraphics[width=\linewidth]{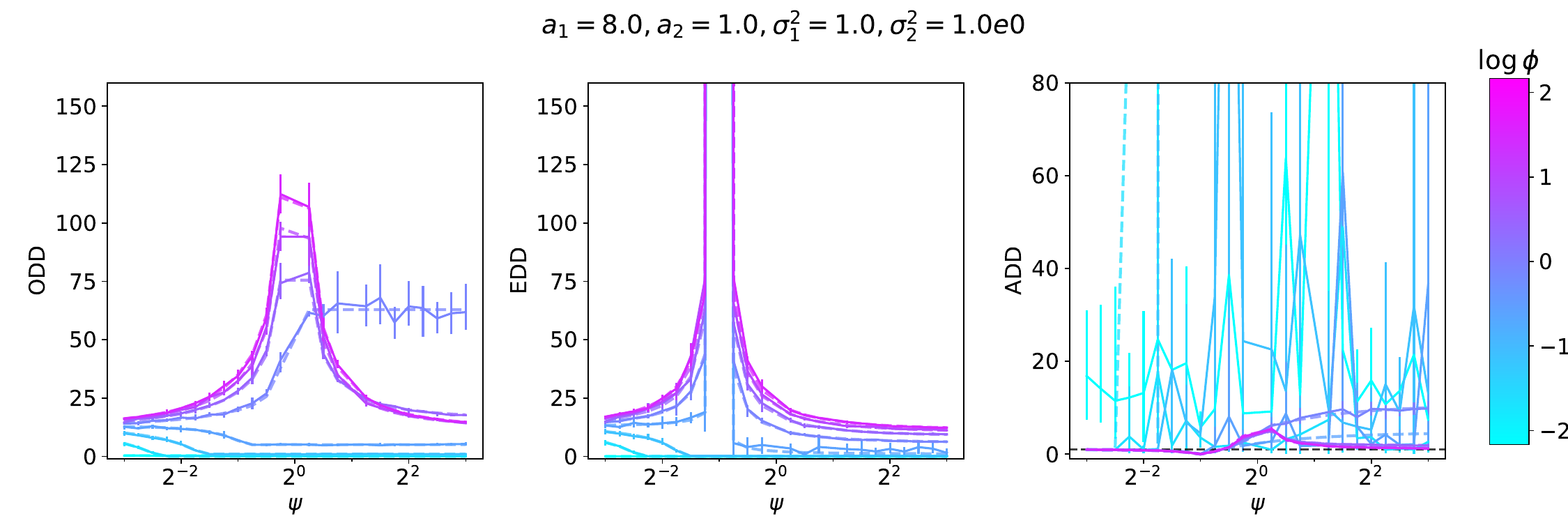}
    \includegraphics[width=\linewidth]{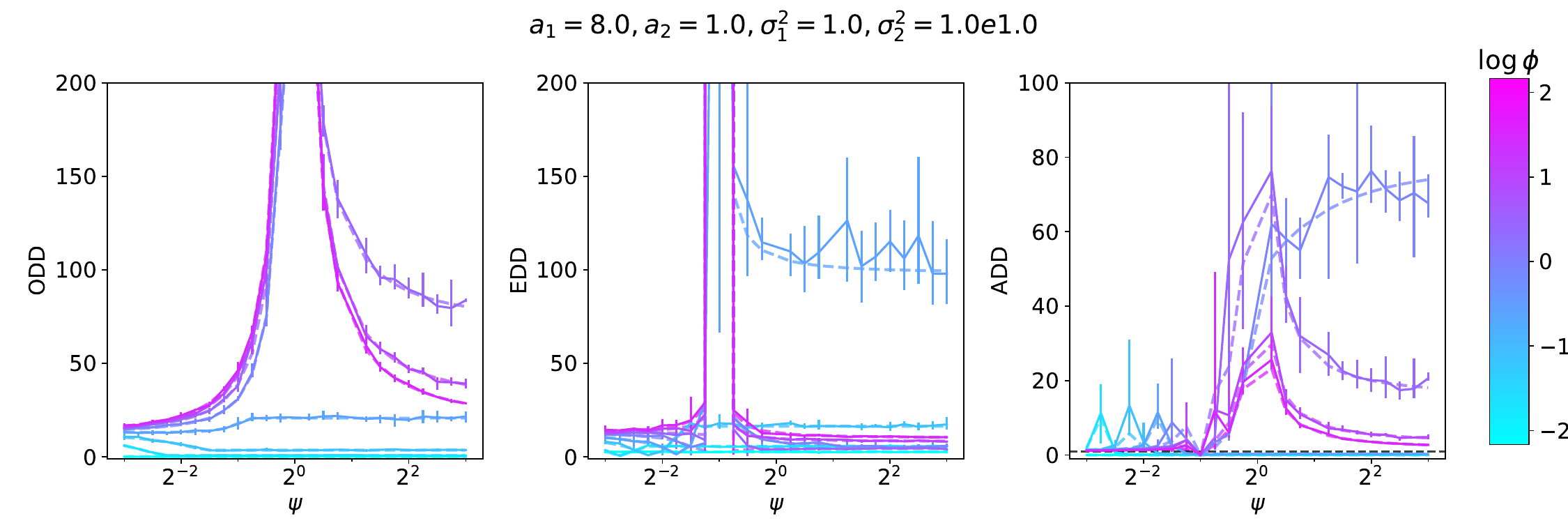}
    \caption{We empirically demonstrate that bias amplification occurs and validate our theory (Theorems \ref{thm:odd-theory-rand-proj} and \ref{thm:edd-theory-rand-proj}) for $ODD$, $EDD$, and $ADD$ under the setup described in Section \ref{sec:bias-amp-iso-setup}. The solid lines capture empirical values while the corresponding lower-opacity dashed lines represent what our theory predicts. We plot $ODD$ and $EDD$ on the same scale for easy comparison, and include a black dashed line at $ADD = 1$ to contrast bias amplification vs. deamplification. The error bars capture the range of the estimators over 25 random seeds.}
\end{figure}

\clearpage

\section{Power-Law Covariance}
\label{sec:bias-amp-pow-setup}

To better understand how $\phi$ (rate of features to samples) and the label noise ratio $c$ affect bias amplification, we derive explicit phase transitions in the bias amplification profile of unregularized ridge regression with random projections in terms of these quantities. We consider the setting of power-law covariance, as it is analytically tractable and can be translated to the case of wide neural networks \citep{Caponnetto2007OptimalRF,Cui_2022, maloney2022solvable}, where the exponents can be empirically gauged. Let the eigenvalues $\lambda_k^{(s)}$ of $\Sigma_s$ have power-law decay, i.e., $\lambda_k^{(s)} = k^{-\beta_s}$, for all $k$ and some positive constants $\beta_1$ and $\beta_2$. WLOG, we will assume $\beta_1 > \beta_2$. Note that $\beta_s$ controls the effective dimension and ultimately the difficulty of fitting the noiseless part of the signal from group $s$. If $\beta_s$ is large, then all the information is concentrated in a few features, and so the learning problem is easier. We similarly assume that the eigenvalues $\mu_k$ of $\Delta$ have power-law decay $\mu_k = k^{-\alpha}$, for all $k$ and constant $\alpha > 0$. Finally, we consider balanced groups (i.e., $p_1 = p_2 = 1/2$). Under this setup, we have the following corollary.

\begin{corollary}
Suppose that in the single model setting, as $\lambda \to 0^+$, $(e_1,e_2,\tau,u_1,u_2,\rho)$ is the unique positive solution to the following fixed-point equations:
\begin{gather}
1/\tau = 1 + 1 / (\gamma \tau),\quad 1/e_s = 1 + \phi \ntrace \Sigma_s L^{-1},\text{ for }s\in\{1,2\}, \\
\rho = 0,\quad u_s = \phi e_s^2 \ntrace \Sigma_s D L^{-2},\text{ for }s\in\{1,2\},\\
\text{ where: }
L = p_1 e_1 \Sigma_1 + p_2 e_2 \Sigma_2,\, D = p_1 u_1 \Sigma_1 + p_2 u_2 \Sigma_2 + B.
\end{gather}

Furthermore, suppose $\psi_s < 1, \gamma \geq 1 \text{ or } 1 \leq \psi_s \leq \gamma.$ Under the assumptions of Theorem \ref{thm:odd-theory-rand-proj} and Assumption \ref{ass:spectral}, as $\lambda \to 0^+$, we have the following approximate analytical phase transitions in the bias amplification profile of ridge regression with random projections:
\begin{align}
    \lim_{\substack{d,n_1,n_2 \to \infty\\\phi_{1,2}\to 2\phi}}ADD \to \frac{c}{|c-1|}, &\quad
    \lim_{c \to 0^+} \lim_{\substack{d,n_1,n_2 \to \infty\\\phi_{1,2}\to 2\phi}} ADD \to 0,\\
    \lim_{c \to \infty} \lim_{\substack{d,n_1,n_2 \to \infty\\\phi_{1,2}\to 2\phi}} ADD \to 1,&\quad
    \lim_{c \to 1}\lim_{\substack{d,n_1,n_2 \to \infty\\\phi_{1,2}\to 2\phi}} ADD \to \infty.
\end{align}
\label{corr:bias-amp-phase-transitions}
\end{corollary}
We relegate the proof to Appendix \ref{sec:power-law-proof} and empirically assess the validity of this result in Figure \ref{fig:power-law-validation}. The phase transitions reveal that bias amplification peaks near $c = 1$, bias deamplification peaks when $c \to 0^+$, and bias is roughly neither amplified or deamplified when $c \to \infty$. Furthermore, the right tail of the $ODD$ profile (which Corollary \ref{corr:bias-amp-phase-transitions} predicts to be proportional to $c$) is higher than the left tail (i.e., 0) for larger $c$. However, the left tail of the $EDD$ profile (which Corollary \ref{corr:bias-amp-phase-transitions} predicts to be proportional to $|c - 1|$) does not increase steeply as $c \to 0^+$. Interestingly, in the proof of Corollary \ref{corr:bias-amp-phase-transitions}, we observe that the bias term depends on $\ntrace \Delta \Sigma_s$; therefore, the setting $\forall k, \lambda_k^{(s)} \geq 1 / \mu_k$ (e.g., common in learning from synthetic data \citep{dohmatob2024Demystified}) can prevent the bias term from vanishing or even cause it to explode. This may explain why training models on synthetic data (i.e., data previously generated by the model) may amplify unfairness \citep{wyllie2024fairness}.

\begin{figure}[!ht]
    \includegraphics[width=0.9\linewidth]{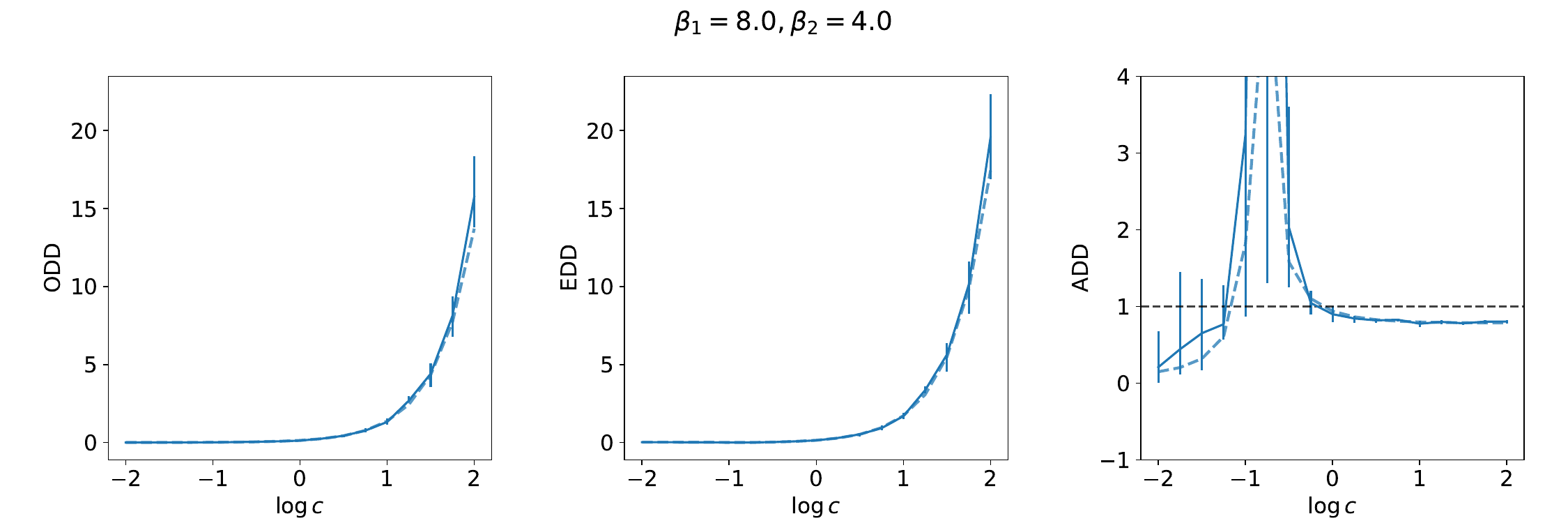}
    \caption{\textbf{Our theory predicts that bias amplification is larger for higher noise ratios than lower noise ratios.} We observe that Corollary \ref{corr:bias-amp-phase-transitions} generally predicts the $ADD$ profile with respect to the noise ratio $c$. The solid lines capture empirical values while the corresponding lower-opacity dashed lines represent what Theorem \ref{thm:odd-theory-rand-proj} predicts. We plot $ODD$ and $EDD$ on the same scale for easy comparison, and include a black dashed line at $ADD = 1$ to contrast bias amplification vs. deamplification. The error bars capture the range of the estimators over 25 random seeds. We consider the setup described in Appendix \ref{sec:bias-amp-pow-setup} with $\psi = 0.5$, $\phi = 0.2$, and $\lambda = 1 \times 10^{-6}$.}
    \label{fig:power-law-validation}
\end{figure}

\clearpage

\section{Proof of Corollary \ref{corr:bias-amp-phase-transitions}}
\label{sec:power-law-proof}

\begin{proof}
We begin by computing the $ODD$ in the limit $\lambda \to 0^+$. We define $u^{(s)}_j = u_j$ for $B = \Sigma_s$. By Assumption \ref{ass:spectral}, we can re-express the constants in Definition \ref{def:cosmo-constants} in terms of the limiting spectral densities of the covariance matrices:
\begin{gather}
e_1 = \frac{1}{1+\phi \int_0^\infty \frac{1}{p_1 e_1 + p_2 e_2 r}\mathrm d\nu(r)},
e_2 = \frac{1}{1+\phi \int_0^\infty \frac{r}{p_1 e_1 + p_2 e_2 r}\mathrm d\nu(r)},\\
\tau = \frac{1}{1+\frac{1}{\gamma \tau}} = 1 - 1 / \gamma, \rho = 0, \\
u^{(1)}_1 = \phi e_1^2 \int_0^\infty \frac{u^{(1)}_1 p_1 + u^{(1)}_2 p_2 r + 1}{(p_1 e_1 + p_2 e_2 r)^2}\mathrm d\nu(r), u^{(1)}_2 = \phi e_2^2 \int_0^\infty \frac{u^{(1)}_1 p_1 r + u^{(1)}_2 p_2 r^2 + r}{(p_1 e_1 + p_2 e_2 r)^2}\mathrm d\nu(r), \\
u^{(2)}_1 = \phi e_1^2 \int_0^\infty \frac{u^{(2)}_1 p_1 + u^{(2)}_2 p_2 r + r}{(p_1 e_1 + p_2 e_2 r)^2}\mathrm d\nu(r), u^{(2)}_2 = \phi e_2^2 \int_0^\infty \frac{u^{(2)}_1 p_1 r + u^{(2)}_2 p_2 r^2 + r^2}{(p_1 e_1 + p_2 e_2 r)^2}\mathrm d\nu(r).
\end{gather}

Since $\beta_1 > \beta_2$, $-\beta_2 - (-\beta_1) > 0$. As such, for $d \to \infty$, the ratios $r_k = \lambda_k^{(2)}/\lambda_k^{(1)}$ have the approximate limiting distribution $\nu=\delta_{r=\infty}$, i.e., a Dirac atom at infinity. Thus:
\begin{gather}
e_1 = 1,
e_2 = 1 - \frac{\phi}{p_2} = 1 - \phi_2, \tau = 1 - 1 / \gamma, \rho = 0, \\
u^{(1)}_1 = 0, u^{(1)}_2 = 0, u^{(2)}_1 = 0, u^{(2)}_2 = \frac{\phi}{p_2 (p_2 - \phi)}.
\end{gather}

Now, we can re-express the variance terms as:
\begin{align}
V_1 (\widehat f) &= \phi \sigma_1^2 \int_0^\infty \frac{p_1}{(p_1 + p_2 e_2 r)^2} \mathrm d\nu(r) + \phi \sigma_2^2 \int_0^\infty \frac{p_2 e_2 r}{(p_1 + p_2 e_2 r)^2} \mathrm d\nu(r) = 0, \\
V_2 (\widehat f) &= \phi \sigma_1^2 \int_0^\infty \frac{p_1 r + p_1 p_2 u_2^{(2)} r}{(p_1 + p_2 e_2 r)^2} \mathrm d\nu(r) + \phi \sigma_2^2 \int_0^\infty \frac{p_2 e_2 r^2}{(p_1 e_1 + p_2 e_2 r)^2} \mathrm d\nu(r) = \frac{\sigma_2^2 \phi}{p_2 - \phi}.
\end{align}

Likewise, we can re-express the bias terms as:
\begin{gather}
B_1 (\widehat f) = \int_0^\infty \int_0^\infty \int_0^\infty \frac{a \delta e_2^2 p_2^2 r^2}{(e_1 p_1 + e_2 p_2 r)^2} \mathrm d\mu(r, a) d\pi(\delta) = \int_0^\infty \int_0^\infty a \delta \mathrm \, d\mu(a) d\pi(\delta) = 0, \\
B_2 (\widehat f) = 0.
\end{gather}
In this calculation, we observe that the adversarial setting $\forall k, \lambda_k^{(1)} \geq 1 / \mu_k$ can prevent the bias term from vanishing. Putting these pieces together and recalling that $p_2 = 1 / 2$:
\begin{gather}
    ODD \to \left| V_1(\widehat f) - V_2(\widehat f) \right| = \frac{2 \phi \sigma_1^2}{1 - 2 \phi} c.
\end{gather}

We now compute the $EDD$. We can once again re-express the constants in Definition \ref{def:cosmo-constants-edd} in terms of the limiting spectral densities of the covariance matrices:
\begin{align}
e_s = \frac{1}{1+\phi_s / e_s} = 1 - \phi_s, \tau_s = 1 - 1 / \gamma.
\end{align}
By Corollary \ref{corr:unregularized-edd}, because $\psi_s < 1, \gamma \geq 1 \text{ or } 1 \leq \psi_s \leq \gamma$,  $B_s (\widehat f_s) = 0$ and $V_s (\widehat f_s) = \frac{\sigma_s^2 \phi_s}{1 - \phi_s}$.

Therefore, because $\phi = p_s \phi_s$:
\begin{gather}
    EDD \to \left| V_1(\widehat f_1) - V_2(\widehat f_2) \right| = \frac{2 \phi}{1 - 2 \phi} \left| \sigma_1^2 - \sigma_2^2 \right| = \frac{2 \phi \sigma_1^2}{1 - 2 \phi} \left| c - 1 \right|, \\
    ADD \to \frac{c}{|c - 1|}.
\end{gather}
\end{proof}

\clearpage

\section{Bias Amplification During Training}
\label{sec:bias-amp-training}

\begin{figure}[!ht]
    \centering
    \includegraphics[width=\linewidth]{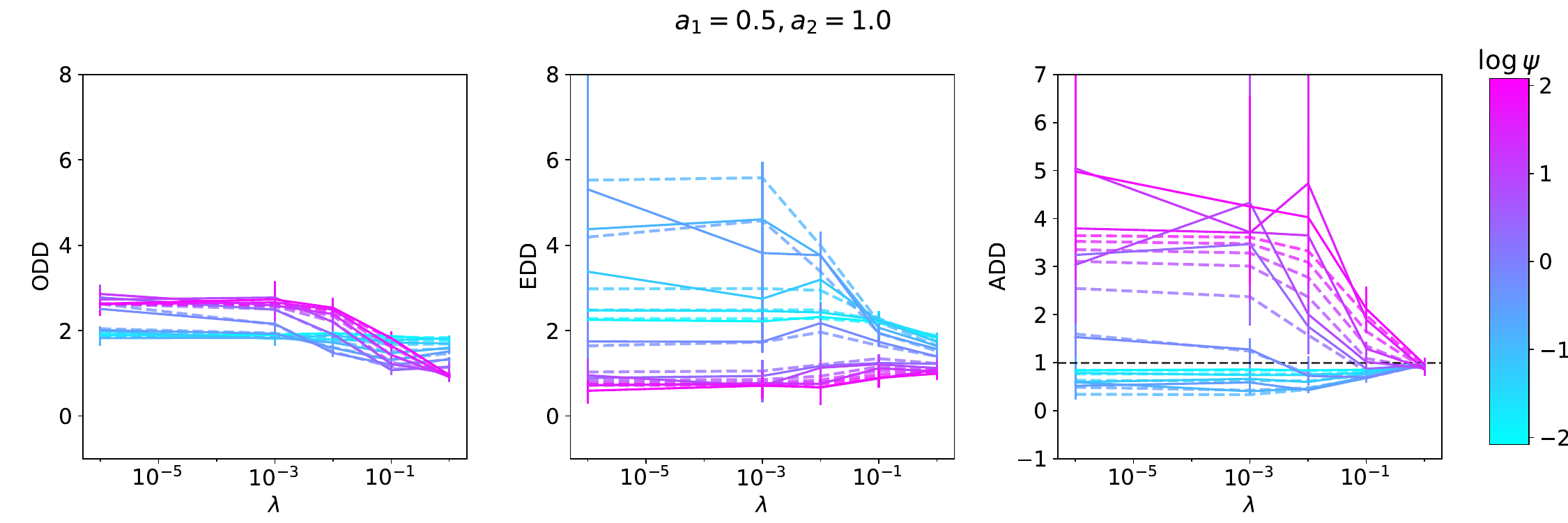}
    \includegraphics[width=\linewidth]{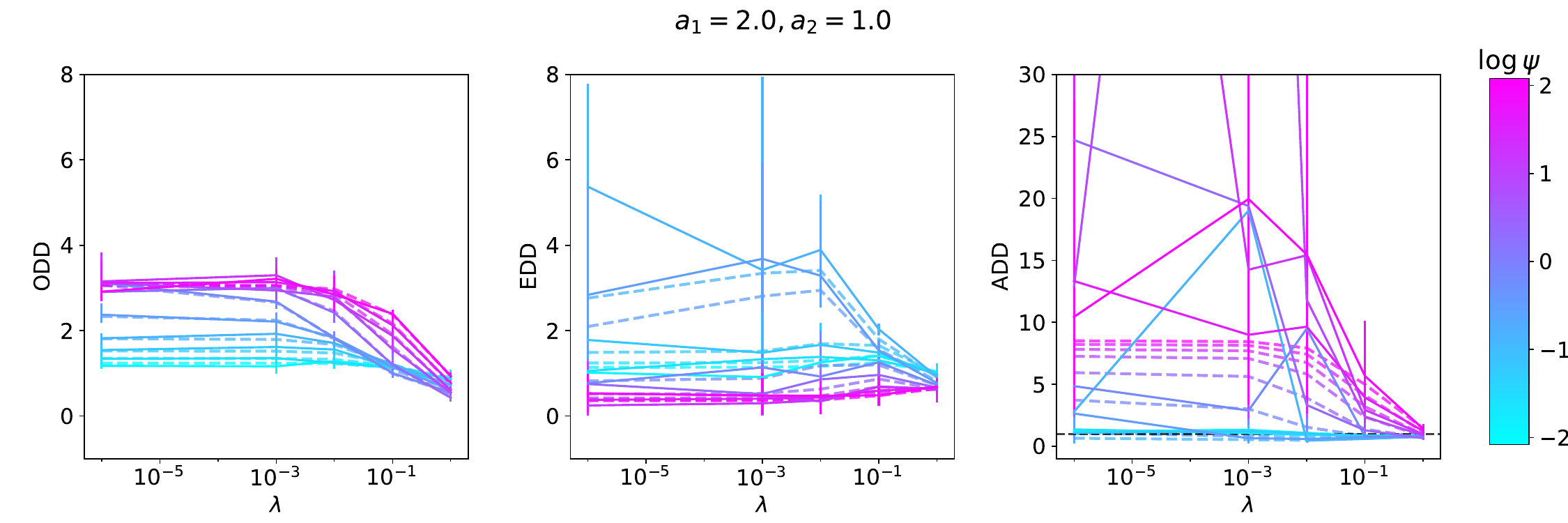}
    \includegraphics[width=\linewidth]{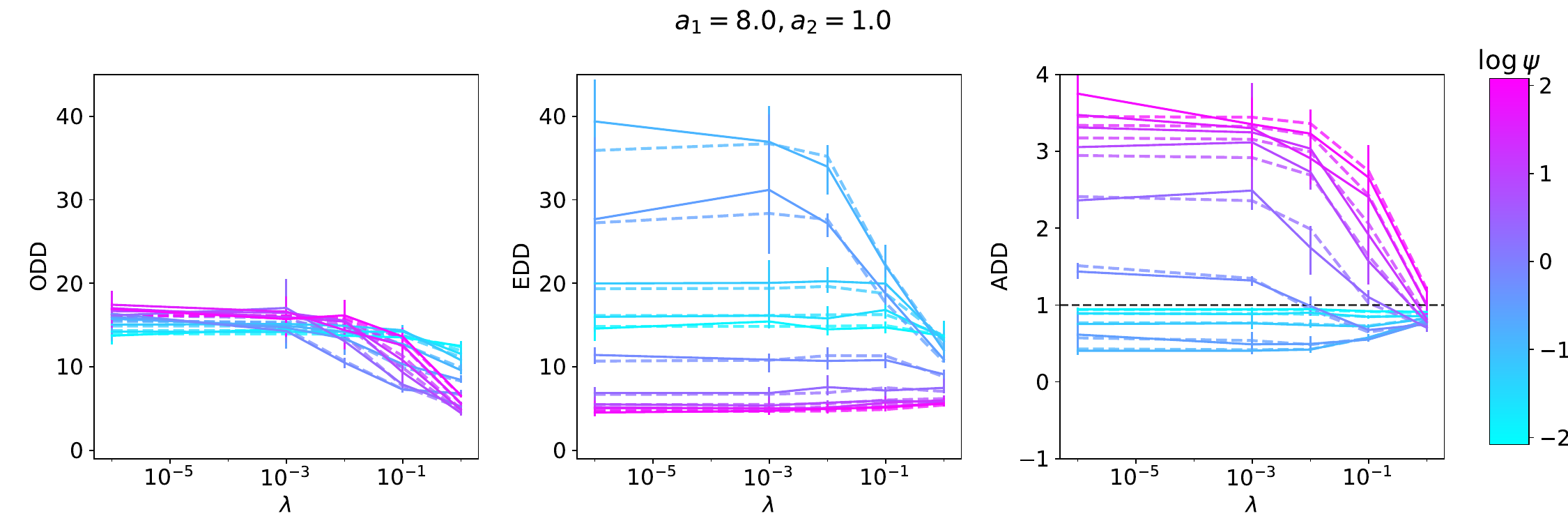}
    \caption{\textbf{Our theory reveals that there may be an optimal regularization penalty to deamplify bias.} We empirically demonstrate that bias amplification can be heavily affected by $\lambda$ and validate our theory (Theorems \ref{thm:odd-theory-rand-proj} and \ref{thm:edd-theory-rand-proj}) for $ODD$, $EDD$, and $ADD$ under the setup described in Section \ref{sec:over-time-setup}. The solid lines capture empirical values while the corresponding lower-opacity dashed lines represent what our theory predicts. We include a black dashed line at $ADD = 1$ to contrast bias amplification vs. deamplification. The error bars capture the range of the estimators over 25 random seeds.}
    \label{fig:bias-amp-training}
\end{figure}

\clearpage

\section{Colored MNIST Plots}
\label{sec:add-cmnist-plots}

We further assess the applicability of our conclusions about the effects of label noise (Figures \ref{fig:colored-mnist}, \ref{fig:colored-mnist-1-errors}) and model size (Figure \ref{fig:colored-mnist-2-errors}) on bias amplification for Colored MNIST. Please see Section \ref{sec:reg-training-dynamics} for a discussion of Figure \ref{fig:colored-mnist}. We observe in Figure \ref{fig:colored-mnist-1-errors} that as we increase the label noise ratio, the $EDD$ generally increases, while the $ODD$ remains relatively low, which is suggested by our theoretical reasoning in Section \ref{sec:reg-training-dynamics}. Furthermore, in Figure \ref{fig:colored-mnist-2-errors}, as the hidden dimension $m$ of the penultimate layer of the CNN increases, the $ODD$ appears to decrease and plateau, which is predicted by our theoretical results (see Section \ref{sec:bias-amp-iso-results}) in the Colored MNIST regime where $\phi < 1$. However, the $EDD$ does not appear to decrease; while this is plausibly predicted by our theory, it requires going beyond our simplistic assumption that $\Sigma_1$ roughly coincides with $\Sigma_2$ and studying the interplay between $\phi_s, \psi_s, \Sigma_s$ for each group $s$ (as suggested by Appendix \ref{sec:bias-amp-plots}).

\begin{figure}[!ht]
    \centering
    \includegraphics[width=0.35\linewidth]{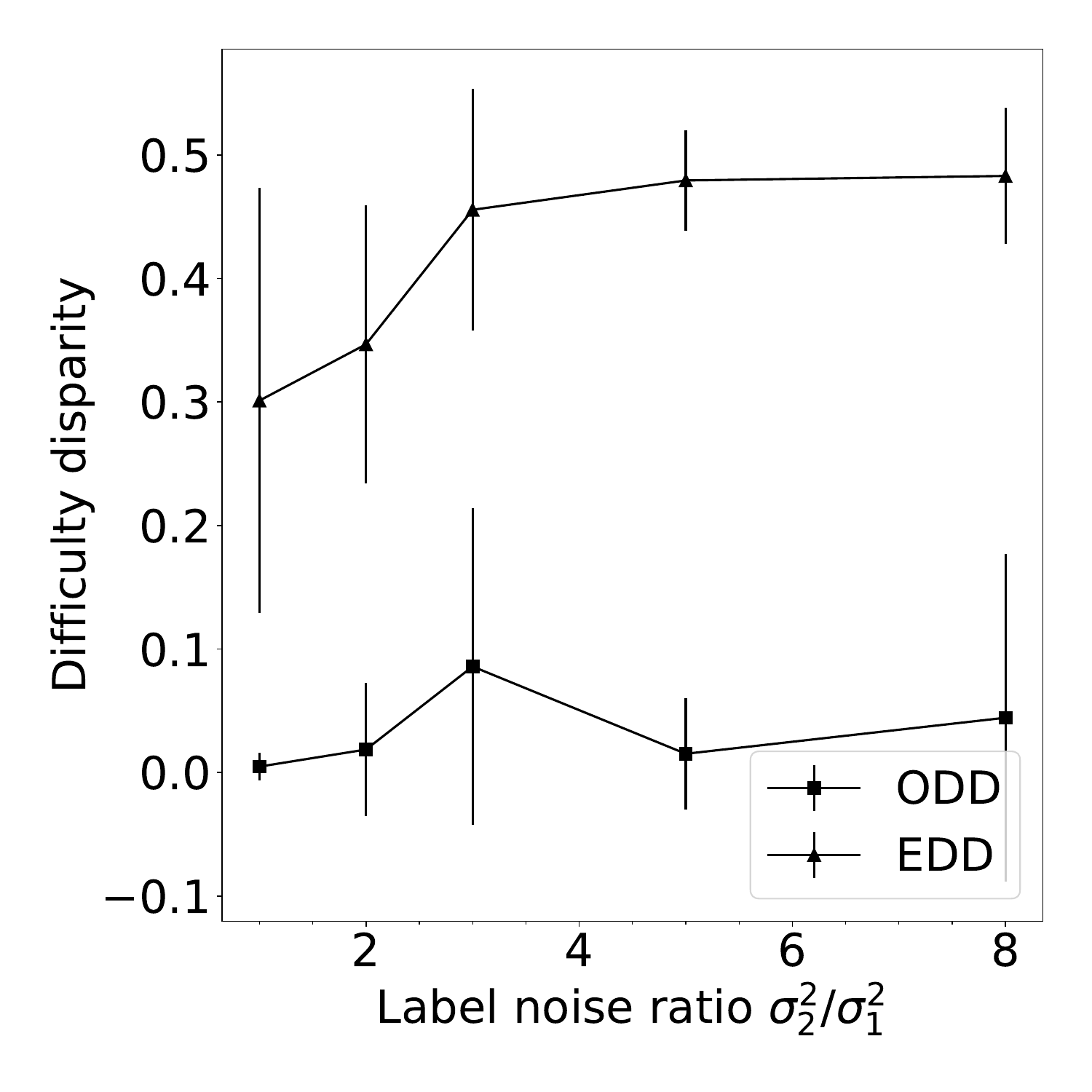}
    \caption{\textbf{Our theory predicts that more disparate label noise between groups deamplifies bias on Colored MNIST.} We plot the $ODD$ and $EDD$ of a CNN for different label noise ratios $c = \sigma_2^2 / \sigma_1^2$ for Colored MNIST. As $c$ increases, the $EDD$ generally increases while the $ODD$ remains relatively low, which is predicted by our theory (see reasoning in Section \ref{sec:over-time-setup}). In our experiments, $\sigma_1^2 = 0.05$ stays fixed while $\sigma_2^2$ varies. For each value of $c$, the model is evaluated after $t = 80$ training steps and has a penultimate layer with dimension $m = 500$. The error bars capture the standard deviation computed over 10 random seeds.}
    \label{fig:colored-mnist-1-errors}
\end{figure}

\begin{figure}[!ht]
    \centering
    \includegraphics[width=0.35\linewidth]{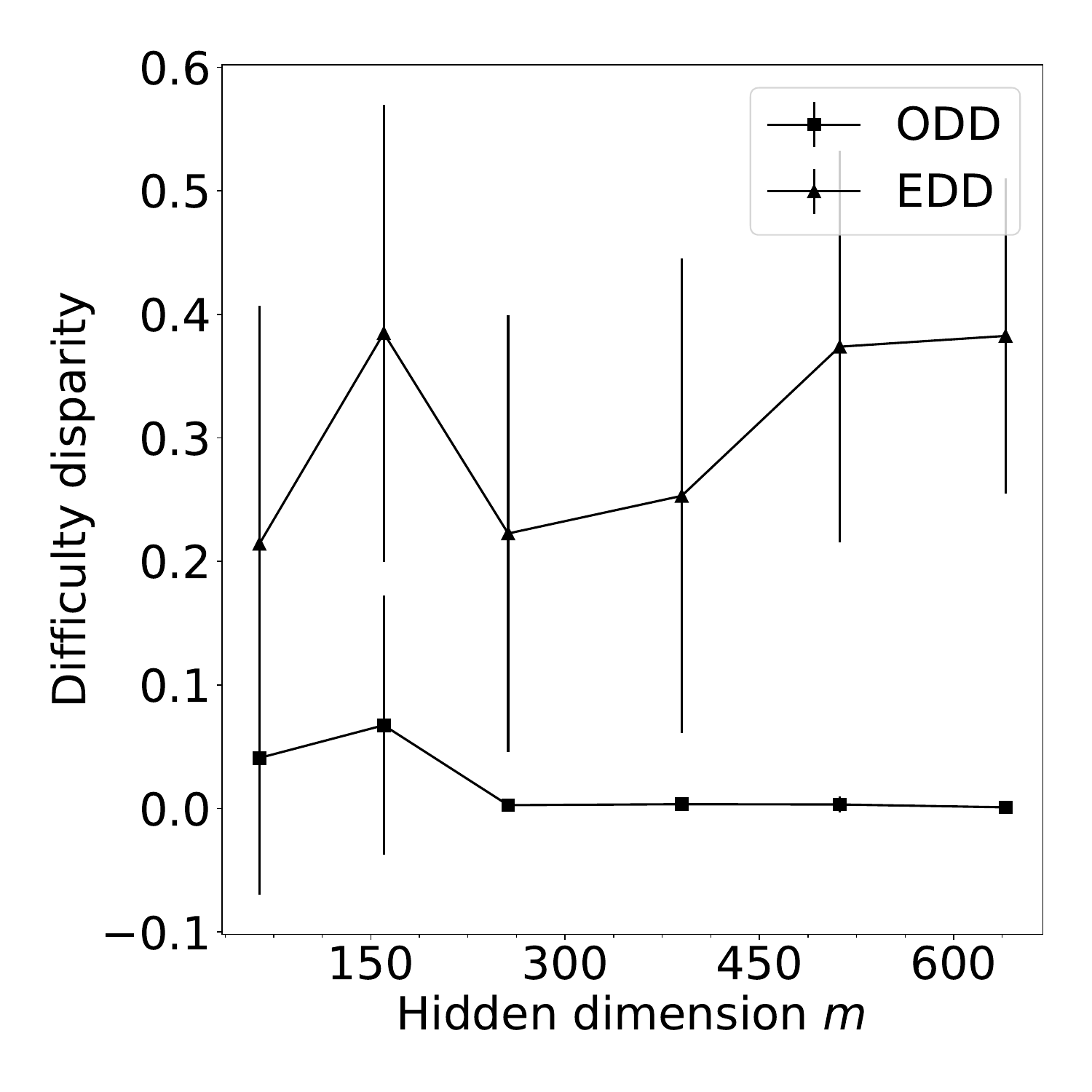}
    \caption{\textbf{Our theory predicts that a larger model size reduces bias on Colored MNIST in the single model setting.} We plot the $ODD$ and $EDD$ of a CNN for different model sizes $m$ (where $m$ is the dimension of the penultimate CNN layer) for Colored MNIST. As $m$ increases, the $ODD$ appears to decrease and plateau, which is in line with what our theory predicts in the regime where $\phi < 1$ (see analysis in Section \ref{sec:bias-amp-iso-results}). The $EDD$ does not tend towards 0. In our experiments, $\sigma_1^2 = \sigma_2^2 = 0.05$. For each value of $m$, the model is evaluated after $t = 80$ training steps. The error bars capture the standard deviation computed over 10 random seeds.}
    \label{fig:colored-mnist-2-errors}
\end{figure}

\clearpage

\section{Minority-Group Bias Plots}
\label{sec:minority-plots}
\label{sec:spurious-plots}

\begin{figure}[!ht]
    \centering
    \includegraphics[width=\linewidth]{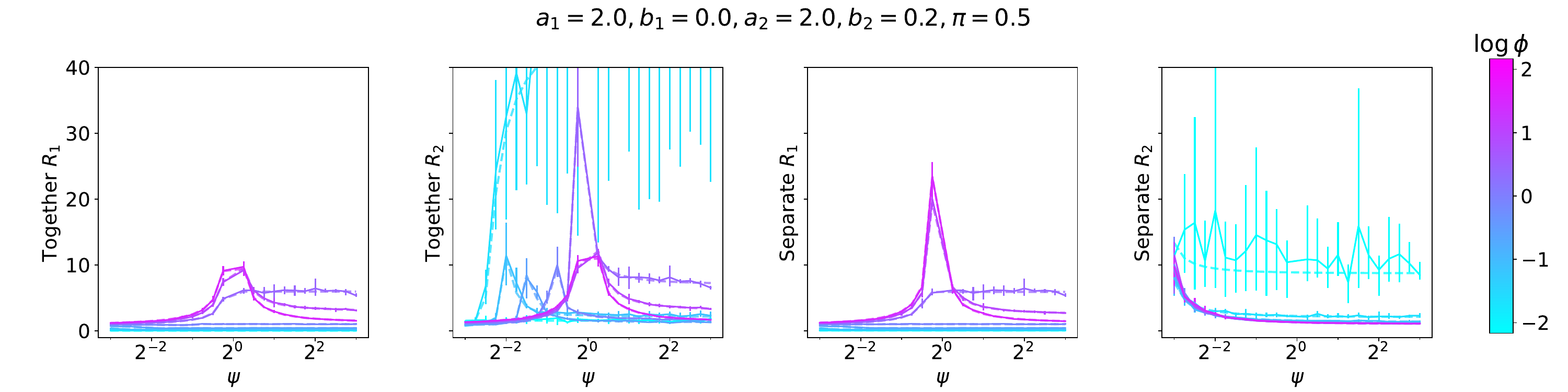}
    \includegraphics[width=\linewidth]{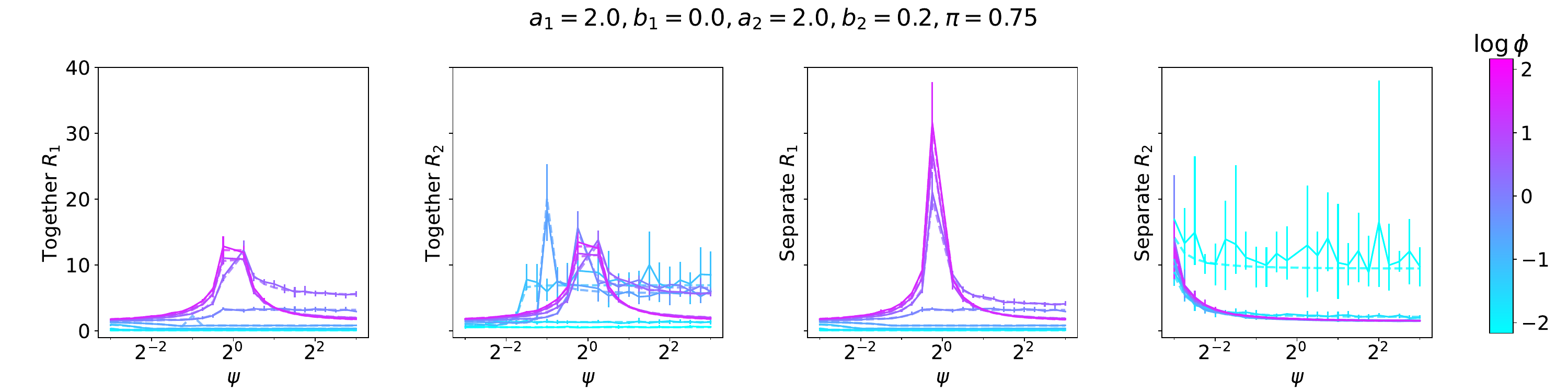}
    \includegraphics[width=\linewidth]{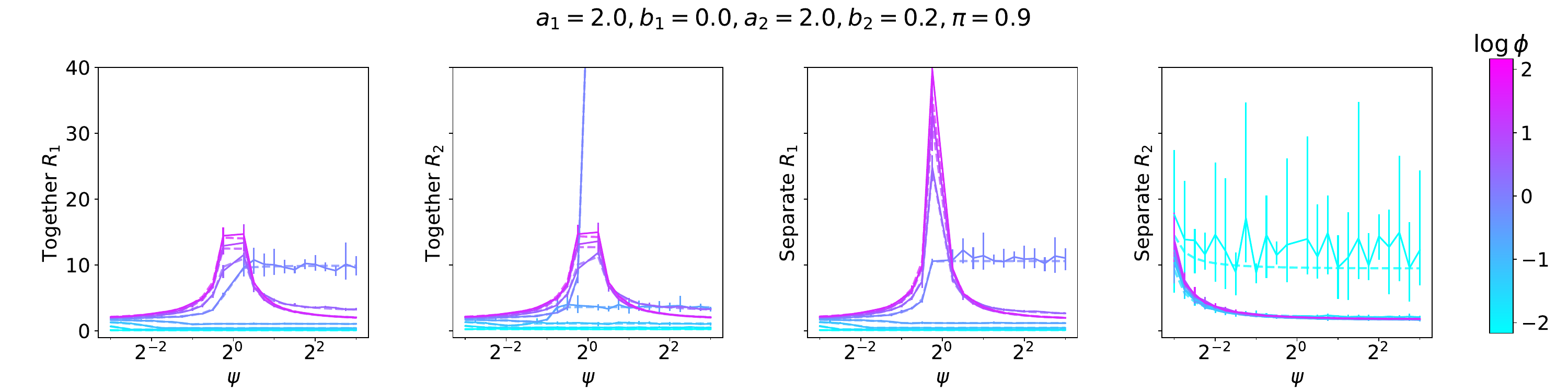}
    \caption{We empirically demonstrate that minority-group bias is affected by extraneous features. We validate our theory (Theorems \ref{thm:odd-theory-rand-proj} and \ref{thm:edd-theory-rand-proj}) for together $R_1, R_2$ (i.e., single model learned for both groups) and separate $R_1, R_2$ (i.e., separate model learned per group) under the setup described in Section \ref{sec:over-time-setup}. The solid lines capture empirical values while the corresponding lower-opacity dashed lines represent what our theory predicts. We include a black dashed line at $ADD = 1$ to contrast bias amplification vs. deamplification. All y-axes are on the same scale for easy comparison. The error bars capture the range of the estimators over 25 random seeds.}
\end{figure}

\clearpage

\begin{figure}[!ht]
    \centering
    \includegraphics[width=\linewidth]{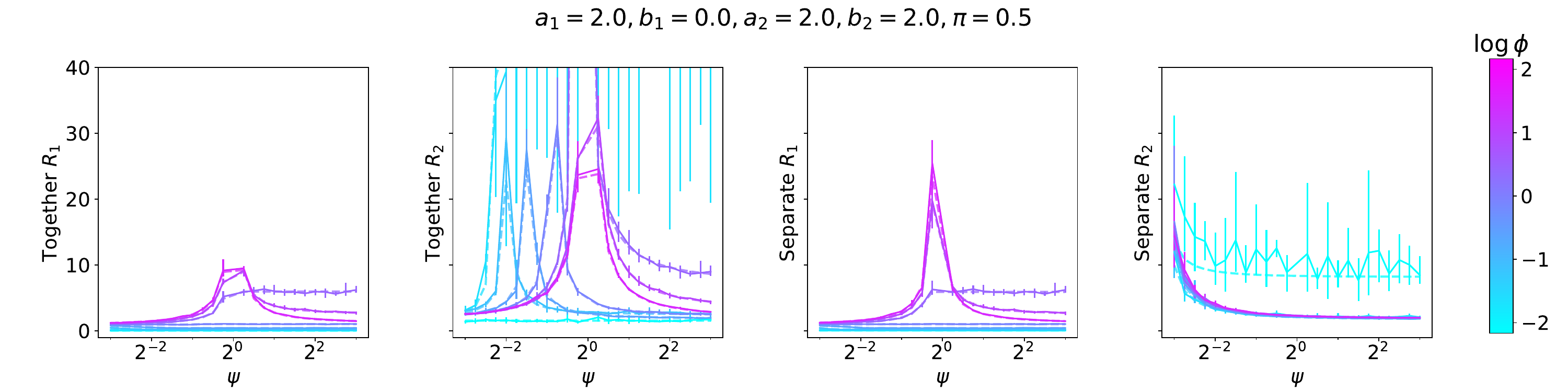}
    \includegraphics[width=\linewidth]{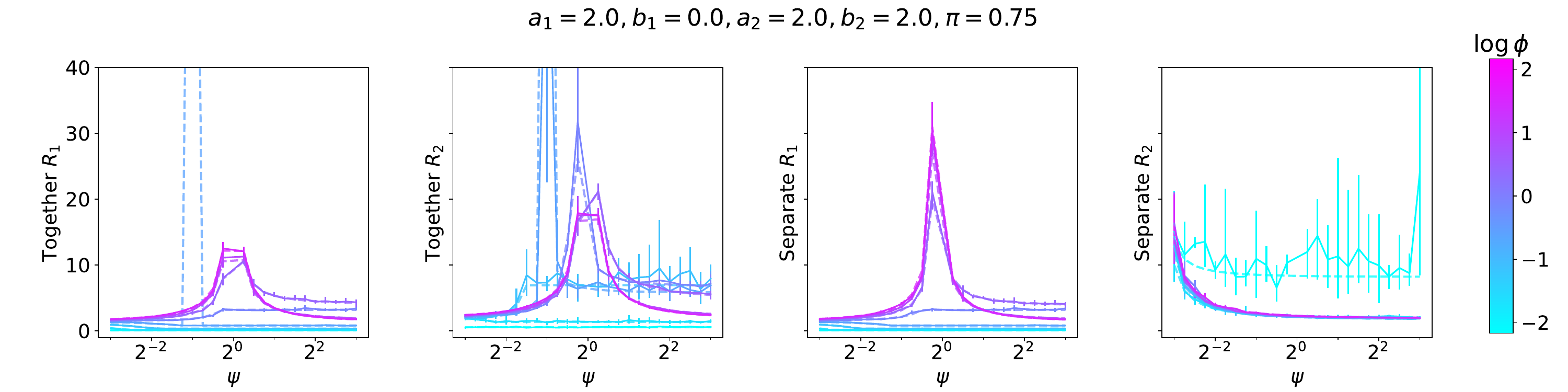}
    \includegraphics[width=\linewidth]{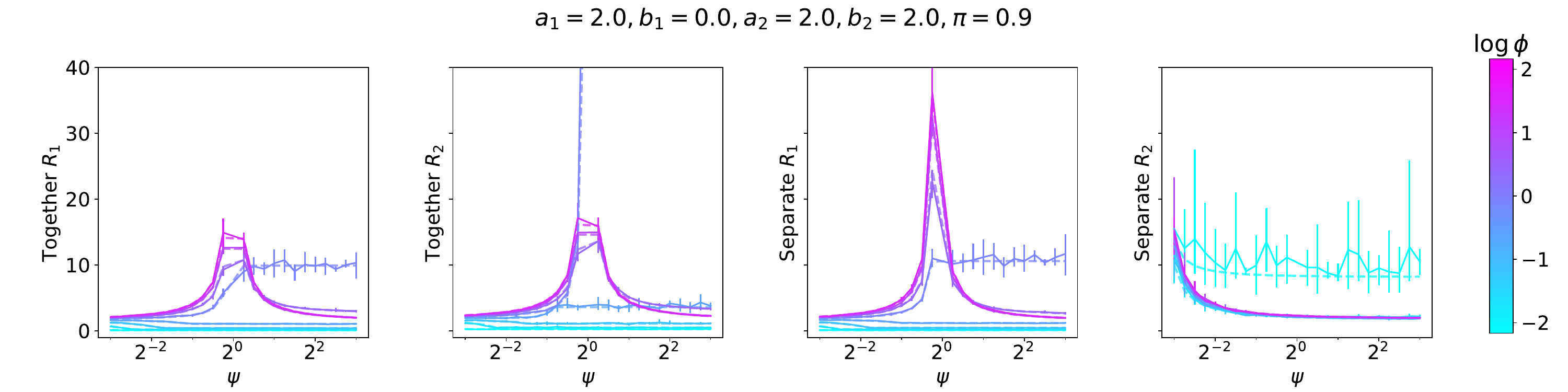}
    \caption{We empirically demonstrate that minority-group bias is affected by extraneous features. We validate our theory (Theorems \ref{thm:odd-theory-rand-proj} and \ref{thm:edd-theory-rand-proj}) for together $R_1, R_2$ (i.e., single model learned for both groups) and separate $R_1, R_2$ (i.e., separate model learned per group) under the setup described in Section \ref{sec:over-time-setup}. The solid lines capture empirical values while the corresponding lower-opacity dashed lines represent what our theory predicts. We include a black dashed line at $ADD = 1$ to contrast bias amplification vs. deamplification. All y-axes are on the same scale for easy comparison. The error bars capture the range of the estimators over 25 random seeds.}
\end{figure}

\clearpage

\begin{figure}[!ht]
    \centering
    \includegraphics[width=\linewidth]{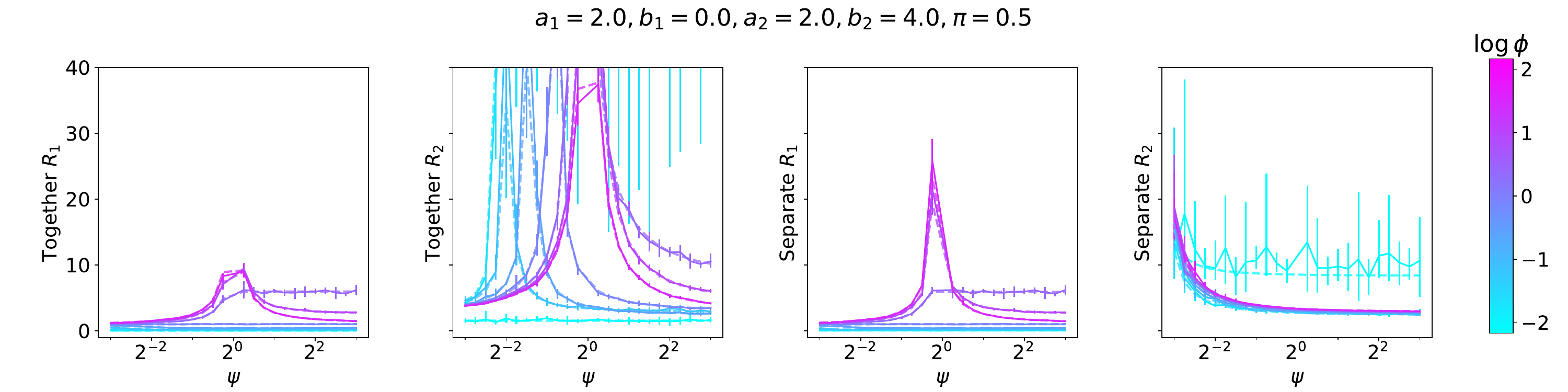}
    \includegraphics[width=\linewidth]{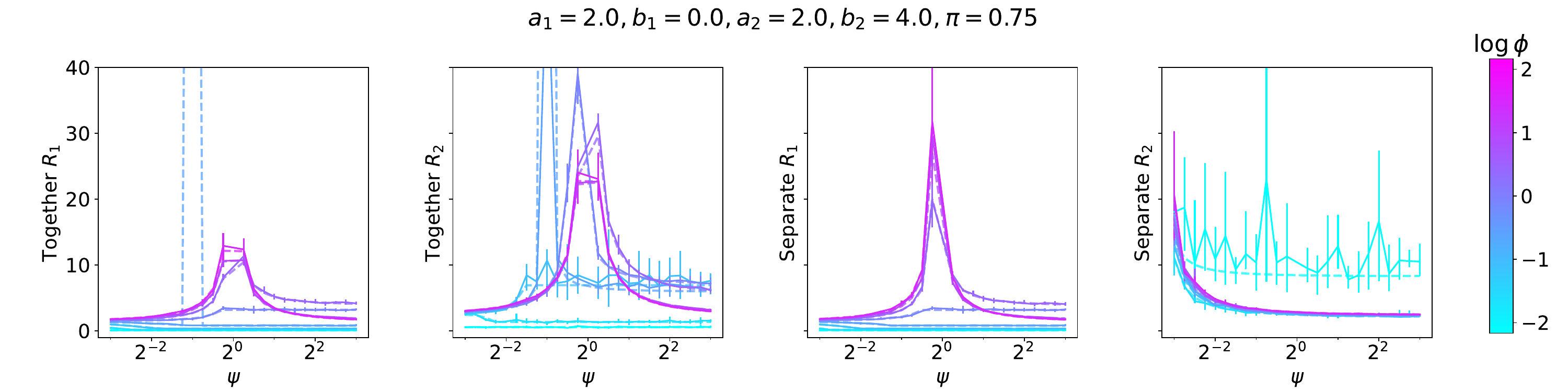}
    \includegraphics[width=\linewidth]{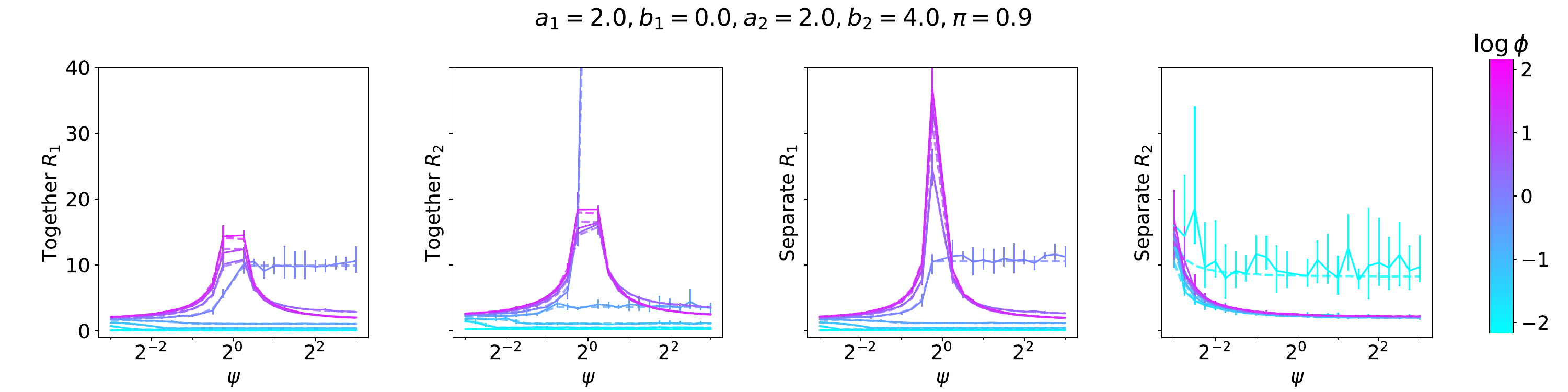}
   \caption{We empirically demonstrate that minority-group bias is affected by extraneous features. We validate our theory (Theorems \ref{thm:odd-theory-rand-proj} and \ref{thm:edd-theory-rand-proj}) for together $R_1, R_2$ (i.e., single model learned for both groups) and separate $R_1, R_2$ (i.e., separate model learned per group) under the setup described in Section \ref{sec:over-time-setup}. The solid lines capture empirical values while the corresponding lower-opacity dashed lines represent what our theory predicts. We include a black dashed line at $ADD = 1$ to contrast bias amplification vs. deamplification. All y-axes are on the same scale for easy comparison. The error bars capture the range of the estimators over 25 random seeds.}
\end{figure}

\clearpage

\section{Actionable Insights from Theory}
\label{sec:actionable-insights}

\paragraph{Searching for optimal hyperparameters.} In practice, an optimal regularization penalty $\lambda$ or training time $t$ can be selected by searching for values that strike a desired balance between overall validation error (that is not too high) and bias amplification (that is not too high). As we would estimate the test error using the empirical validation error, we can estimate bias amplification using the validation set. Moreover, we would need to train: (1) the main model on a mixture of data from groups, and (2) auxiliary separate models on the data for each group.

However, it may be expensive to train auxiliary models for each candidate value of $\lambda$ and $t$. The search space can be reduced by using insights from our theory. For instance, with overparamaterization, as $\lambda$ decreases (or $t$ increases), bias amplification increases and plateaus, and with underparameterization, as $\lambda$ decreases (or $t$ increases), bias deamplification increases and plateaus. When the curves are monotone with respect to $\lambda$, the optimal $\lambda$ is either at the left tail of the curve (e.g., $\lambda = 0$) or the right tail (i.e., the largest $\lambda$ among the reasonable options). In contrast, Figure \ref{fig:bias-amp-training} shows that when $\psi$ is close to the interpolation threshold of 1, bias amplification is often not monotone with respect to $\lambda$.

\paragraph{Informing evaluation and mitigation strategies.} Our theory offers avenues to assess whether a ML model trained on certain real-world data is prone to bias amplification and mitigate this amplification, even though we may lack direct access to population parameters like $\Sigma$. We can estimate such parameters using samples (e.g., $\widehat{\Sigma} = X^\top X$). However, even if we are unable to robustly estimate these parameters, our theory still provides valuable insights. For example, we observe that the ratios of parameters for the groups is often what matters, e.g., label noise ratio $\sigma_2^2 / \sigma_1^2$ (see Section \ref{sec:bias-amp-iso-results}), ratio of covariance eigenvalues (see Appendix \ref{sec:bias-amp-pow-setup}). Thus, practitioners can use our theory to get intuition about when {\em disparities} in the variability of labels and features across groups can amplify bias.

Moreover, our findings warn against the conventional wisdom that increased model overparameterization or data balancing can alleviate bias issues. In addition, our theory informs criteria for feature selection (e.g., discarding features with disparate variance across groups) and warns ML practitioners about the interplay between high vs. low feature-to-sample regimes and overparameterization in inducing bias amplification. Nevertheless, additional work is required to make rigorous connections between our theoretical findings and better strategies for evaluating and mitigating the bias of models.

\end{document}